%% file: main.tex
\documentclass{article}
\usepackage[utf8]{inputenc}
\usepackage{geometry}
\usepackage{amsmath}
\usepackage{amssymb}
\usepackage{mathtools}
\usepackage{amsthm}
\usepackage{bm}
\usepackage{comment}
\usepackage{etoolbox}
\usepackage[usenames,dvipsnames]{xcolor}
\definecolor{Bea_blue}{RGB}{71,127,124}
\definecolor{Bea_red}{RGB}{210, 77, 4}

\usepackage[colorlinks=true, citecolor=teal, linktoc=page, backref=page]{hyperref}
\usepackage{url}
\usepackage{thmtools} 
\usepackage{authblk}
\usepackage{thm-restate}
\usepackage{nicefrac} 
\usepackage{algpseudocode}
\usepackage{etoolbox}
\usepackage{bbm}
\usepackage{bm}
\usepackage{enumitem}
\usepackage[T1]{fontenc}
\usepackage[english]{babel}
\usepackage{ucs}
\usepackage{url}
\usepackage{dsfont}
\usepackage{graphicx}
\usepackage{caption}
\usepackage{pstricks,pstricks-add}
\usepackage{fontawesome}
\usepackage[mathscr]{euscript}
\usepackage{tikz}
\usetikzlibrary{decorations.pathreplacing, calligraphy}
\usepackage{stmaryrd}
\usepackage{cleveref}
\usepackage{subfigure}
\usepackage{physics}
\usepackage{comment}
\usepackage{etoolbox}
\usepackage[numbers]{natbib}
\usepackage{xspace}
\usepackage[ruled]{algorithm2e}
\usepackage{Style_files/New_commands}
\usepackage{float}
\crefname{algocf}{alg.}{algs.}
\Crefname{algocf}{Algorithm}{Algorithms}

\title{Computational-Statistical Gaps in Gaussian Single-Index Models}

\author[1]{Alex Damian}
\author[2]{Loucas Pillaud-Vivien}
\author[3]{Jason D. Lee}
\author[4,5]{Joan Bruna}
\affil[1]{PACM, Princeton University}
\affil[2]{Ecole de Ponts ParisTech, CERMICS}
\affil[3]{Electrical Engineering Department, Princeton University}
\affil[4]{Courant Institute of Mathematical Sciences, New York University}
\affil[5]{Center for Data Science, New York University}

\begin{document}

\maketitle

\begin{abstract}%
  Single-Index Models are high-dimensional regression problems 
  with planted structure, whereby labels depend on an unknown one-dimensional
  projection of the input via a generic, non-linear, and potentially non-deterministic transformation. As such, they encompass a broad class of statistical inference tasks, and provide a rich template to study statistical and computational trade-offs in the high-dimensional regime. 

  While the information-theoretic sample complexity to recover the hidden direction is linear in the dimension $d$, we show that computationally efficient algorithms, both within the Statistical Query (SQ) and the Low-Degree Polynomial (LDP) framework, necessarily require $\Omega(d^{\k/2})$ samples, where $\k$ is a ``generative'' exponent associated with the model that we explicitly characterize. Moreover, we show that this sample complexity is also sufficient, by establishing matching upper bounds using a partial-trace algorithm. Therefore, our results provide evidence of a sharp computational-to-statistical gap (under both the SQ and LDP class) whenever $\k>2$. To complete the study, we construct smooth and Lipschitz deterministic target functions with arbitrarily large generative exponents $\k$.
\end{abstract}

\setcounter{tocdepth}{1}
\tableofcontents

\input{introduction}

\input{mainresults}
\input{conclusions}

\clearpage
\bibliographystyle{alpha}

\bibliography{main}

\clearpage

\appendix

\input{notation}

\input{sq_app}

\input{hermite_app}

\input{proofs_sqexp}
\input{proofs_lowerbound}
\input{proofs_upperbound}

\input{proofs_smoothness}
\input{proofs_IT}
\input{concentration_app}

\end{document}

%% file: introduction.tex
\section{Introduction}
\label{sec:intro}

\subsection{Problem Setup}

The focus of this paper is to study a family of high-dimensional inference tasks 
characterized by \emph{planted} low-dimensional structure. In the context of supervised learning, where a learner observes a dataset $\{ (x_i, y_i) \}_{i = 1}^n$ with input features $x \in \R^d$ and labels $y\in \R$, the natural starting point is to consider data with hidden one-dimensional structure, and where features are drawn from the standard Gaussian measure $\gamma_d$ in $\R^d$:

\begin{definition}[Gaussian Single-Index Model]\label{def:single-index_model}
    We say that a joint distribution $\PP \in \mathcal{P}(\R^d \times \R)$ follows a Gaussian single index model if there exists a probability measure $\P \in \mathcal{G} \subset \mathcal{P}(\R^2)$ and 
    $w^\star \in {S}^{d-1}$ such that
    $\PP = [R_{w^\star} \otimes \mathrm{Id}]_\# [\gamma_{d-1} \otimes \P]$,
    where $R_{w^\star} \in \mathcal{O}_d$ is any orthogonal matrix
    whose last column is $w^\star$, i.e. of the form $R_{w^\star}=[R_{\!\perp}\, w^\star]$ and 
    \begin{align}
    \mathcal{G} = \left\{ \nu_{(z,y)} \in \mathcal{P}( \R \times \R) ;\ \nu_z = \gamma_1 ; \E_\nu[Y^2] < \infty; \ \mathbb{D}_{\chi^2}(\nu || \gamma_1 \otimes \nu_y ) > 0 \right\}    
    \end{align}
    is the class of non-separable bivariate probability measures, whose first marginal is the standard one-dimensional Gaussian, and whose second-order moment w.r.t. its second argument is finite. 
\end{definition}
In words, a single index model is a joint distribution in $\R^d \times \R$ that admits a product structure $\mathrm{d}\PP(x, y) = \mathrm{d}\gamma_{d-1}(R_{\!\perp}^\top x)\, \mathrm{d}\P(w^\star \cdot x, y)$ into a non-informative component $R_{\!\perp}^\top x$ of dimension $d-1$, and an informative component, determined precisely by $\P$ and the direction $w^\star$. The Gaussian setting further specifies the marginal distribution of the features.  
We will denote the planted model by $\PP_{w^\star, \P}$ (or simply $\PP_{w^\star}$ when the context is clear). We have used the conventions that, for any random variable $X$, $\P_x$ stands for the law of $X$ under $\P$, e.g. $\P_z = \gamma_1$, and $\P_y$ the marginal of $Y$. Similarly, we will make use of notations $\P_{z|y},\P_{y|z}$, that stand respectively for the conditional probability laws of $Z$ given $Y$ and $Y$ given $Z$. %

Note that in this language, Gaussian single-index models are closely related to Non-Gaussian Component Analysis (NGCA) \citep{diakonikolas2017statistical}. In NGCA, a $d$-dimensional distribution admits a product structure in terms of a univariate non-Gaussian marginal and a $d-1$ Gaussian distribution. In our case, the planted non-Gaussian component is two-dimensional, but the statistician is given one direction of the non-Gaussian subspace (the label $Y$). 

Throughout the paper, we use the letter $Z$ to denote the one-dimensional (Gaussian) random variable~$w^\star \cdot X$. When there exists $\sigma: \R \to \R$ such that $\P_{y|z}(\cdot, z) = \delta_{\sigma(z)}$
, we say that $(X,Y)$ follows a deterministic single-index model as $Y = \sigma(Z) = \sigma(w^\star \cdot X)$, where $\sigma$ is said to be \emph{link function}. However, Definition \ref{def:single-index_model} allows for additional randomness in the label, as long as its distribution only depends on $x$ through $z = w^\star \cdot x$. Examples of non-deterministic single-index models include 
additive noise, where $Y = \sigma(Z) + \xi$ and $\xi$ is an independent random variable, e.g. $\xi \sim N(0,1)$; multiplicative noise, where $Y = \xi \sigma(Z)$, Mixture of distributions, where $Y \sim \mu_1$ if $Z \ge 0$ and $Y \sim \mu_2$ if $Z < 0$, or Massart-type noise, where $Y = \xi \sigma(Z)$ and $\PP(\xi = 1) = 1-\eta(Z)$ and $\PP(\xi = -1) = \eta(Z)$; see Figure \ref{fig:examples_noise}.

\begin{figure}
    \centering
    \includegraphics[width=0.3\textwidth]{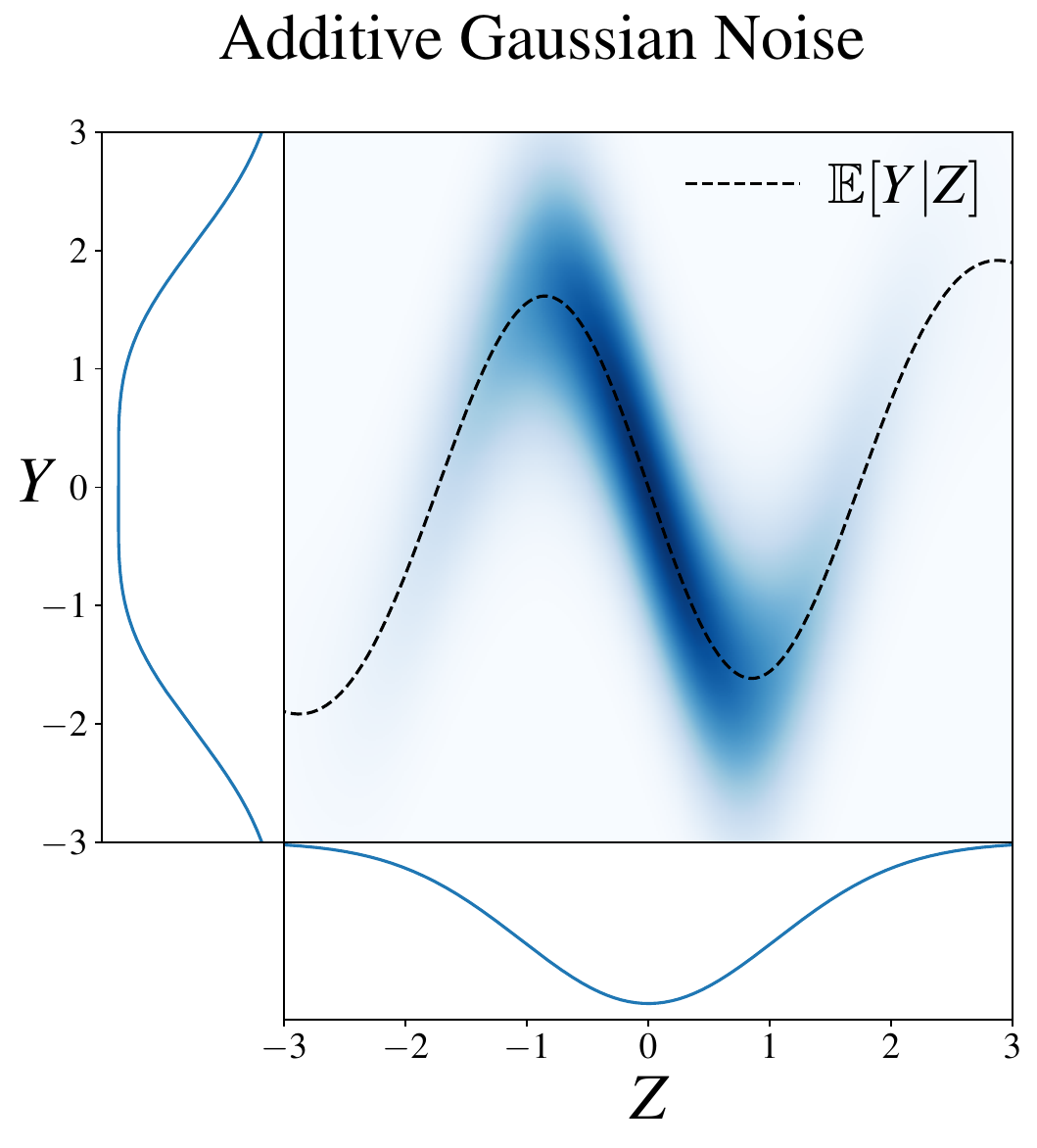}
    \includegraphics[width=0.3\textwidth]{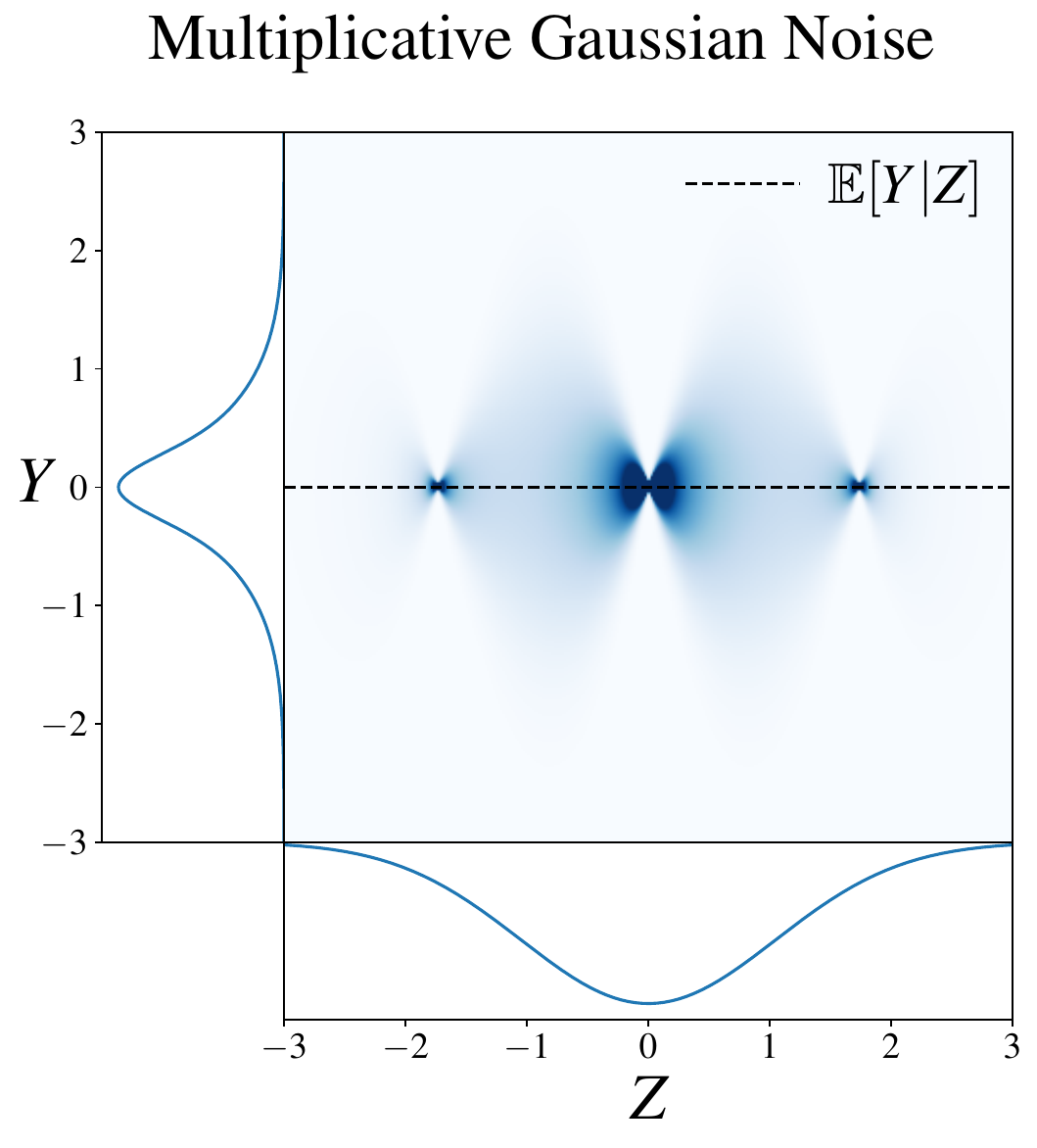}
    \includegraphics[width=0.3\textwidth]{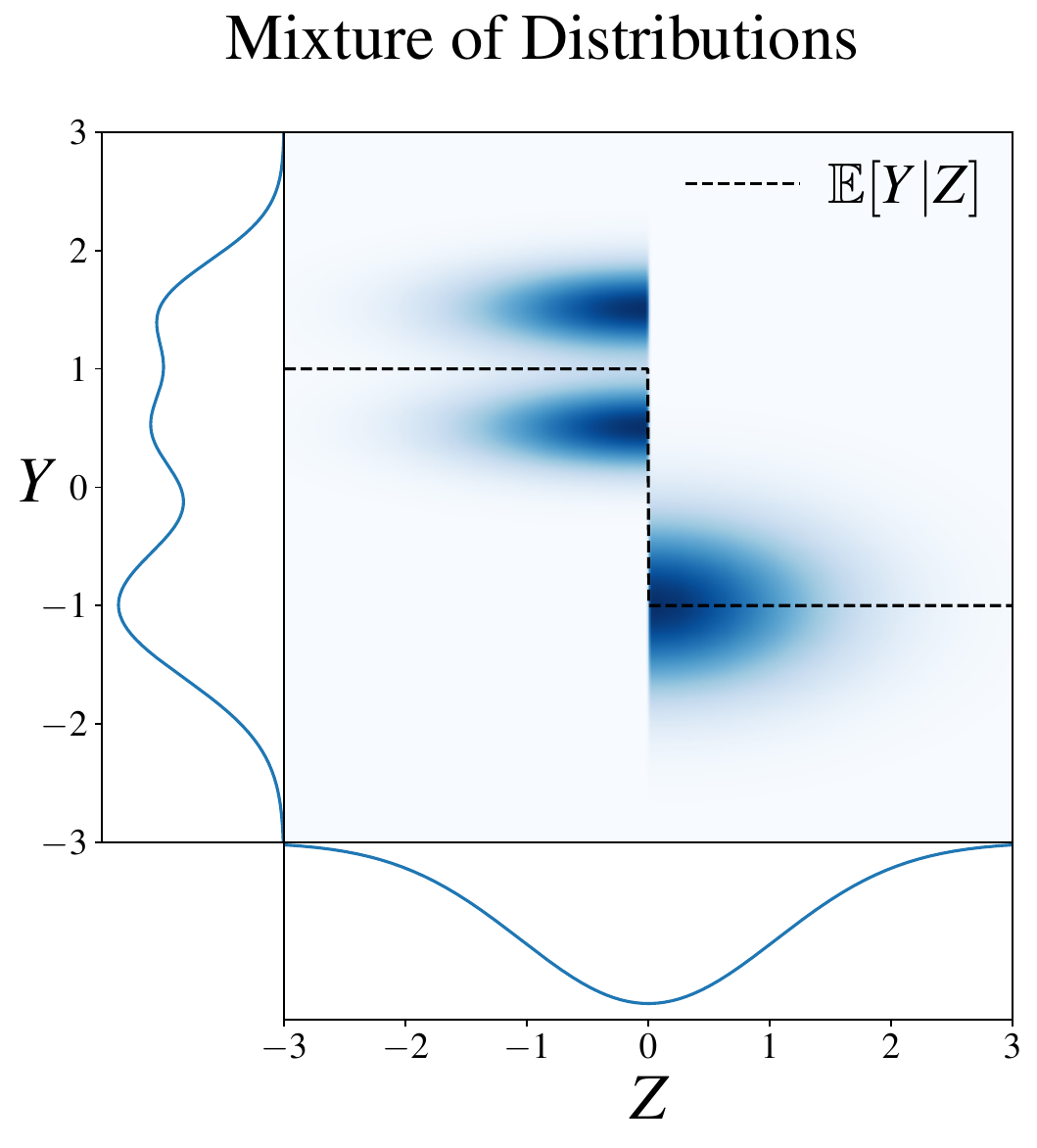}
    \caption{Visualization of the joint density $\P$ of $(Z,Y)$ for the additive noise model, multiplicative noise model, and mixture of distributions model. The heatmap shows the density of $\P$ and the plots to the left of and below the heatmap show the densities of the marginals $\P_y$ and $\P_z$ respectively.}
    \label{fig:examples_noise}
\end{figure}

In the remainder of this paper, and unless stated otherwise, we assume that $\P$ is known, so that the inference task reduces to estimating the 
hidden direction $w^\star$ drawn from the uniform prior over ${S}^{d-1}$, after obserivng $n$ iid samples from $\PP_{w^\star, \P}$.  
By instantiating specific choices for $\P$, one recovers several well-known statistical inference problems, such as linear recovery, phase-recovery, one-bit compressed sensing, generalized linear models or Non-Gaussian Component Analysis \citep{diakonikolas2017statistical}, as well as close variants of Tensor PCA \citep{montanari2014statistical}. An important common theme across these different statistical models over recent years has been to understand \emph{computational-to-statistical gaps}, namely comparing the required amount of samples needed to estimate $w^\star$, using either computationally efficient methods or 
brute-force search. Computational efficiency may be measured either by restricting estimation algorithms to belong to certain computational models, or by establishing reductions to problems believed to be computationally hard. 
In this paper, we focus our attention on \emph{Statistical-Query} algorithms \citep{kearns1998efficient}, which capture a broad class of learning algorithms including robust gradient-descent methods, as well as the \emph{Low-Degree Polynomial} method \cite{hopkins2018statistical, kunisky2019notes}, a flexible framework to assess the average-case complexity of statistical inference tasks. 

The primary focus of this work is to derive the optimal sample complexity for recovering~$w^\star$ given i.i.d. samples~$\{(x_i,y_i)\}_{i = 1}^n$ from a single-index model (Definition \ref{def:single-index_model}). Perhaps unsuprinsingly, one can recover $w^\star$ up to error $\epsilon$ by brute-force search when $n =\Theta(d/\epsilon^2)$ (Theorem \ref{thm:ITupper}), in accordance with related statistical inference tasks. However, when restricted to SQ and LDP algorithms, our main results will establish a tight sample complexity of order $n = \Theta(d^{\k/2})$, with matching upper and lower bounds, where $\k = \k(\P)$ is an exponent associated with $\P$ that we explicitly characterize (Definition~\ref{def:sq_exponent}). We thus obtain evidence of a computational-statistical gap of polynomial scale for a broad class of inference problems.

\subsection{Background and Related Work}

Single-index models (also called generalized linear models) have a long history in the statistics literature \citep{mccullagh1983generalized, ichimura1993semiparametric,hristache2001index,hardle2004nonparametric,dalalyan2008new,dudeja2018learning}. When the link function $\sigma$ is monotonic, there are perceptron-like algorithms (e.g. Isotron/GLM-tron \citep{kalai2009isotron,kakade2011efficient}) that recover the ground truth signal with $n = \Theta(d)$ samples where $d$ is the dimension of the data.

Perhaps the simplest example of a generalized linear model with a non-monotonic link function is phase retrieval in which $\sigma(u) = |u|$. In contrast to the monotonic case, phase retrieval exhibits a conjectured \emph{statistical-computational gap} in the noisy setting \citep{mondelli2018fundamental,barbier2019optimal,maillard2020phase} In the presence of label noise, there are constants $c_{\mathrm{IT}} < c_{\mathrm{ALG}}$ such that recovery is information theoretically possible with $n/d \ge c_{\mathrm{IT}}$ but is conjectured to be computationally hard unless $n/d \ge c_{\mathrm{ALG}}$. Note, however, that this is not true in the noise-free setting as there are computationally efficient algorithms based on lattice reduction that achieve the information-theoretic threshold \citep{pmlr-v65-andoni17a,song2021cryptographic}.

For both monotonic and quadratic link functions, the optimal sample complexity scales linearly in the input dimension $d$. In addition, this rate can even be recovered by directly optimizing the maximum likelihood objective using simple first order methods such as online stochastic gradient descent (SGD). However, the problem is significantly harder when the link function is more complicated. \cite{arous2020online} demonstrated that for general single index models, online SGD on the maximum likelihood objective succeeds if and only if $n = \tilde \Theta(d^{\max(1,\kk-1)})$ where $\kk$ is the \emph{information exponent} of the single index model, which is defined to be the degree of the first nonzero Hermite coefficient of $\sigma$ (Definition \ref{def:information_exponent}). However, the significance of the information exponent extends far past optimizing the maximum likelihood objective. It has appeared as the fundamental object in determining the sample complexity in many follow up works \citep{bietti2022learning,damian2022neural,damian2023smoothing,dandi2023twolayer,abbe2023sgd}.

\cite{damian2022neural} formalized this by proving a correlational statistical query (CSQ) lower bound which shows that either $n \gtrsim d^{\max(1,\kk/2)}$ samples or exponentially many queries are necessary for learning a single index model with information exponent $\kk$. In addition, this lower bound is tight as it is possible to learn a single index model with $\tilde \Theta(d^{\max(1,\kk/2)})$ samples in polynomial time by training wide three layer neural networks \citep{chen2020towards} or by smoothing the landscape of the maximum likelihood objective \citep{damian2023smoothing}.

The information exponent arose as the fundamental object governing sample complexities given \emph{correlational queries} of the form $\E[Y h(X)]$. This is rich enough to capture gradient methods with mean squared error loss as they only interact with the data through correlational queries:
\begin{align}
    L(\theta) = \mathbb{E}[(f_\theta(X)-Y)^2] \quad\implies\quad \nabla_\theta L(\theta) = \mathbb{E}[(f_\theta(X)-\underbrace{Y)\nabla_\theta f_\theta(X)}_{\substack{\text{correlational } \text{query}}}]~.
\end{align}
However, methods outside of the correlational statistical query (CSQ) framework can break the~$n \gtrsim d^{\max(1,\kk/2)}$ sample complexity barrier \citep{chen2020learning,mondelli2018fundamental,barbier2019optimal,maillard2020phase,Dandi2024TheBO}. The general technique for these methods is to first apply a pre-processing function~$\mathcal{T}$ to the label~$Y$ to lower the information exponent to~$2$ before running a CSQ algorithm. This is possible because the information exponent defined by \cite{arous2020online} is not composition invariant, i.e. it is possible that the information exponent of~$(X,\mathcal{T}(Y))$ is strictly smaller than the information exponent of~$(X,Y)$. In fact, \citep{mondelli2018fundamental,barbier2019optimal,maillard2020phase} give a necessary and sufficient condition on~$\P$ that enables such~$\mathcal{T}$ to lower the information exponent to~$2$. Using the notation described in Definition \ref{def:single-index_model}, the condition on $\P$ writes:
\begin{align}
\label{eq:basicSQ2}
    \mathbb{E} \left[\mathbb{E} \left[Z^2-1 | Y\right]^2\right] &\ne 0~,
\end{align}
where expectations are taken with respect to~$(Z,Y) \sim \P~$.
Such `pre-processing' methods that go beyond the CSQ lower bound are in fact instances of SQ-algorithms, which interact with data via general queries of the form $\E[\phi(X,Y)]$ for a broad class of test functions $\phi$. In particular, these works imply that the SQ complexity for learning single-index models satisfying \eqref{eq:basicSQ2} is $n = \Theta(d)$. The natural follow-up question --and the main focus of this work-- is to quantify the statistical-computational gap for \emph{arbitrary} $\P$, going beyond the criterion given by Eq.\eqref{eq:basicSQ2} and identifying necessary and sufficient conditions for efficient recovery.

\subsection{Summary of Main Results}
Our first main result establishes sample complexity lower bounds required by any SQ-algorithm to solve the single-index problem determined by $\P$:
\begin{theorem}[SQ lower bound, informal version of \Cref{thm:sq_lower_bound}]
\label{thm:informal1}
Given $\P \in \mathcal{G}$ and $n$ i.i.d. samples from $\PP_{w^\star, \P}$, there exists an explicit exponent $\k = \k(\P)<\infty$ such that no polynomial time SQ-algorithm can succeed in recovering $w^\star$ unless $n \gtrsim d^{\k/2}$.  
\end{theorem}
Additionally, this SQ complexity coincides with the rate where the low-degree method succeeds:
\begin{theorem}[Low-degree method detection lower bound, informal version of \Cref{thm:low_degree}]\label{thm:informal2}
Under the low-degree conjecture (\Cref{conjecture:low_deg}), if $\k = \k(\P)$ is the exponent in \Cref{thm:informal1}, then given $n$ i.i.d. samples from either $\PP_{w^\star,\P}$ where $w^\star \sim \unif(S^{d-1})$ or a null distribution $\PP_0 := \gamma_d \otimes \P_y$, no polynomial time algorithm can distinguish $\PP_{w^\star,\P}$ from $\PP_0$ unless $n \gtrsim d^{\k/2}$.
\end{theorem}
We denote the associated exponent $\k(\P)$ as the \emph{generative exponent} of the model, in contrast with the information exponent, and analyze its main properties in Section \ref{sec:sqexp}. In the meantime, the attentive reader might already anticipate that indeed one has $\k(\P) \leq \kk(\P)$. 

Our second main result shows that these computational lower bounds are tight, by exhibiting a polynomial-time algorithm, based on the partial-trace estimator, that succeeds as soon as $n \gtrsim d^{\k/2}$: 
\begin{theorem}[informal version of \Cref{thm:optimal_sq_alg} and \Cref{coro:partialtrace}]
\label{thm:informal}
Given $\P \in \mathcal{G}$ and $n$ i.i.d. samples from $\PP_{w^\star, \P}$, there exists an efficient algorithm that succeeds in estimating $w^\star$ when $n \gtrsim d^{max(1,\k/2)}$, even when $\P$ is misspecified.  
\end{theorem}

Combined with an information-theoretic sample complexity upper bound $n = O(d)$ (\Cref{thm:ITupper}) ---which follows from relatively standard arguments, \Cref{thm:informal} thus establishes a sharp computational-to-statistical gap (under both the SQ and the LDP frameworks) as soon as $\k(\P) >2$. Our last main contribution shows that for any $k$, there exists smooth distributions such that $\k(\P) = k$:
\begin{theorem}[informal version of \Cref{thm:smooth_link} and \Cref{thm:additive_noise_link}]
    For any $k \in \mathbb{N}$ and $\tau \geq 0$, there exists $\sigma \in C^{\infty}_b(\R)$ such that $Z\sim \gamma$, $Y = \sigma(Z) + W$, with $W \sim N(0, \tau^2)$ defines a joint distribution $(Y,Z) \sim \P$ satisfying $\k(\P) = k$. 
\end{theorem}

%% file: mainresults.tex
\label{sec:mainres}

\paragraph{Notations} Given a probability measure $\mu$ defined over $\R^m$, 
we denote by $L^2(\R^m, \mu)$ the space of $\mu$-measurable, square-integrable functions. For $f, g \in L^2(\R^m, \mu)$, we write $\langle f,g \rangle_\mu = \E_{X \sim \mu} [f(X) g(X)]$ and $\|f\|_\mu^2 = \langle f, f \rangle_\mu$. We use $\gamma = \gamma_1$ to denote the standard Gaussian measure $N(0,1)$ and for $d \ge 1$ we use $\gamma_d$ to denote the standard isotropic Gaussian measure in $\R^d$, $N(0,I_d)$.

\paragraph{Hermite Polynomials}
We define the normalized Hermite polynomials $\{h_k\}_{k \ge 0} \in L^2(\R, \gamma_1)$ by
\begin{align}
    h_k(z) := \frac{(-1)^k}{\sqrt{k!}} \frac{1}{\gamma_1(z)} \frac{\partial^k \gamma_1(z)}{\partial z^k}~,\,z \in \R, \,k \in \N~.
\end{align}
These polynomials satisfy the orthogonality relations $\mathbb{E}_{Z \sim \gamma_1}[h_j(Z)h_k(Z)] = \mathbf{1}_{j=k}.$

\paragraph{Acknowledgements:} We thank Guy Bresler, Yatin Dandi, Ilias Diakonikolas, Daniel Hsu, Florent Krzakala, Theodor Misiakievicz, Tselil Schramm, Denny Wu, Ilias Zadik and Lenka Zdeborova for useful feedback during the completion of this work, which was partially developed during 2023's Summer School ``Statistical Physics and ML back together again" in Cargese. 
AD acknowledges support from a NSF Graduate Research Fellowship. AD and JDL acknowledge support of the ARO under MURI Award W911NF-11-1-0304, the Sloan Research Fellowship, NSF CCF 2002272, NSF IIS 2107304, NSF CIF 2212262, ONR Young Investigator Award, and NSF CAREER Award 2144994.
JB was partially supported by the Alfred P. Sloan Foundation
and awards NSF RI-1816753, NSF CAREER CIF 1845360, NSF CHS-1901091 and NSF DMS-MoDL
2134216.

\begin{figure}
    \centering
    \includegraphics[height=0.3\textwidth]{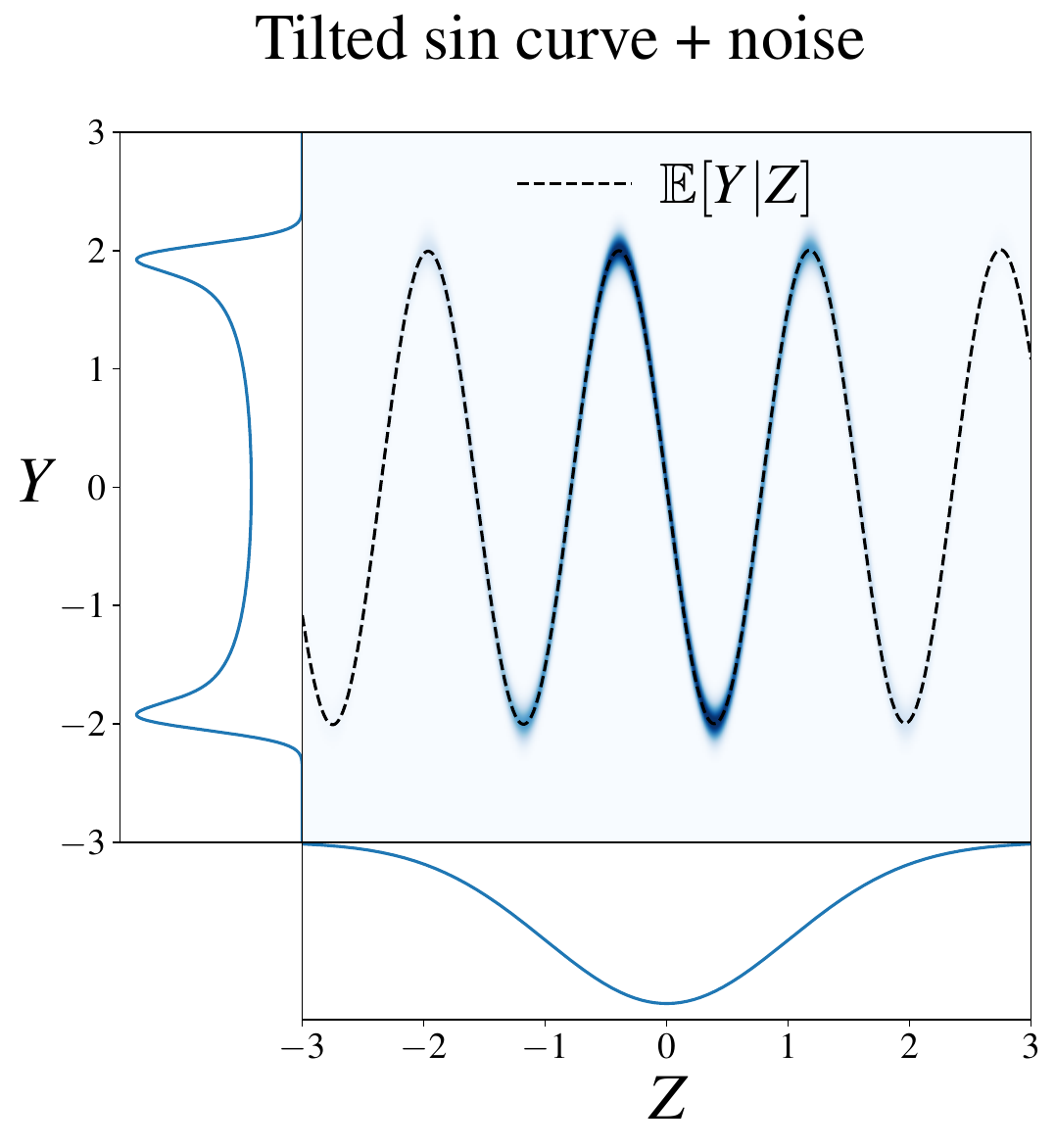}
    \includegraphics[height=0.3\textwidth]{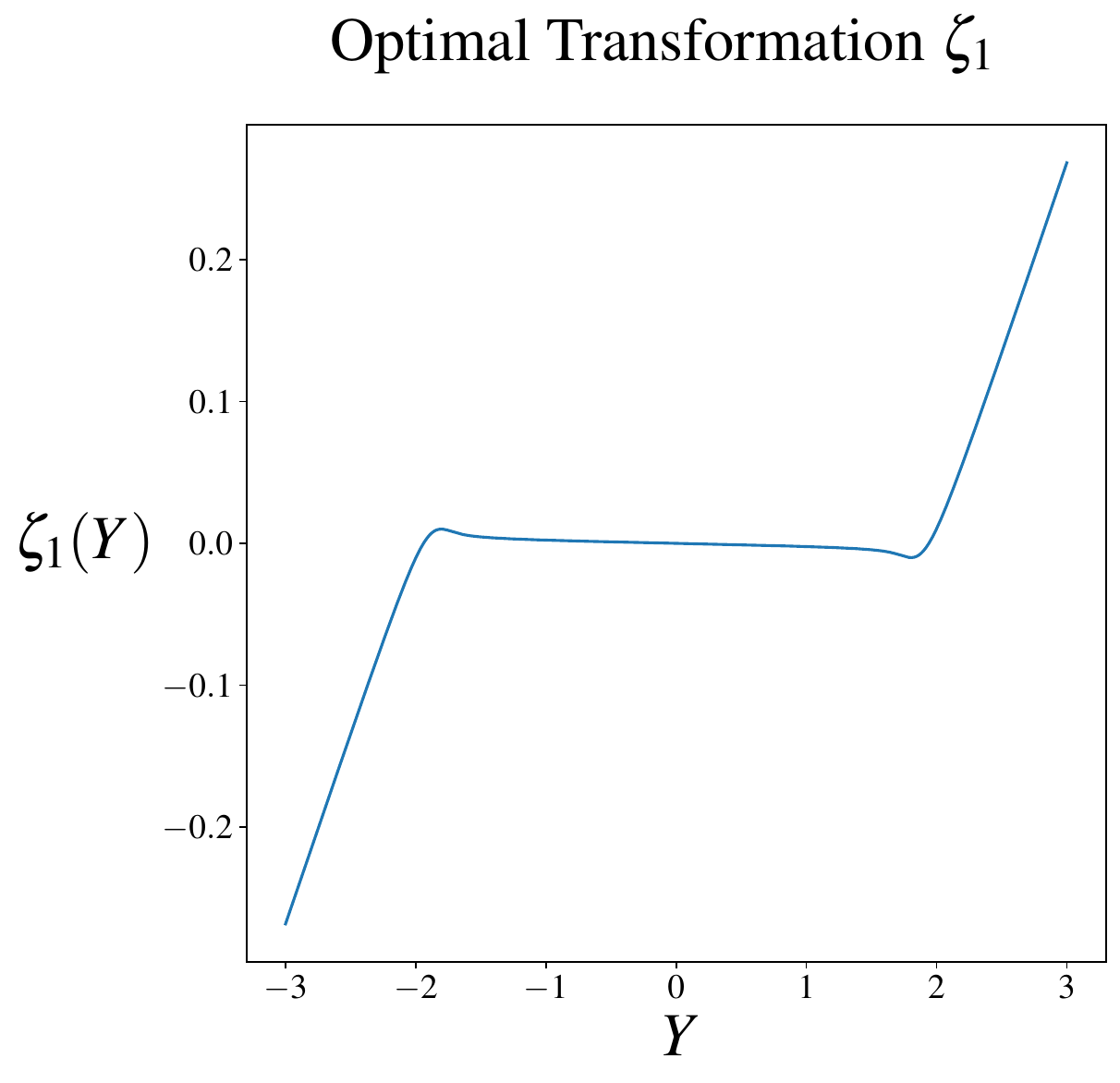}
    \includegraphics[height=0.3\textwidth]{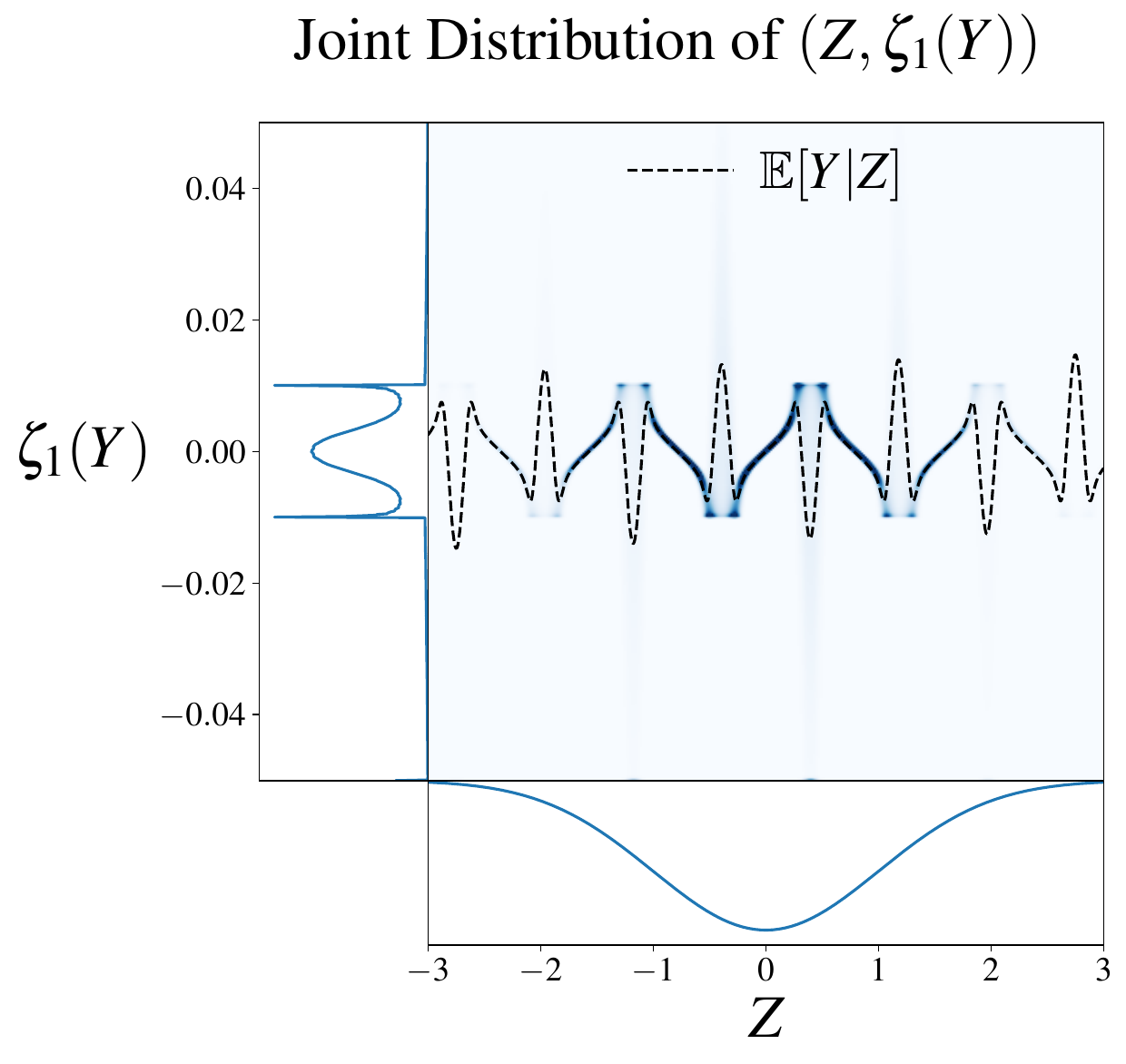} \\
    \includegraphics[height=0.3\textwidth]{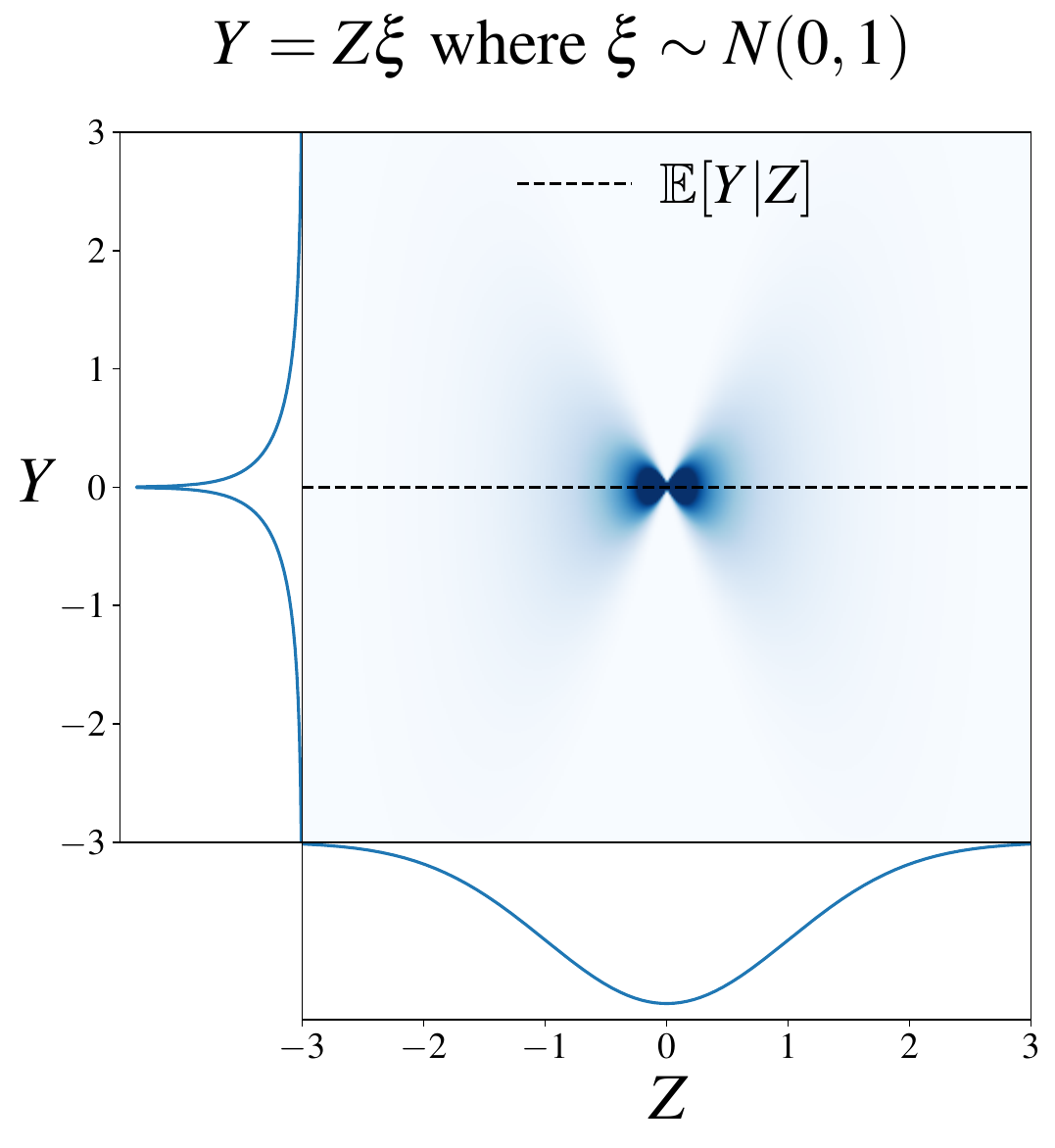}
    \includegraphics[height=0.3\textwidth]{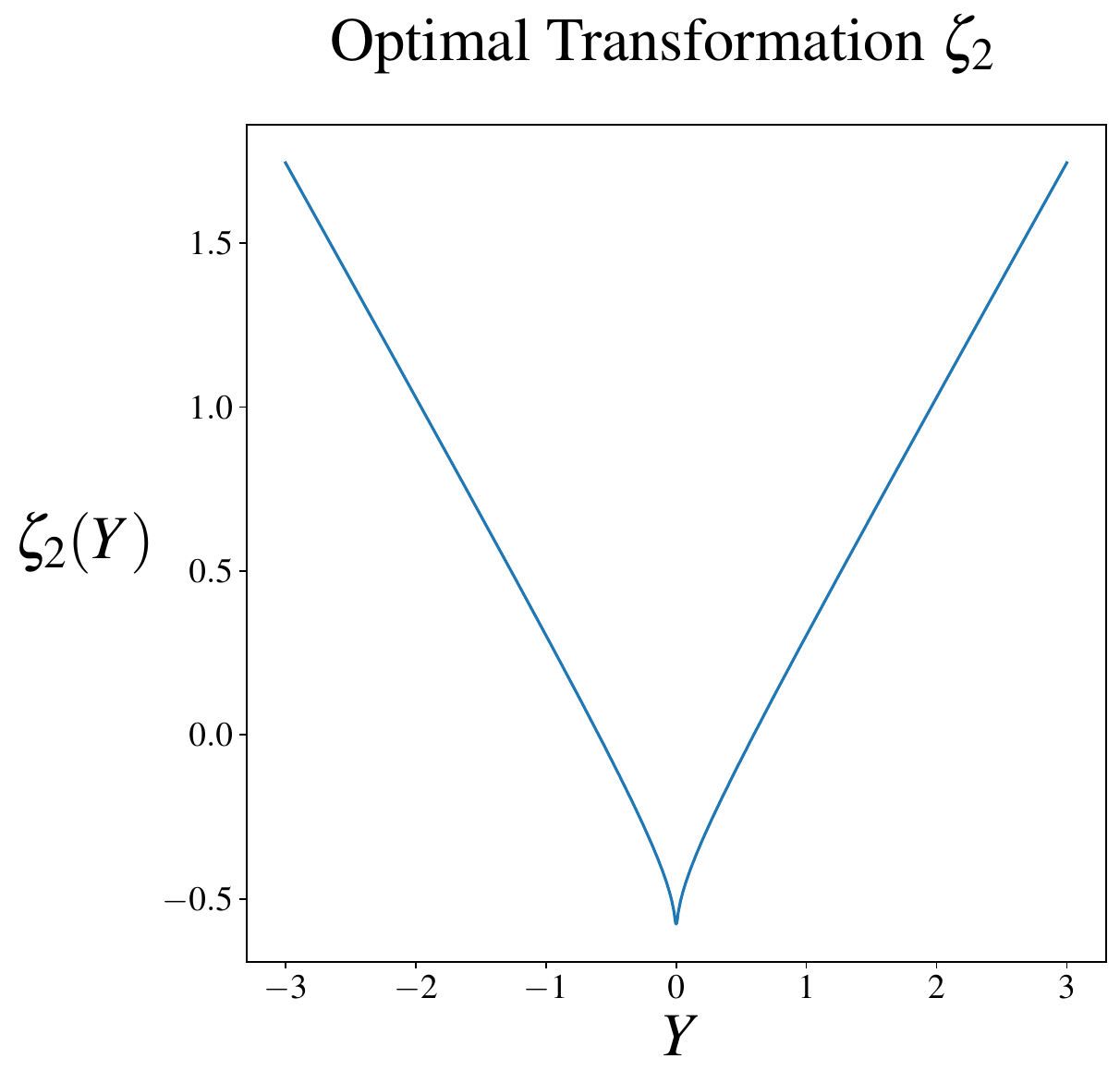}
    \includegraphics[height=0.3\textwidth]{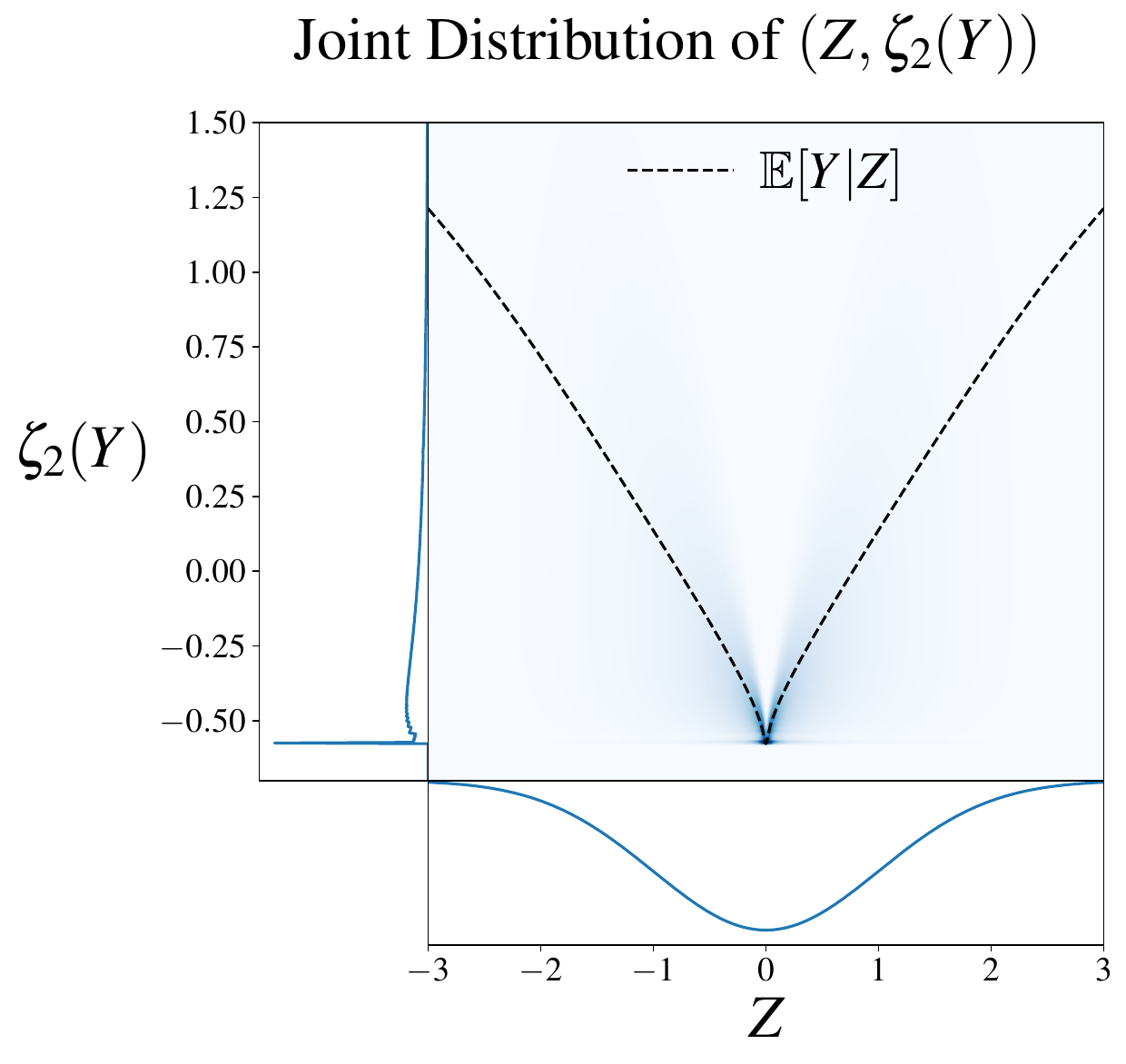} \\
    \includegraphics[height=0.3\textwidth]{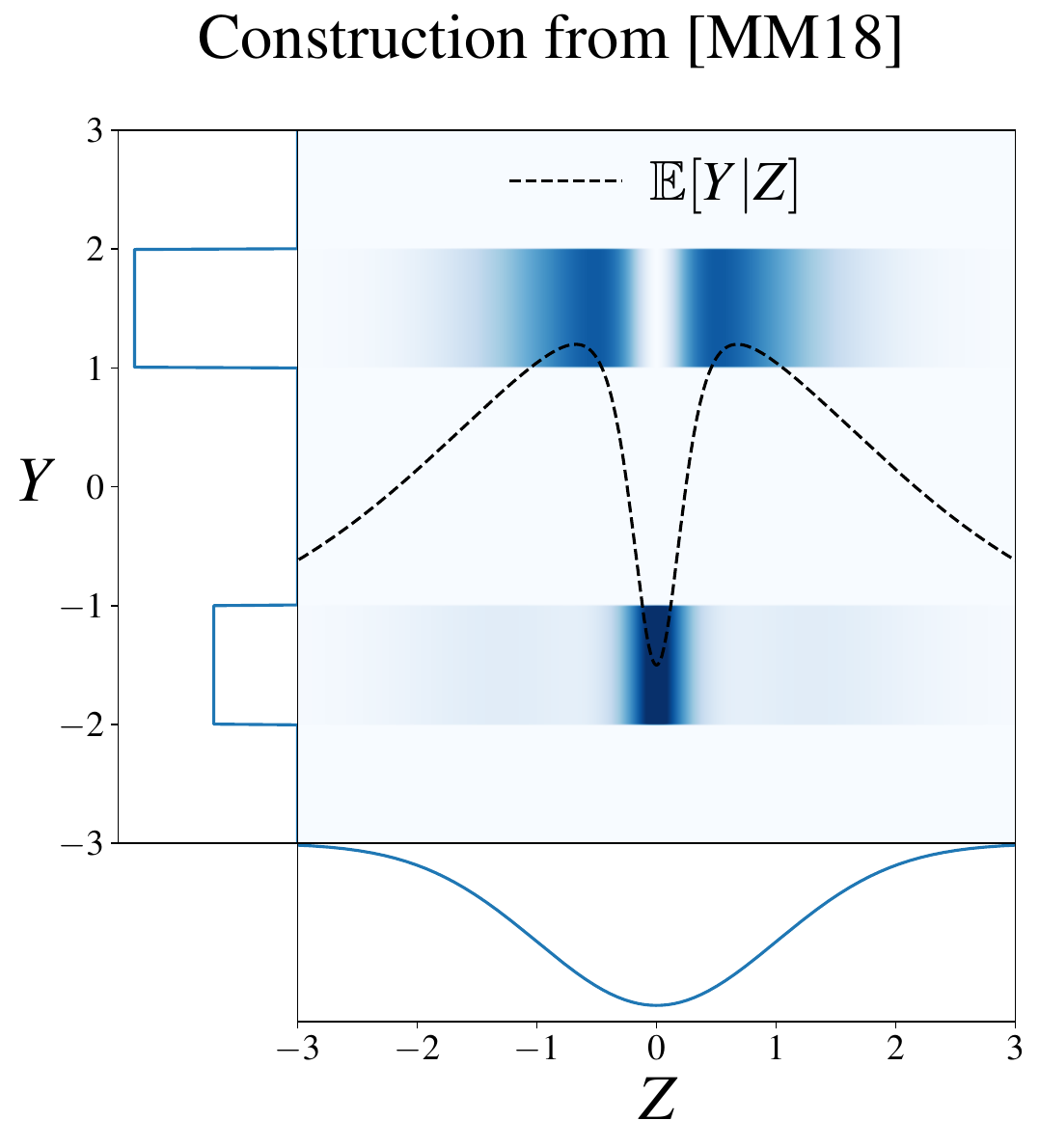}
    \includegraphics[height=0.3\textwidth]{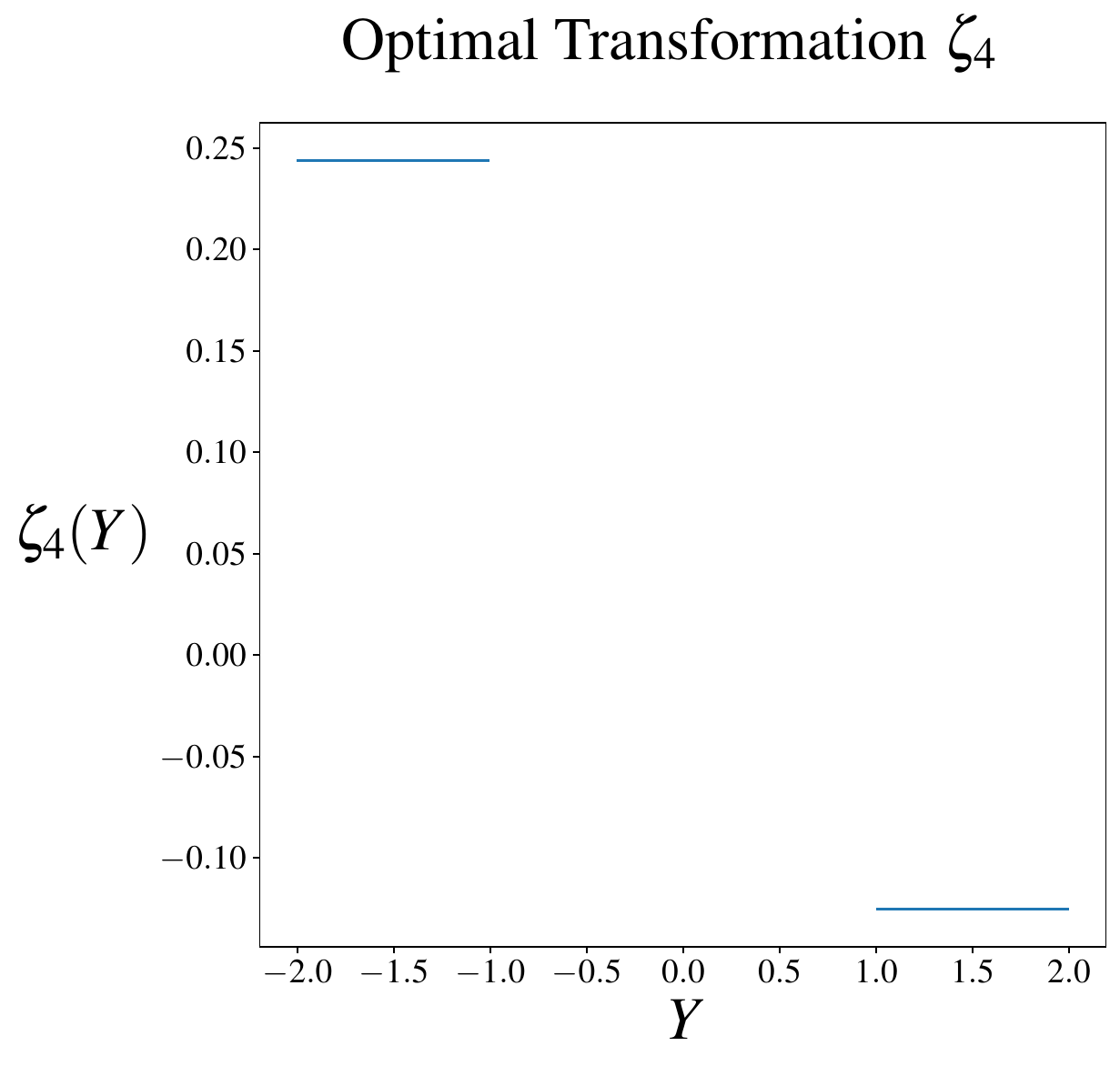}
    \includegraphics[height=0.3\textwidth]{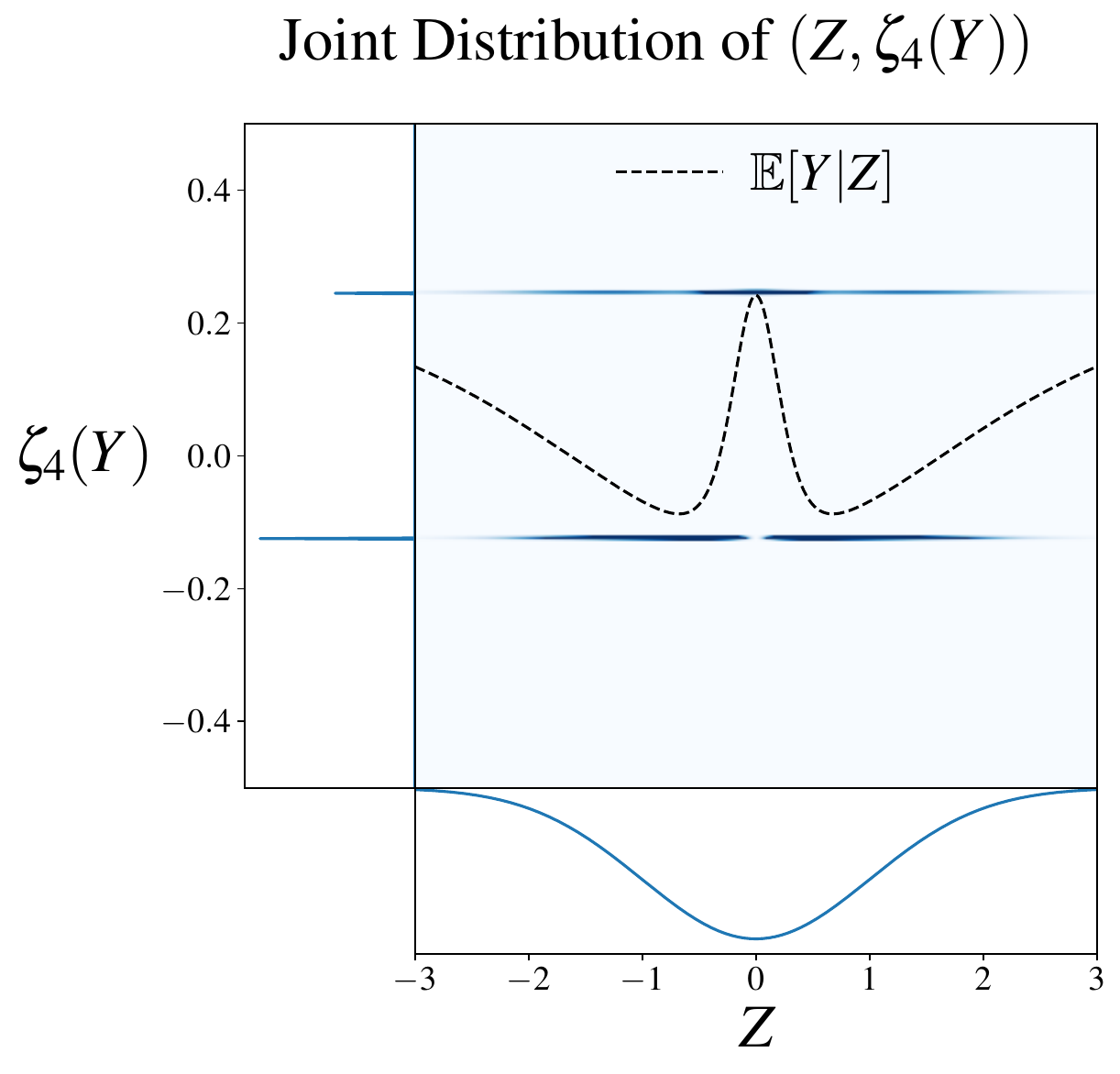}
    \caption{We plot three examples of a joint distribution $\P$ of $(Z,Y)$, the \emph{witness} function $\xi_\k(y)$, and the joint distribution of $(Z,\xi_k(Y))$. In the first example, $Y = c_1 x + \sin(c_2 Z)$ for constants $c_1,c_2$ such that $\beta_1 = 0$. The transformation $\xi_1$ zeros out the bulk and amplifies the caps of the curve in order to lower the information exponent from $\kk(\P) = 2$ to $\kk((\mathrm{Id} \otimes \zeta_1)_\#\P) = 1$. As a result, $\k(\P) = 1$. In the second example, the model $Y = Z \xi$ has multiplicative Gaussian noise and $E[Y|Z] = 0$ so this model has $\kk(\P) = \infty$. The transformation $\zeta_2$ interpolates between $\sqrt{|Y|}$ for $|Y| \approx 0$ and $|Y|$ for $|Y|$ farther from $0$. The transformed distribution (right column) now has $\kk((\mathrm{Id} \otimes \xi_k)_\# \P) = 2$ so $\k(\P) = 2$. The third example is the distribution used in \cite{mondelli2018fundamental} as an example where $\k > 2$. In this case, we verify that $\k = \kk = 4$ so the tight sample complexity for the single index model corresponding to this choice of $\P$ is $n \gtrsim d^2$.}
    \label{fig:examples_zeta}
\end{figure}

\section{The Generative Exponent}

\label{sec:sqexp}

Let us start by defining the information exponent of \citep{arous2020online, dudeja2018learning} in our framework. We begin with Parseval's identity for Hermite polynomials. We define $\sigma$ such that $\sigma(Z) := \E_\P[Y | Z]$ is the conditional expectation of $Y$ given $Z$, thus $\sigma \in L^2(\R,\gamma)$ and we have
\begin{restatable}[Spectral Variance Decomposition]{fact}{variancedecomp}
The variance of $\sigma(Z)$ verifies the expansion
\label{fact:variance_decomp}
\begin{align}
    \mathrm{Var}_\P[\sigma(Z)] = \sum_{l \ge 1} \beta_l^2 \qq{where} \beta_l := \E_\P[Y h_l(Z)]~.
\end{align}
\end{restatable}

Therefore if $\sigma$ is not constant, there exists $l \ge 1$ such that $\beta_l \ne 0$. We define the information exponent $\kk$ as the first such $l$:
\begin{definition}[Information Exponent revisited]\label{def:information_exponent}
The \emph{information exponent} of $\P \in \mathcal{G}$ is defined by: 
\begin{align}
    \label{eq:infexp}
    \kk(\P) := \min\{l \ge 1 ; \beta_l \neq 0 \} \qq{where} \beta_l := \E_\P[Y h_l(Z)].  
\end{align}
    
\end{definition}

Note that \Cref{def:information_exponent} only depends on $\P$ through the conditional expectation $\sigma$, and is therefore agnostic to any form of label noise. 
Below, we define another index, the \textit{\sqexp}. As for the information exponent, its definition follows from an expansion, but for this exponent, we do it through the expansion of the $\chi^2$-mutual information of $(Z,Y) \sim \P$, i.e. the $\chi^2$-divergence between $\P$ and the product of its marginals $I_{\chi^2}[\P] := \mathbb{D}_{\chi^2}[\P || \P_z \otimes \P_y]$:
\begin{restatable}[Mutual Information Decomposition]{lemma}{chiinfo}
\label{lem:chiinfo}
We have the following expansion
\begin{align}
    I_{\chi^2}[\P] = \sum_{k \ge 1} \lambda_k^2 \qq{where} \lambda_k := \norm{\zeta_k}_{\P_y} \qand \zeta_k := \E[h_k(Z)|Y].
\end{align}
\end{restatable}
The proof of this Lemma is postponed to Section~\ref{app:proof_of_section_sqexp} of the Appendix.
We note that because $\E_\P[Y^2] < \infty$, the conditional expectation $\zeta_k := \E[ h_k(Z) | Y]$ is well defined. In addition, $\zeta_k \in L^2(\R,\P_y)$ because for each $k$,
$ \norm{\zeta_k}_{\P_y}^2 \leq \E_Y \E_{Z | Y} [h_k(Z)^2] = \E_Z[h_k(Z)^2] = 1.$
Therefore when $\P$ is not a product measure, as is required by \Cref{def:single-index_model}, we have that $I_{\chi^2}[\P] > 0$ so there exists $k \ge 1$ such that $\lambda_k \ne 0$. We define the generative exponent $\k$ as the first such $k$:
\begin{definition}[Generative Exponent]\label{def:sq_exponent}
The \sqexp of $\P \in \mathcal{G}$ is defined by:
\begin{align}
\label{eq:genexp}
    \k(\P) := \min\{k \ge 1 ; \lambda_k \ne 0\} \qq{where} \lambda_k := \norm{\zeta_k}_{\P_y} \qq{and} \zeta_k := \E_\P[h_k(Z)|Y].
\end{align}
\end{definition}

Observe that $\beta_k$ and $\zeta_k$ are related by $\beta_k = \langle y, \zeta_k \rangle_{\P_y}$. We can thus reinterpret the exponent $\kk(\P)$ as the smallest $k$ such that $\zeta_k$ has non-zero correlation with linear functions in $L^2(\R,\P_y)$, capturing the correlational structure of CSQ queries. This also implies that the \sqexp is at most the \infexp, i.e. $\k(\P) \le \kk(\P)$ because $\lambda_k = 0 \Rightarrow \beta_k = 0$.

The function $\zeta_k(y)$ measures the $k$-th Hermite moment of the conditional `generative' process $Z | Y=y$; as such, the fact that $\lambda_k > 0$ `witnesses' $\chi^2$-mutual information carried by order-$k$ moments, and reciprocally $\lambda_k = 0$ indicates the absence of order-$k$ exploitable information, even after conditioning on the observed labels. This is illustrated in Figure \ref{fig:examples_zeta} and formalized by our SQ and low-degree lower bounds; see next section. 

\begin{remark}
    Observe that it is possible to build distributions $\P$ such that $\k(\P)<\infty$, but $\kk(\P) = \infty$, i.e. such that $\E[Y|Z]$ is constant. Consider for example $\P = \varphi_\# \gamma_2$, with $\varphi(z,\xi)=(z,z\xi)$. We have $\E[Y|Z]=0$, hence $\kk(\P) = \infty$. However, $\k(\P) = 2$ (see \Cref{fig:examples_zeta}). This is reflected by the fact that squaring the labels reduces the problem to noisy phase retrieval as~$\E[Y^2|Z=z] = z^2$.
\end{remark}

Finally, we show that the \sqexp can be expressed as the smallest information exponent over all possible transformations of the label by squared-integrable functions: 
\begin{restatable}[A Variational Representation]{proposition}{sqexpvariational}
\label{lem:composition_lemma}
    The \sqexp $\k(\P)$ can be written as:
    \begin{align}
        \k(\P) = \inf_{\mathcal{T} \in L^2(\P_y)} \kk((\mathrm{Id} \otimes \mathcal{T})_\#  \P)~.
    \end{align}    
\end{restatable}
This provides a `user-friendly' characterization of the \sqexp: for any polynomial $Q$ of degree $k < \k(\P)$ and any measurable test function $\mathcal{T} \in L^2(\R,\P_y)$, ~\Cref{lem:composition_lemma} tells that we must have $\E_{\P}[ \mathcal{T}(Y) Q(Z)]=0$. %
In addition, because $\kk$ only depends on the conditional expectation $\E[\mathcal{T}(Y)|Z]$, this variational representation extends to non-deterministic channels, i.e. if $Y \to Y'$ is a Markov chain with $\E[(Y')^2] < \infty$ then $\k(Z,Y') \ge \k(Z,Y)$. In particular, no post-processing of $Y$, either random or deterministic, can reduce the \sqexp.

We conclude this section with some representatitve examples for deterministic models, illustrated in Figures \ref{fig:hermites} and \ref{fig:examples_zeta}.
\begin{restatable}[Explicit Examples of \sqexp]{example}{examplegen}
\label{ex:sqexp}
We give the following explicit examples:
\begin{enumerate}[label = (\roman*)]
    \item For $\sigma$ a polynomial, we have $\k(\sigma) \leq 2$ and $\k(\sigma)=2$ iff $\sigma$ is even. In particular, $\k( h_j) = 1$ if $j$ is odd and $\k(h_j) = 2$ if $j$ is even.
    \item For $\sigma(z) = z^2 e^{-z^2}$, we have $\k(\sigma)=4$. 
    \item From \citep[Remark 3]{mondelli2018fundamental}, if $y \in \{-1,1\}$ is boolean with $P(Y=1|Z=z) = \E_{W \sim \gamma} [\tanh(c_1 z)^2 - \tanh(c_2 z)^2]$ for carefully chosen $c_1,c_2$ then $\k(\sigma) = 4$. 
\end{enumerate}
\end{restatable}
 
Finally, an immediate consequence of ~\Cref{lem:composition_lemma} is the following:
\begin{corollary}[invariance to bijections]
    If $\varphi:\R \to \R$ is a bijection such that $\varphi, \varphi^{-1} \in L^2(\P_y)$, then $\k(\sigma) = \k(\varphi \circ \sigma)$. 
\end{corollary}

\begin{figure}
    \centering
    \includegraphics[height=0.29\textwidth]{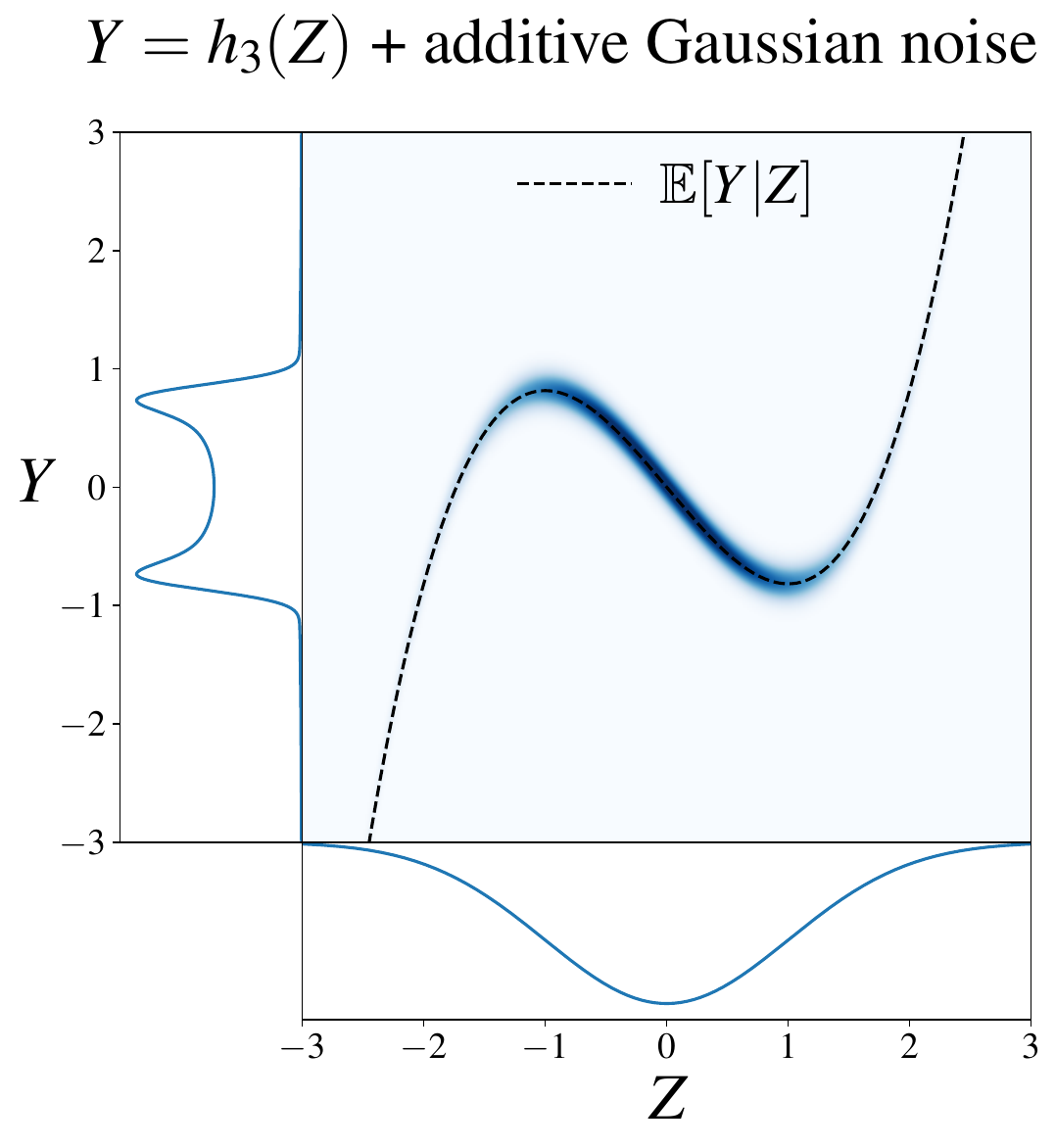}
    \includegraphics[height=0.29\textwidth]{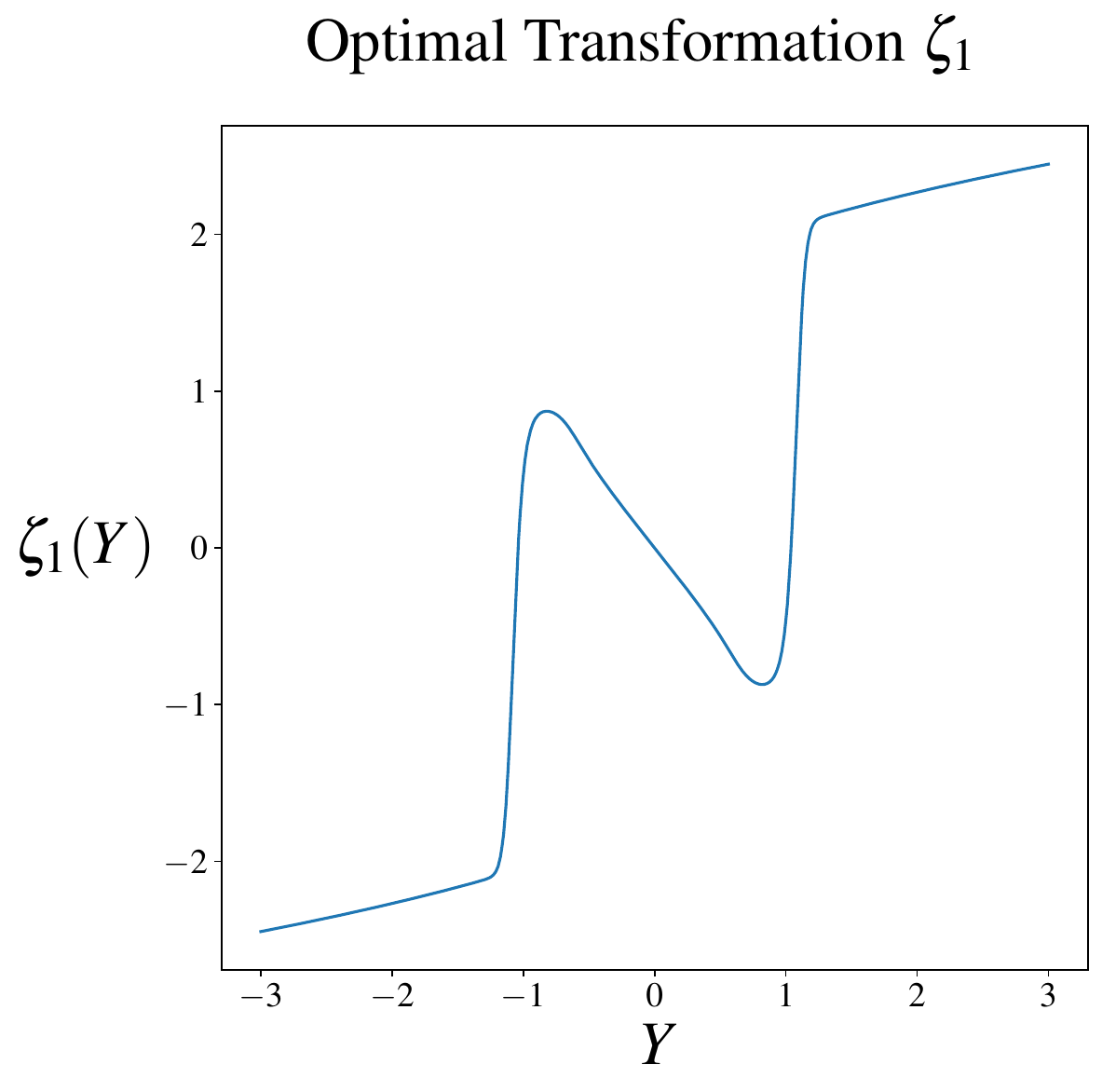}
    \includegraphics[height=0.29\textwidth]{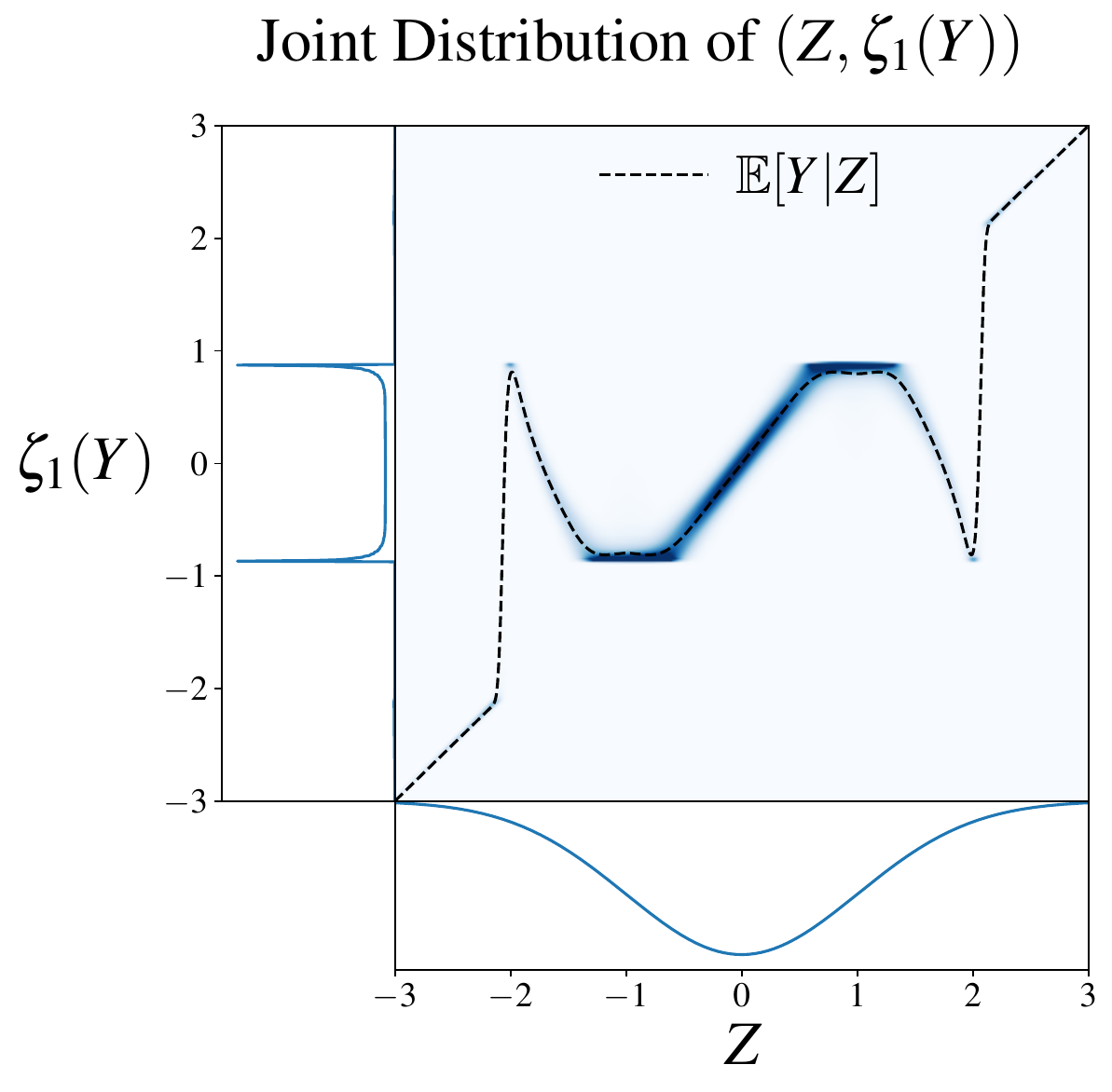} \\
    \includegraphics[height=0.29\textwidth]{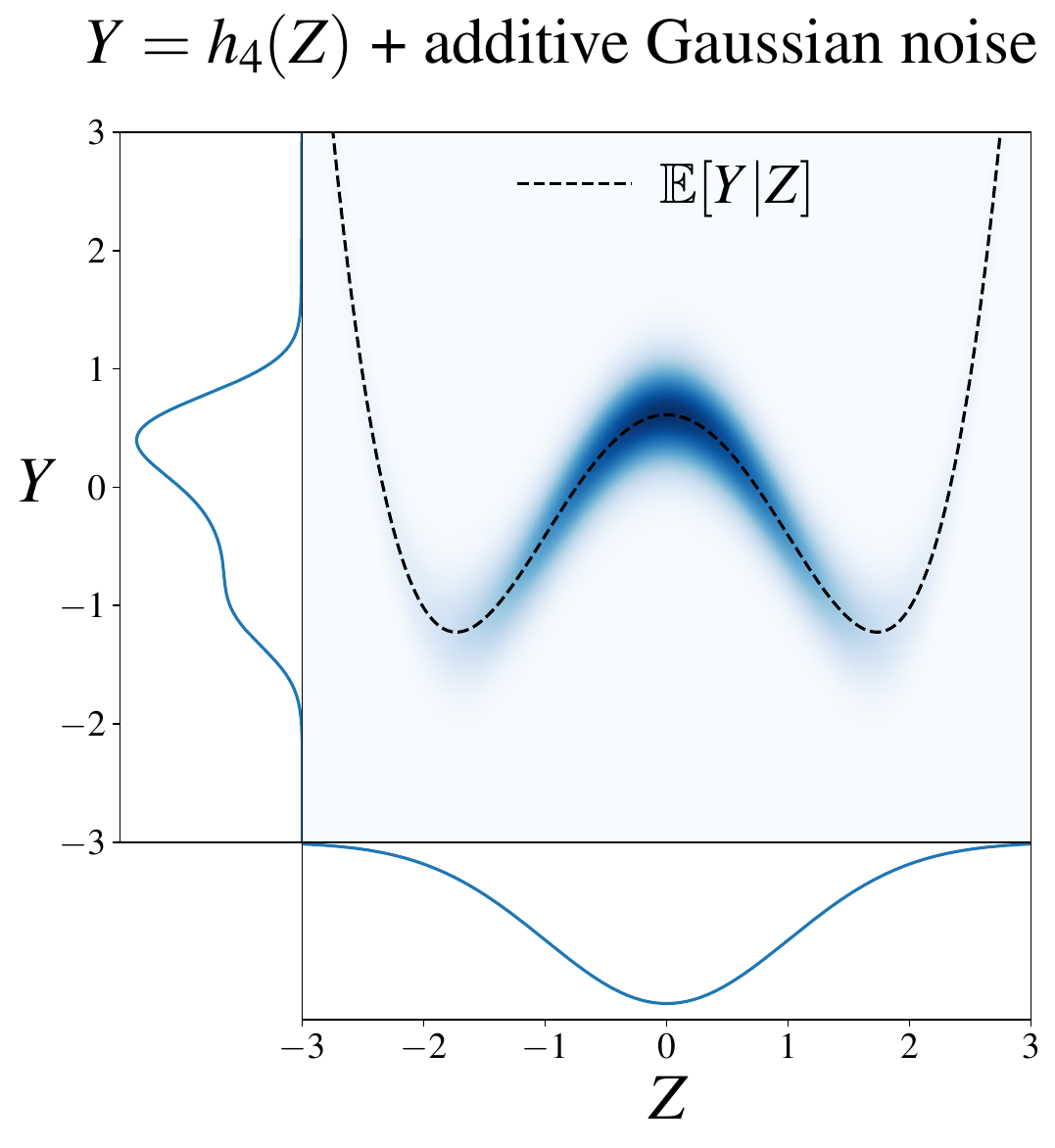}
    \includegraphics[height=0.29\textwidth]{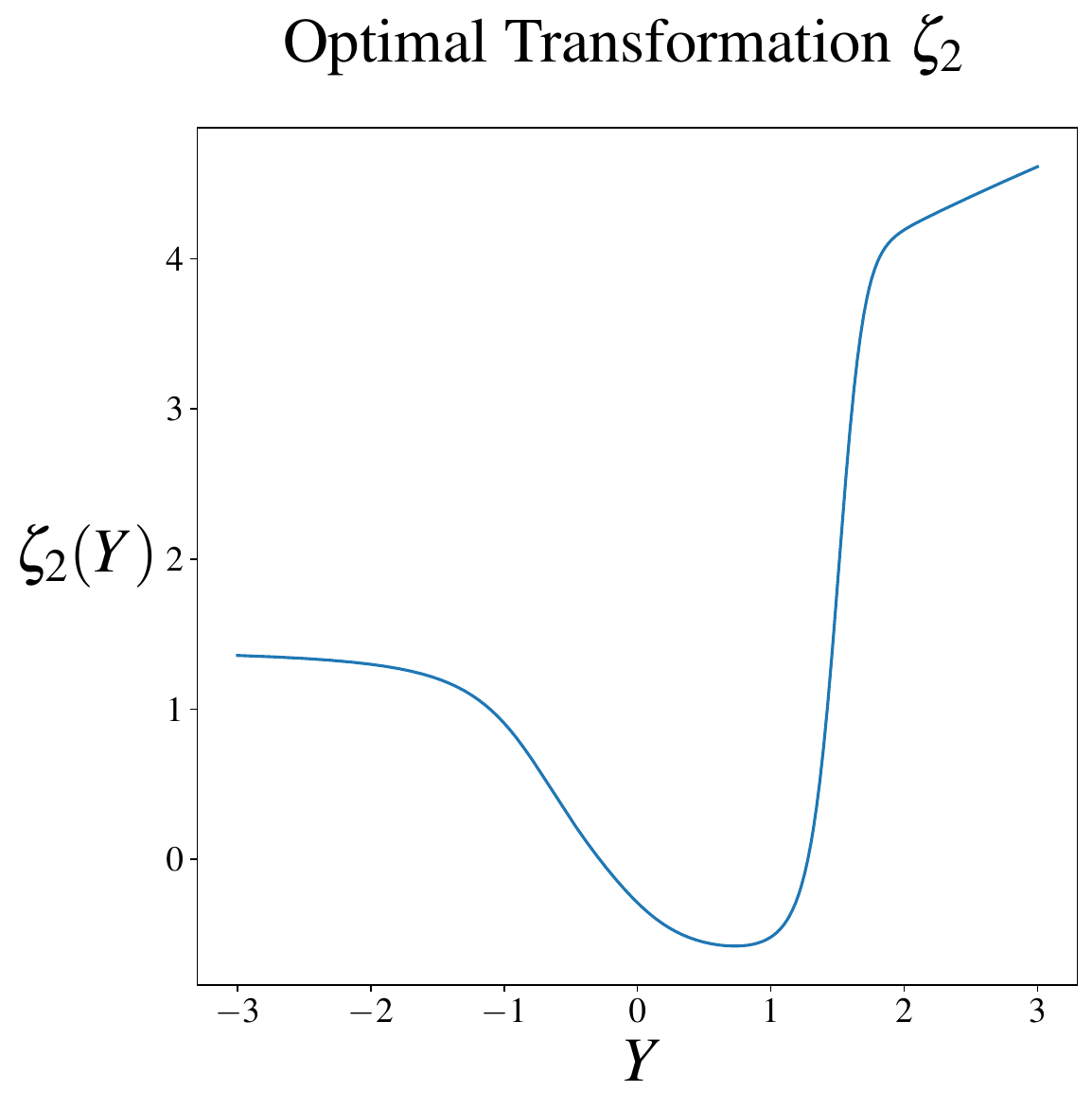}
    \includegraphics[height=0.29\textwidth]{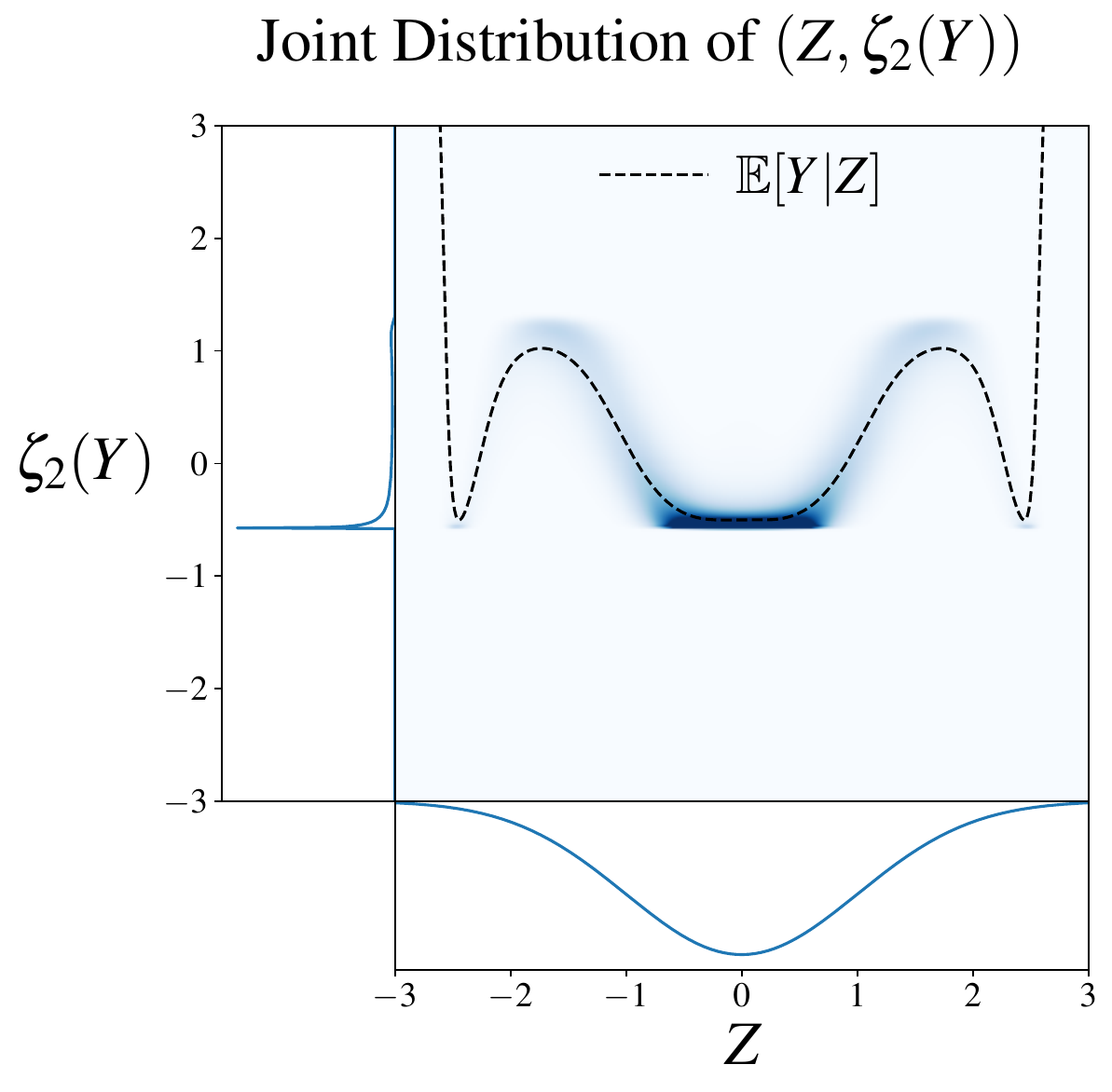} \\
    \includegraphics[height=0.29\textwidth]{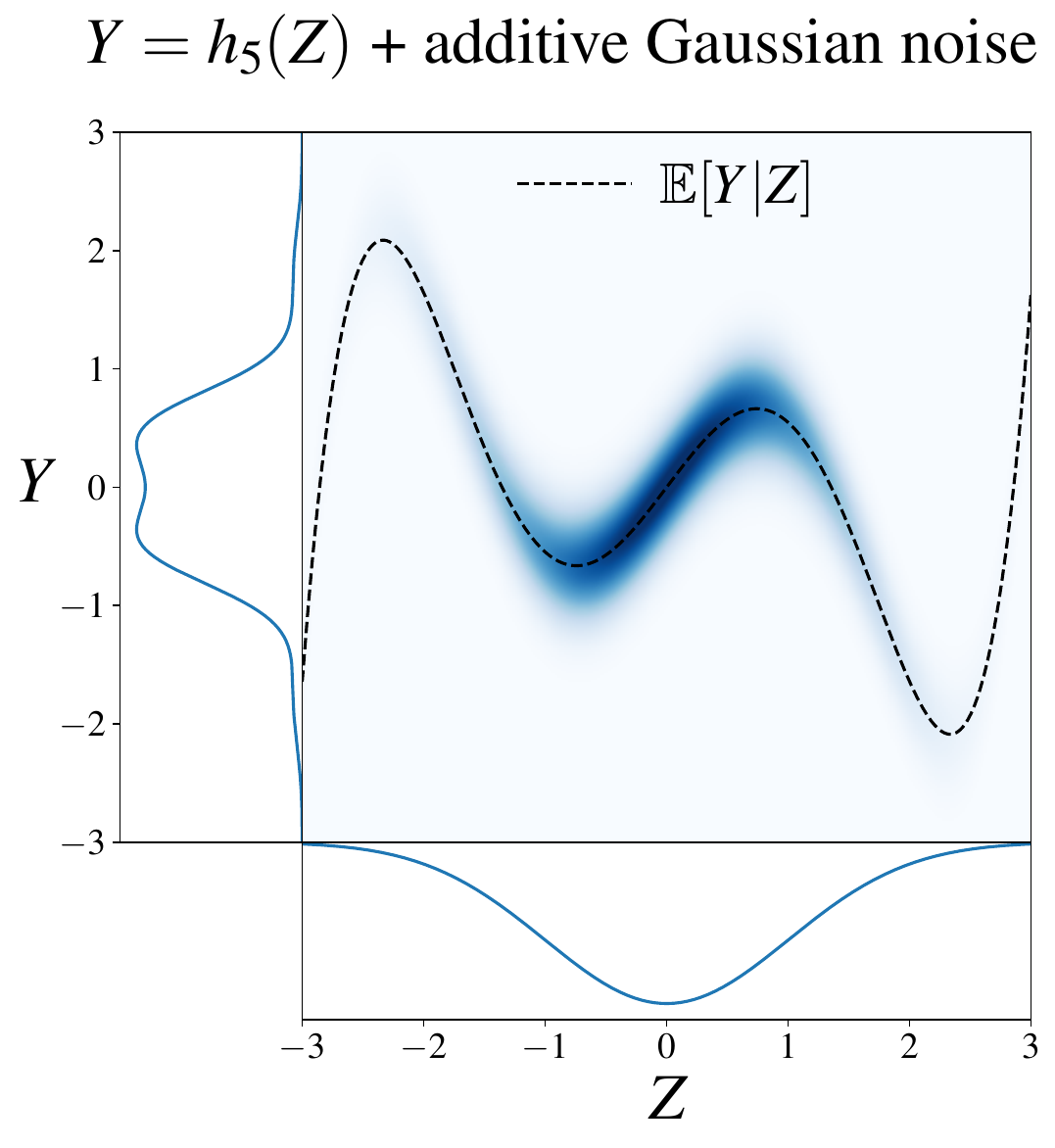}
    \includegraphics[height=0.29\textwidth]{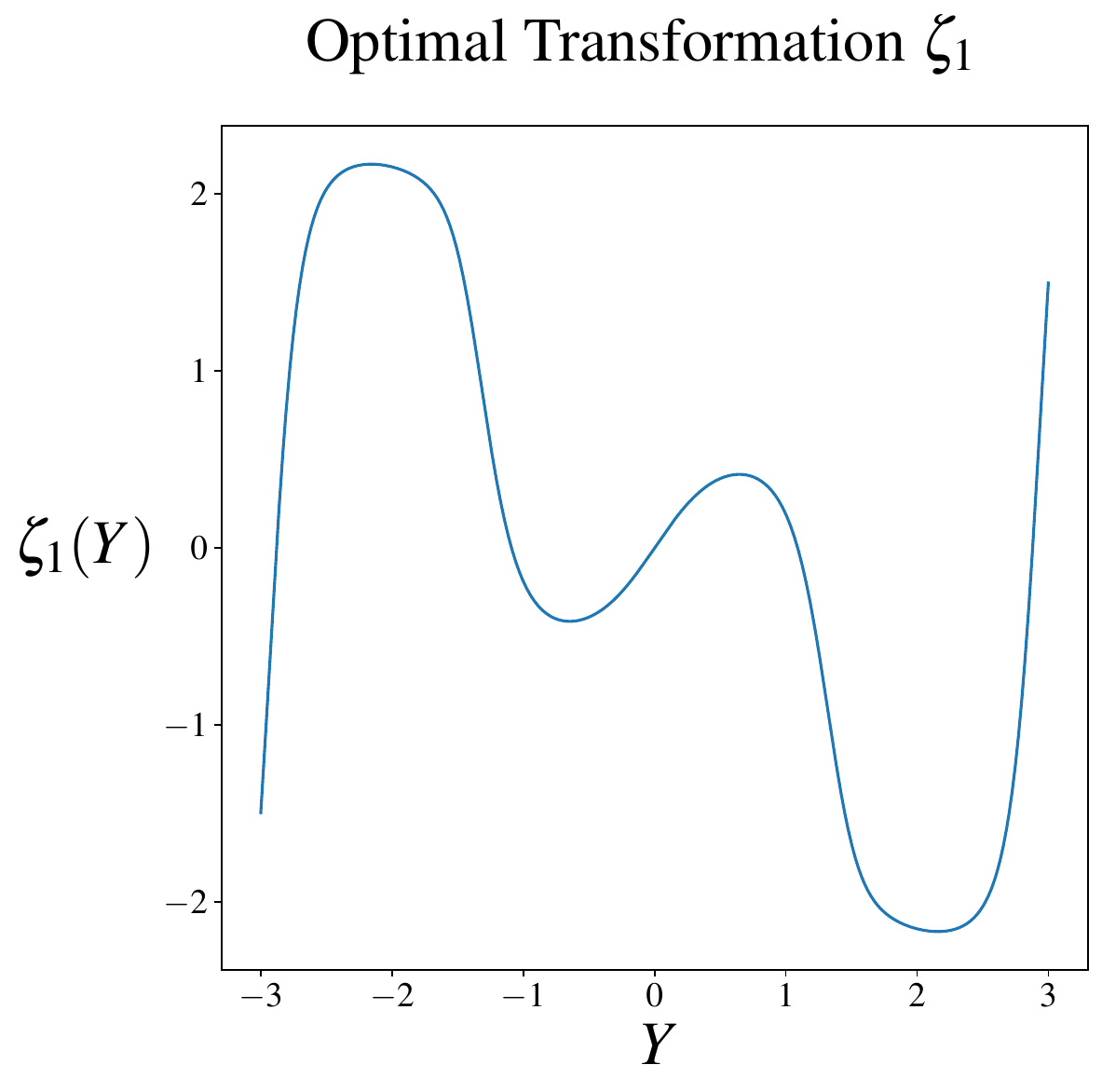}
    \includegraphics[height=0.29\textwidth]{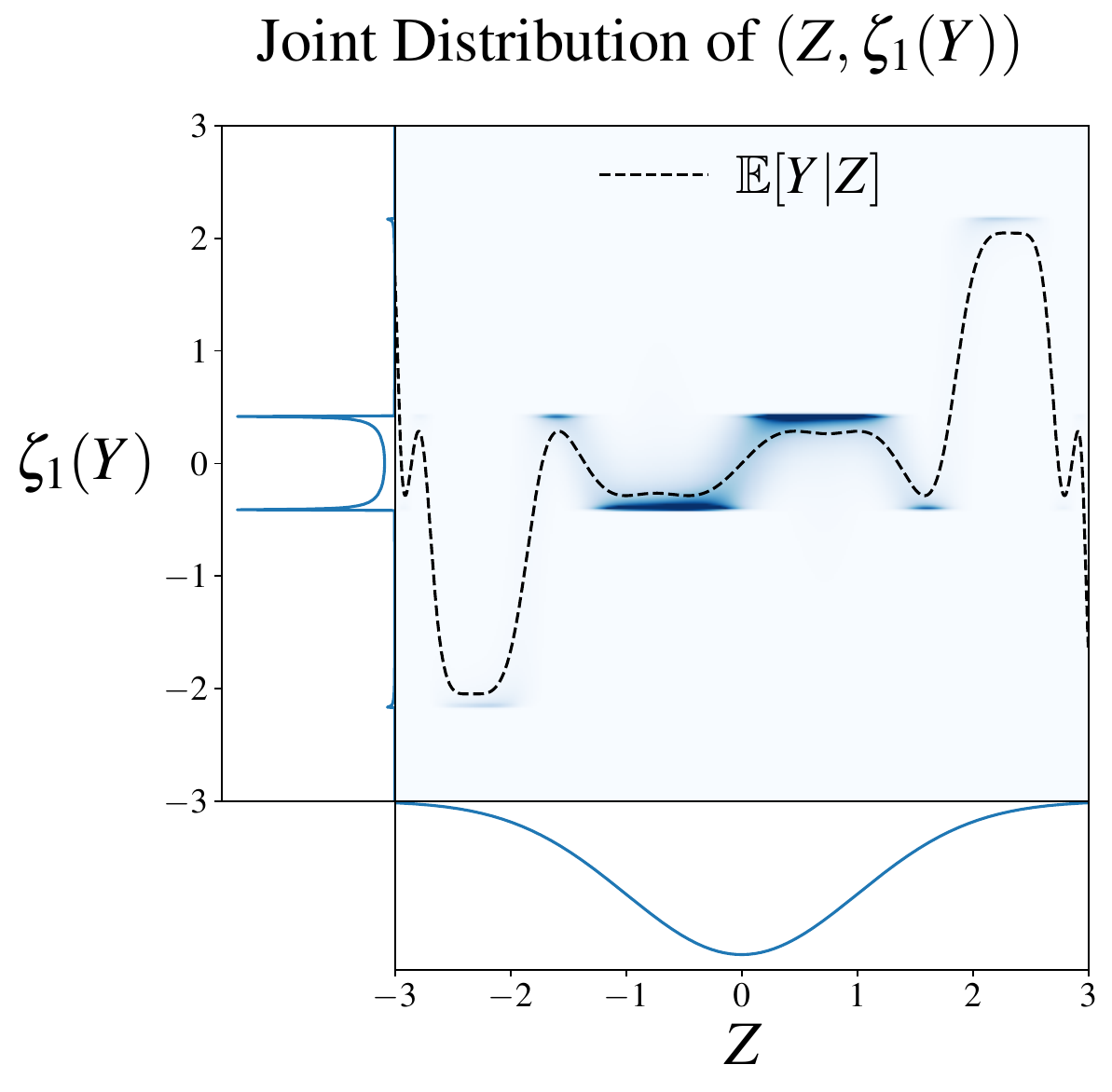} \\
    \includegraphics[height=0.29\textwidth]{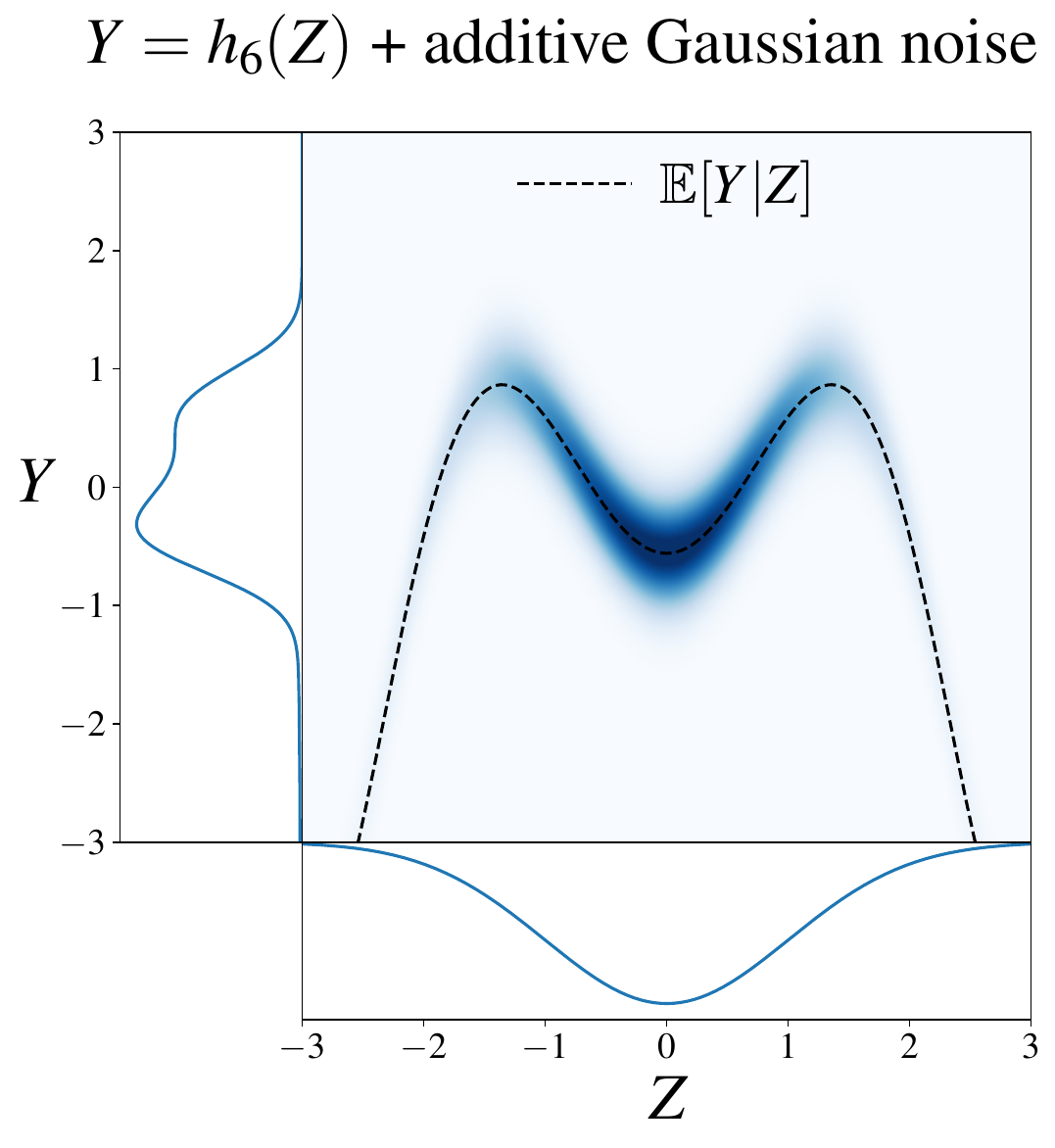}
    \includegraphics[height=0.29\textwidth]{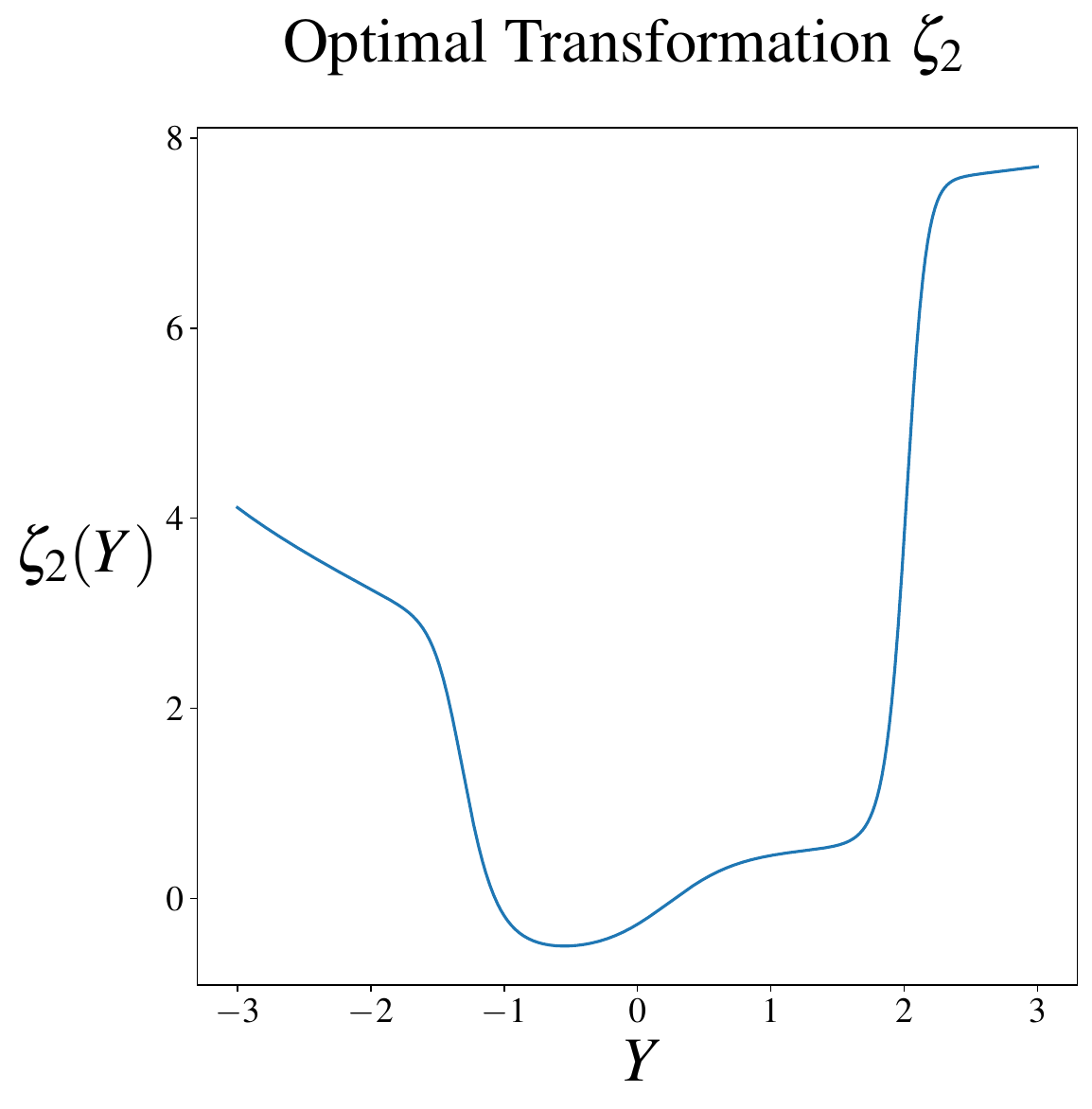}
    \includegraphics[height=0.29\textwidth]{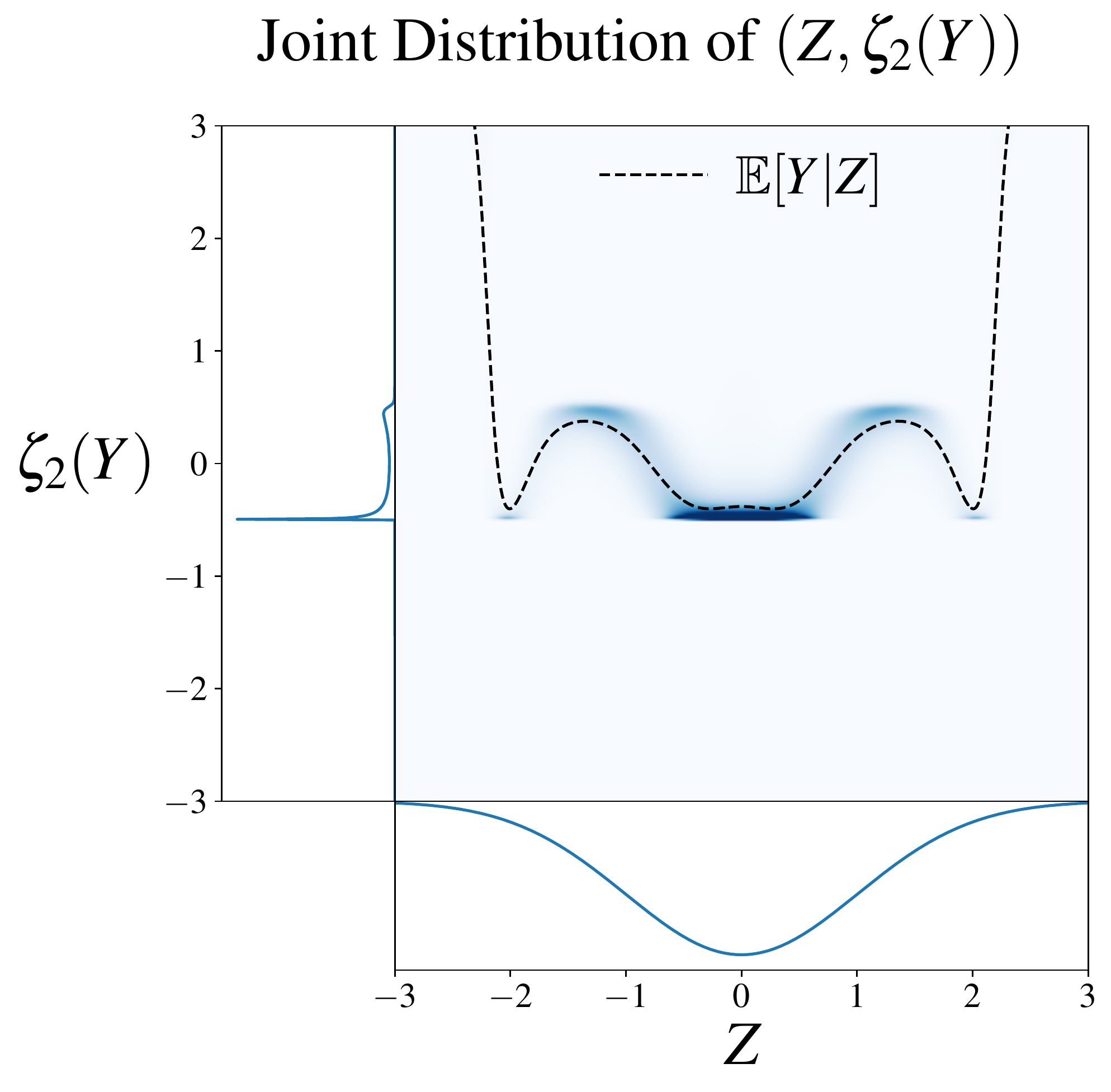}
    \caption{For every $k$, $Y=h_k(Z)$ has \sqexp $1$ if $k$ is odd and $2$ if $k$ is even. In particular, the difficulty of learning the single index model defined by $\P = (\mathrm{Id} \otimes h_k)_\# \gamma_1$ does not grow with $k$.}
    \label{fig:hermites}
\end{figure}

\section{Computational Lower Bounds}
\label{sec:sqlower}

\subsection{From CSQ to SQ lower bounds}
Recall that the CSQ complexity of Gaussian single-index models is established 
by considering the correlation $\E_X \left[ \E[Y|X] \cdot \E[Y'|X] \right]$, where $(X,Y) \sim \PP_w$ and $(X, Y') \sim \PP_{w'}$ are two hypothesis in the class. This correlation admits a closed-form representation via the Ornstein-Ulhenbeck semigroup in $L^2(\R, \gamma_1)$ and the Hermite decomposition of $\sigma(z) = \mathbb{E}_\P[Y|Z=z] = \sum_{k} \beta_k h_k(z)$:
\begin{align}
\label{eq:CSQ_basic}
\E_X \left[ \E[Y|X] \cdot \E[Y'|X] \right] = \sum_{k \geq \kk} m^k \beta_k^2~,~\text{ where } m = w \cdot w'~.
\end{align}
For randomly chosen $w,w' \sim \mathrm{Unif}(S^{d-1})$, the correlation $m = w \cdot w'$ is of order $d^{-1/2}$ so the $k$th term in this expansion is of order $d^{-k/2}$. Therefore, the first nonzero term of this expansion, which is of order $d^{-\kk/2}$, dominates the search problem. This is used in the proof of the CSQ lower bound that CSQ algorithms require $n \gtrsim d^{\kk/2}$ samples to learn $w^\star$ \citep{damian2022neural,abbe2022merged}. 

However, unlike the CSQ complexity, which is determined by average pairwise correlations, the SQ-complexity is determined by the $\chi^2$ symmetrized divergence:
\begin{align}
    \chi^2_0 ( \PP_{w} ,  \PP_{w'}) := \E_{\mathbb{P}_0} \left[ \frac{\mathrm{d}\PP_w}{\mathrm{d}\mathbb{P}_0}\frac{\mathrm{d}\PP_{w'}}{\mathrm{d}\mathbb{P}_0}\right] -1 \qq{where} \mathbb{P}_0 = \gamma_d \otimes \P_y~.
\end{align}
Remarkably, one can exhibit an analog of (\ref{eq:CSQ_basic}), where the coefficients $\beta_k$ are replaced by $\lambda_k$:

\begin{restatable}[$\chi^2$ Representation]{lemma}{chikey}
\label{lem:chi2key}
      Let $w,w' \in S^{d-1}$, and denote $m = w \cdot w'$, then we have  %
    \begin{align}
        \chi^2_0(\PP_{w},\PP_{w'}) & = \sum_{k \ge \k} m^k \lambda_k^2~.
    \end{align}  
\end{restatable}

The proof of this Lemma is postponed to Section~\ref{app:sec_proof_section_sqlower} of the Appendix. As above, $m$ is of order $d^{-1/2}$ so this sum is dominated by the first nonzero term in the sequence, which is given precisely by the \sqexp $\k = \k(\P)$.

\subsection{SQ Framework for Single-Index Models}
\label{sec:sqlowersub}

Appendix \ref{app:sq_app} reviews the basic framework 
to establish SQ query complexity for search problems, 
which has been used for a variety of statistical inference tasks, e.g.
\citep{diakonikolas2017statistical,dudeja2021statistical}, and instantiates it for the particular
setting of the single-index model; see Section~\ref{sec:SQinstance_singleindex}.
We consider the statistical query oracle $\mathrm{VSTAT}(n)$ which responds to query $h: \R^d \times \R \to [0,1]$ with $\hat h$ which satisfies:
$|\hat h - p| \le \sqrt{p(1-p)/n} + 1/n$, where $p = \E_{X,Y}[h(X,Y)]$.

The primary challenge in proving an SQ lower bound is computing the \emph{SQ-dimension}; see Definition \ref{def:sq_dim}. %
This is enabled in our setting thanks to the key Lemma \ref{lem:chi2key}. This result leads to a control of the relative $\chi^2$-divergence of the form $ \chi^2_0(\PP_{w},\PP_{w'}) \le \lambda_\k^2 m^\k + \frac{m^{k^\star}}{1-m}$, with $m = w \cdot w'$ (Lemma \ref{lem:chi_easyupper}), which yields the following SQ lower bound (proved in \Cref{app:proof_sq_lower}):
\begin{restatable}[Statistical Query Lower Bound]{theorem}{sqlowerthm}
\label{thm:sq_lower_bound}
    Assume that $(X,Y)$ follow a Gaussian single-index model (\Cref{def:single-index_model}) with \sqexp $\k$. Let $q = d^r$ for any $r \le d^{1/4}$ and assume that $\lambda^2_\k \ge r d^{-1/2}$. Then there exists a constant $c_\k$ depending only on $\k$ such that to return $\hat w$ with $\abs{\hat w \cdot w^\star} \ge \tilde \omega(d^{-1/2})$, any algorithm using $q$ queries to $\mathrm{VSTAT}(n)$ requires:
    \begin{align}
    \label{eq:sqlowerbound}
        n \ge \frac{c_\k}{\lambda^2_\k} \left(\frac{d}{r^2}\right)^{\frac{\k}{2}}~.
    \end{align}
\end{restatable}
Therefore under the standard heuristic that $\mathrm{VSTAT}(n)$ measures the algorithmic complexity of an algorithm using $n$ samples, any algorithm that uses $\poly(d)$ queries must use $n \gtrsim d^{\k/2}$ samples. A remarkable feature of \cref{eq:sqlowerbound} is the absence of polylog factors; this is achieved thanks to an improved construction of the discrete hypothesis set $\mathcal{D}_D$ (\Cref{def:sq_dim}) using spherical codes; see \Cref{lem:many_orthogonal_vecs}.

\begin{remark}[Choice of Null Model]
In the proof of \Cref{thm:sq_lower_bound}, we compare $\PP_w$ with the null model $\PP_0 := \gamma_d \otimes \P_y$. As a result, Theorem 3.2 primarily measures the detection threshold. Indeed, \cite{dudeja2021statistical} observed a detection/estimation gap for the related problem of tensor PCA and proved that SQ algorithms require $n \gtrsim d^{\frac{\k}{2} + 1}$ for estimation. We leave open the possibility that such a gap exists under the SQ framework for our setting as well. However, because \Cref{alg:partial_trace} succeeds in estimating $w^\star$ with $n \asymp d^{\k/2}$ (see \Cref{thm:optimal_sq_alg}) this would be an artifact of the SQ framework.
\end{remark}

\subsection{The Low Degree Polynomial Method}
\label{sec:low_degree}

In this section we prove a lower bound for the class of \emph{low-degree polynomials}, which has been used to argue for the existence of statistical-computational gaps in a wide variety of problems \cite{kunisky2019notes, bandeira2022franz, wein2023average}. This has been formalized through various versions of the low degree conjecture \citep[Hypothesis 2.1.5, Conjecture 2.2.4]{hopkins2018statistical}, for which we state an informal version for distinguishing between two distributions, $\PP$ (signal) and $\PP_0$ (null). We define the likelihood ratio $\L := \frac{d\PP}{d\PP_0}$ and the low-degree likelihood ratio $\L_{\le D}$ by the orthogonal projection in $L^2(\PP_0)$ of $\L$ onto polynomials of degree at most $D$. The low-degree conjecture states that low-degree polynomials (i.e. $D \approx \log(n)$) act as a proxy for polynomial time algorithms:

\begin{conjecture}[Low-Degree Conjecture, informal] \footnote{The conjecture is stated for inference problems with some level of robustness to noise; this excludes known failures of LD methods to capture computational hardness in problems with algebraic structure; see \cite{zadik2022lattice, diakonikolas2022non}.} \label{conjecture:low_deg}
    Let $D = \omega(\log(d))$. Then if $\norm{\L_{\le D}}_{\PP_0} = 1 + o_d(1)$, no polynomial time algorithm can distinguish $\PP$ and $\PP_0$. If $\norm{\L_{\le D}}_{\PP_0} = O_d(1)$, they can do so with only constant probability.
\end{conjecture}

By definition, we have that $\| \L \|_{L^2(\PP_0)}^2 = 1 + \mathbb{D}_{\chi^2}[\PP||\PP_0] = 1 + I_{\chi^2}[\P]$, which hints at the fact that the decomposition \cref{lem:chi2key} and its associated \sqexp will also be driving the behavior of the Low-Degree estimator. This is indeed the case: our main result in this section computes the low degree likelihood ratio $\norm{\L_{\le D}}_{\PP_0}$ and shows that it remains near $1$ unless $n \gtrsim d^{\k/2}$:
\begin{restatable}[Low-Degree Method Lower Bound]{theorem}{lowdegthm}
\label{thm:low_degree}
	Let $\P$ be a single index model, let $\PP$ be the distribution of $\PP_w$ when $w$ is chosen randomly from $\unif(S^{d-1})$. Let $X \in \R^{n \times d}$ and $Y \in \R^n$ denote a sequence of $n$ i.i.d. inputs and targets drawn from $\P$. Let $\PP_0 := \gamma_d \otimes \mathsf{P}_y$ denote the null distribution. Let $\L(x,y)$ denote the likelihood ratio $\frac{d\PP}{d\PP_0}(x,y)$ and let $\L_{\le D}(x,y)$ denote the orthogonal projection in $L^2(\PP_0)$ of $\L(x,y)$ onto polynomials of degree at most $D$ in $x$. Then if $\frac{D}{\lambda_k^2} \ll \sqrt{d}$ and $\delta := \frac{n}{d^{\k/2}},$
	\begin{align*}
		\norm{\L_{\le D}}_{\PP_0}^2 &= (1+o_d(1))\sum_{j=0}^{\lfloor \frac{D}{k^\star} \rfloor} \1_{2 \mid k^\star j} (k^\star j-1)!! \frac{(\lambda_{k^\star}^2 \delta)^{j}}{j!}.
	\end{align*}
\end{restatable}
The proof is in \Cref{app:low_degree_lower}. This provides an immediate corollary which proves impossibility results for weak and strong recovery for all polynomial time algorithms under the low-degree conjecture (\Cref{conjecture:low_deg}):
\begin{restatable}{corollary}{lowdegcorollary}
\label{corollary:low_deg_weak_strong}
Let $\PP,\PP_0,\L$ be as in \Cref{thm:low_degree} and assume $\k > 1$. Then,
    \begin{itemize}
        \setlength\itemsep{-2ex}
        \item \textbf{Weak Detection:} If $D = \log(d)^2$ and $n \le d^\gamma$ for $\gamma < \frac{\k}{2}$, then $\norm{\L_{\le D}}_{\PP_0} = 1 + o_d(1)$. \\
        \item \textbf{Strong Detection:} For any $D \ll \sqrt{d}$, if $n \le \frac{d^{\frac{\k}{2}}}{D^{\frac{\k}{2}-1}}$, then $\norm{\L_{\le D}}_{\PP_0} = O_d(1)$.
    \end{itemize}
\end{restatable}
This shows that under \Cref{conjecture:low_deg}, $n \gtrsim d^{\frac{\k}{2}}$ samples are necessary for polynomials of degree $D = \log(d)^2$ to distinguish between $\PP$ and $\PP_0$. Note that because recovery is strictly harder than detection, this also implies that recovering $w^\star$ from $\PP$ also requires $n \gtrsim d^{\frac{\k}{2}}$ samples, which is matched by our upper bound (\Cref{thm:optimal_sq_alg}). We also remark that the $\frac{d^{\k/2}}{D^{\k/2-1}}$ threshold in \Cref{corollary:low_deg_weak_strong} matches the optimal known computational-statistical trade-off for tensor PCA \cite{bhattiprolu2017sumofsquares,wein2019kikuchi}.

\subsection{Discussion}
\label{sec:discussion}

\paragraph{Relationship between SQ and LD lower bounds}
\label{para:relationship}
Our statistical query lower bound (\Cref{thm:sq_lower_bound}) and our low-degree polynomial lower bound (\Cref{thm:low_degree} and \Cref{corollary:low_deg_weak_strong}) have similar forms. The statistical query lower bound says that given $d^r$ statistical queries or a polynomial of degree $D$ (which can be computed in $d^D$ time) we need:
\begin{align*}
    n \gtrsim \frac{d^{\k/2}}{r^\k} \qq{or} n \gtrsim \frac{d^{\k/2}}{D^{\k/2-1}}
\end{align*}
samples respectively. A common theme in these lower bounds is the possibility of trading off statistics with computation. In particular, for $\k > 2$, given $n = \delta d^{\k/2}$ samples, $w^\star$ is always learnable given sufficiently many queries ($p$ sufficiently large) or a sufficiently high degree polynomial ($D$ sufficiently large). These bounds differ slightly in their dependence on $r,D$: the SQ bound depends on $r^\k$ while the low-degree bound depends on $D^{\k/2-1}$. We believe the low-degree result is tight, as it exactly matches existing statistical-computational tradeoffs that are known for tensor PCA \citep{bhattiprolu2017sumofsquares,wein2019kikuchi}. This gap is due to the fact that our low-degree lower bound relies on \emph{average} correlations over a prior, while the SQ lower bound relies on \emph{worst case} correlations over a discrete set of points, which we constructed using spherical codes (see \Cref{lem:many_orthogonal_vecs}).

Indeed, the proof of the low-degree lower bound (\Cref{thm:low_degree}) suggests an algorithm for achieving this continuous statistical-computational tradeoff by computing higher order tensors than in \Cref{alg:partial_trace}. Specifically, if $p$ is a sufficiently large even integer, then we consider
\begin{align*}
    T := \frac{1}{\binom{n}{p}} \sum_{S \subset [n], |S|=p} \sym\qty(\bigotimes_{i \in S} y_i \bs{h}_k(x_i)) \in \R^{\k p}.
\end{align*}
We can now apply partial trace to $T$ to reduce it to a matrix $M$ and then return the top eigenvector of $M$. Note that by construction, $\E[T] = (w^\star)^{\otimes \k p}$ so $\E[M] = w^\star {w^\star}^T$. We leave the analysis of $M$ and the question of whether $M$ obeys Gaussian universality (as in \Cref{sec:partialtrace}) to future work.

Finally, we mention the work  \cite{brennan2021statistical}, which provides 
generic `translations' between SQ and low-degree lower bounds under mild conditions, at the expense of a loss in the parameters (number of queries to degree of polynomial). It would be interesting to quantify this loss in our setting.

\paragraph{NGCA \emph{aka} `Gaussian Pancakes'}
The SQ lower bounds that we present here for single index models are similar in essence to the SQ lower bounds for Gaussian hypothesis testing from \cite{diakonikolas2017statistical}. In short, given a non-gaussian univariate distribution $\mu \in \mathcal{P}(\R)$, this problem considers distributions $\mathbb{Q} \in \mathcal{P}(\R^d)$ of the form $\mathbb{Q} =\mathbb{Q}_{\mu, w^\star} = R_\#( \gamma_{d-1} \otimes \mu)$, where $R \in \mathcal{O}_d$ as in Definition \ref{def:single-index_model}. The relevant quantity that governs the SQ-complexity of recovering the planted direction $w^\star$ (where again $w^\star = Re_d$) is the smallest degree $k$ such that $\E_\mu[ h_k(z)] \neq 0$. 
As such, there is a direct reduction from the single index model setting to NGCA: %
\begin{restatable}[Single-Index to NGCA Reduction]{proposition}{singleNGCA}
\label{prop:singleNGCA}
An SQ algorithm that solves NGCA with planted distribution $\mu$ of exponent $k$ and direction $w^* \in S^{d-1}$ yields an efficient SQ algorithm to solve the single-index model with \sqexp $k$ and direction $w^*$.    
\end{restatable}

As a result, our SQ-hardness results imply the hardness results in \cite{diakonikolas2017statistical}. An intriguing question is whether the reduction could go the other way, namely whether the SQ-hardness of Gaussian pancakes implies the SQ-hardness of learning certain single-index models. While we do not answer this question in the present paper, we show in Appendix~\ref{sec:pancake2} a reduction from NGCA to a slight variant of the single-index model. 

\paragraph{Tensor PCA} Both our upper bound (\Cref{thm:optimal_sq_alg}) and lower bound (\Cref{thm:sq_lower_bound}) were heavily inspired by related work in the tensor PCA literature. Specifically, the partial trace estimator in \Cref{alg:partial_trace} is related to the partial trace estimator for tensor PCA \citep{hopkins2016fast,dudeja2024statisticalcomputational} which is the simplest method that achieves the conjectured optimal rate of $n = \Theta(d^{k/2})$ for tensor PCA. The derivation of our lower bounds were also inspired by similar results for tensor PCA, including the SQ lower bound for tensor PCA proved by \cite{dudeja2021statistical}, and the Low-Degree method lower bound of \cite[Theorem 3.3]{kunisky2019notes}. 
 
\paragraph{CLWE and periodic structures}
\cite{song2021cryptographic} proved a lower bound for learning single-index models. They showed that under a cryptographic hardness assumption, there is a noisy single-index problem such that regardless of the samples $n$, no polynomial time algorithm can recover the ground truth $w^\star$. Our results imply a similar lower bound for the special case of SQ algorithms without the cryptographic hardness assumption (\Cref{lem:clwe_lower_bound}). We also note that in the noise-free setting, the situation is different, and in fact there are (non-SQ) polynomial-time algorithms, such as LLL, that break the SQ-lower bound for certain single-index models and related NGCA \citep{song2021cryptographic,zadik2022lattice,diakonikolas2022non} by exploiting periodic structures in the link function.

\section{Partial Trace Estimator}
\label{sec:partialtrace}

\begin{algorithm}
{\small 
\SetAlgoLined
\KwIn{dataset $\mathcal{D} = \{(x_i,y_i)\}_{i=1}^n$, moment $k$}
Draw $n/2$ fresh samples $\mathcal{D}_0$ from $\mathcal{D}$

\uIf{$k$ is even}{
    Construct the partial trace matrix: $M_n \leftarrow \frac{1}{|\mathcal{D}_0|} \sum_{(x,y) \in \mathcal{D}_0} y \bs{h}_k(x)[I^{\otimes \frac{k-2}{2}}]$\;

    $v \gets v_1(M_n)$, the eigenvector corresponding to the top eigenvalue of $M_n$ (in absolute value)
}
\uElseIf{$k$ is odd}{
    Construct the partial trace vector:
    $v_n \gets \frac{1}{|\mathcal{D}_0|} \sum_{(x,y) \in \mathcal{D}_1} y \bs{h}_k(x)[I^{\otimes \frac{k-1}{2}}]$ \;
    
    Normalize $v \gets v_n/\|v_n\|$\;

    \uIf{$k \ge 3$}{
    Run $\log(d)$ steps of tensor power iteration:
    
    \For{$i=1,...,\log(d)$}{
        Draw $n/2^{i+2}$ fresh samples $\mathcal{D}_i$ from $\mathcal{D}$
        
        $v \gets \frac{1}{\mathcal{D}_i} \sum_{(x,y) \in \mathcal{D}_i} y \bs{h}_k(x)[v^{\otimes (k-1)}]$ \;
    
        $v \gets v/\|v\|$\;
    }
    }
}
\SetKwProg{myprog}{}{}{}
\myprog{Run one final step of tensor power iteration:}{
Draw $n/4$ fresh samples $\mathcal{D}_{\text{fin}}$ from $\mathcal{D}$
    
$v \gets \frac{1}{\mathcal{D}_{\text{fin}}} \sum_{(x,y) \in \mathcal{D}_{\text{fin}}} y \bs{h}_k(x)[v^{\otimes (k-1)}]$ \;

$v \gets v/\|v\|$\;
}
\KwOut{$v$}
}
\caption{Tensor Power Iteration with Partial Trace Warm Start}
\label{alg:partial_trace}
\end{algorithm}

We have shown so far that Theorems \ref{thm:sq_lower_bound} and \ref{thm:low_degree} extend the `necessary' criterion presented in Eq. \eqref{eq:basicSQ2} from \citep{mondelli2018fundamental,barbier2019optimal,maillard2020phase} to arbitrary exponents $\k >2$, as indeed a simple calculation shows that  $\mathbb{E} [\mathbb{E} [Z^2-1 | Y]^2] = 2\lambda_2^2$. 
Concerning the `sufficient' direction of Eq.\eqref{eq:basicSQ2}, the algorithms that match the previous lower bound for $\k\leq 2$ suggest a meta-strategy for the general case $\k>2$: first apply a transformation $\mathcal{T}$ of the labels to reduce to correlational queries, and then apply an optimal CSQ algorithm. Moreover, Theorem \ref{thm:low_degree} further illuminates the matter, since it identifies the optimal degree-$D$ test.

The definition of the \sqexp reveals that the desired pre-processing is precisely $\mathcal{T}(y) = \zeta_{\k}(y)$. We can then combine this pre-processing with an order-$\k$ optimal CSQ algorithm, based on the partial-trace algorithm \citep{hopkins2016fast, feldman2023sharp, damian2023smoothing}, described in Algorithm \ref{alg:partial_trace}. As shown in \Cref{app:proofs_upper}, this algorithm achieves the promised optimal sample complexity:

\begin{restatable}[\Cref{alg:partial_trace} Statistical Guarantee]{theorem}{tracethm}
\label{thm:optimal_sq_alg}
    Let $\{(x_i,y_i)\}_{i=1}^n$ be i.i.d. samples from $\PP_{w^\star, \P}$. For any $\epsilon > 0$ and $k \ge 1$, there exists a constant $C_k$ depending only on $k$ and a denoiser $\mathcal{T}: \R \to \R$ such that if $\tilde \lambda_k = \lambda_k/\log^k(3/\lambda_k)$ and $n \ge C_k [d^{k/2}/\tilde \lambda_k^2 + d/(\tilde \lambda_k^2 \epsilon^2)],$ \Cref{alg:partial_trace} applied to samples $\{(x_i,\mathcal{T}(y_i))\}_{i=1}^n$ returns a vector $\hat w$ with $(\hat w \cdot w^\star)^2 \ge 1-\epsilon^2$ with probability greater than $1-2e^{-d^c}$ for a constant $c = c(k)$ depending only on $k$.
\end{restatable}
The partial trace algorithm has been analysed in the context of Tensor PCA \citep{hopkins2016fast, feldman2023sharp}, and provides efficient recovery of the planted direction up to the computational threshold \citep{dudeja2021statistical}. At the technical level, the primary challenge we face is that the errors in \Cref{alg:partial_trace} are heavy tailed and non-Gaussian. Explicitly, while it is possible to reduce learning a single index model to a form of tensor PCA, the resulting noise matrix has highly correlated entries with heavy tails. For example, while the individual entries of the resulting noise tensor are of order $n^{-1/2}$, as for tensor PCA, the operator norm of this tensor is of order $d^{k/2}$ rather than $d^{1/2}$ for tensor PCA. 

\paragraph{Characterizing the precise weak recovery threshold}

\begin{figure}
    \centering
    \includegraphics[width=0.32\textwidth]{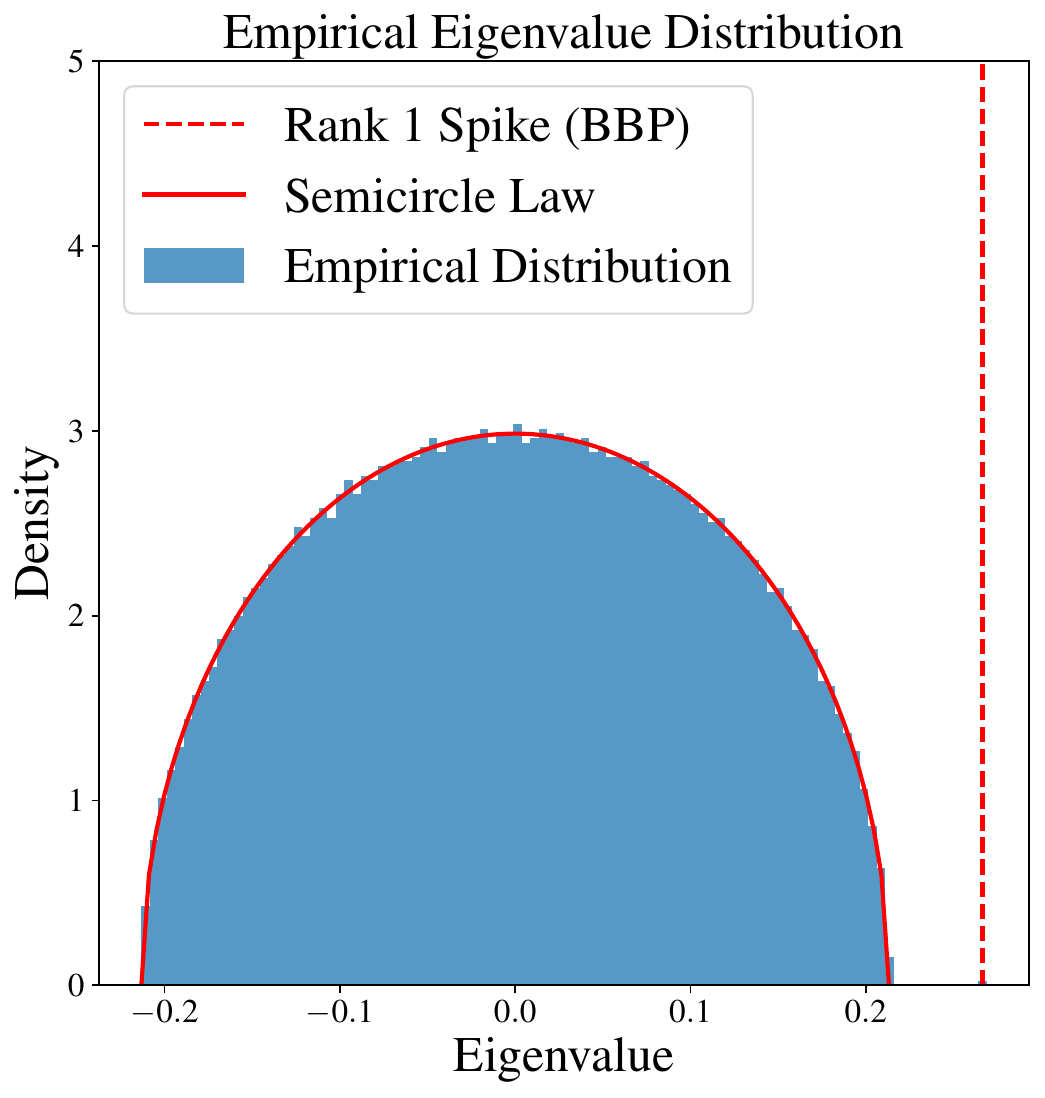}
    \includegraphics[width=0.32\textwidth]{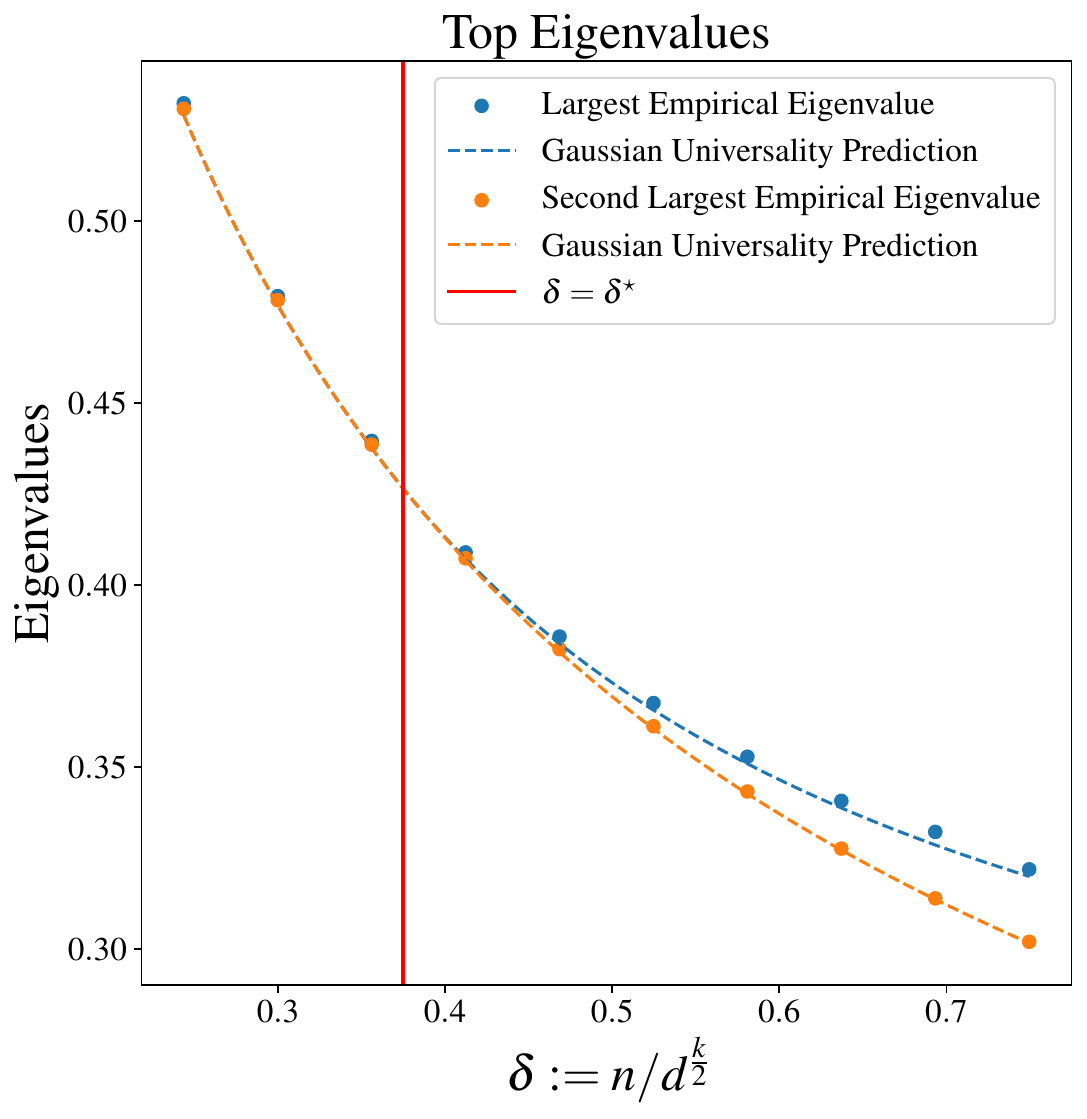}
    \includegraphics[width=0.32\textwidth]{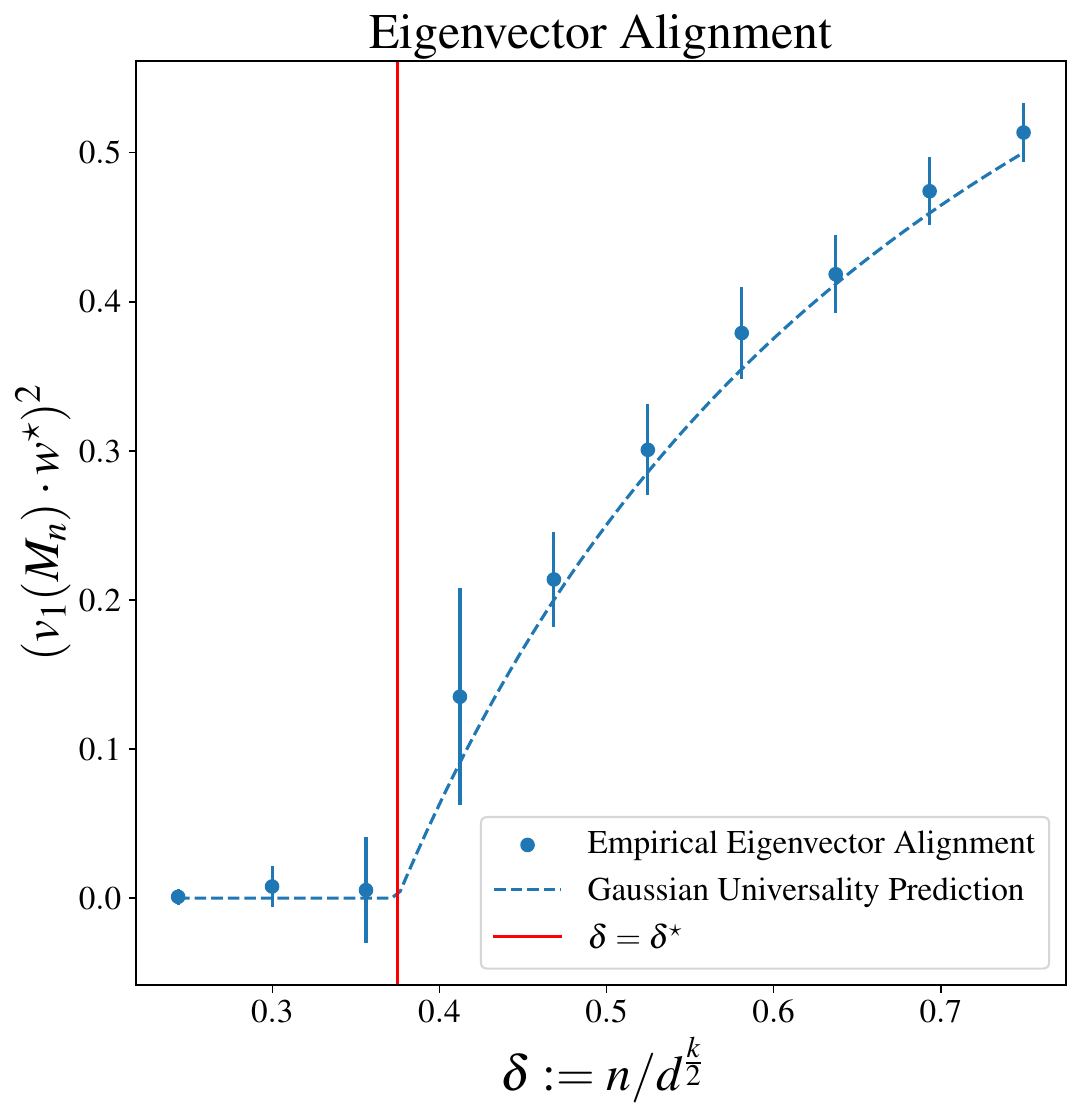}
    \caption{We empirically verify our observation in \Cref{lem:partial_trace_even} that $M_n$ satisfies Gaussian universality for $\k > 2$. We take $\k = 4$ and compute $M_n$ for $d = 4096$ and varying $\delta = \frac{n}{d^2}$. The target function is $Y = \zeta_4(Z^2 e^{-Z^2})$ (see \Cref{fig:examples_zeta}). Here $\delta^\star = \frac{\E[Y^2]}{\beta_k^2 k(k-1)!!}$ is the predicted BBP threshold \citep{baik2005phase}. The dots represent the medians over 10 random seeds of each quantity (eigenvalues and eigenvector correlation) and the error bars represent the standard deviation over these 10 trials.}
    \label{fig:universality}
\end{figure} 

A remarkable consequence of our analysis is that, when $\k>2$ is even, after applying the partial trace operation to reduce this tensor to a matrix, this matrix obeys Gaussian universality \cite{brailovskaya2023universality}, and in particular, the spectrum of the partial trace matrix $M_n$ converges to a semicircle law. 
    Explicitly, this means that the spectrum of $M_n$ is close to the spectrum of the Gaussian matrix $G$ with $\E[G] = \E[M_n]$ and $\E[G \otimes G] = \E[M_n \otimes M_n]$. Therefore to understand the spectral properties of $M_n$, it suffices to compute the the covariance of $M_n$, which we do in \Cref{lem:compute_cov_M1}. From \Cref{lem:compute_cov_M1}, we see that when $\k \ge 4$, the Gaussian equivalent matrix $G$ can be approximated by another Gaussian matrix $G^\star$ (in the space probability space) which satisfies:
    \begin{align}
        G^\star \stackrel{d}{=} \beta_\k w^\star {w^\star}^T + \sqrt{\frac{\E[\mathcal{T}(Y)^2] d^{\k/2}}{n \k(\k-1)!!}} \cdot W \qc \norm{G - G^\star}_2 \lesssim \sqrt{\frac{d^{\frac{\k-1}{2}}}{n}}~,
    \end{align}
    where $\beta_\k = \E[ \mathcal{T}(Y) h_\k(Z)]$ is the signal strength after applying the appropriate label transformation, $W$ is a standard Wigner matrix, and the bound on $\norm{G - G^\star}_2$ holds with high probability. We can therefore leverage existing results for spiked Wigner matrices to get exact constants for weak recovery when considering this partial trace estimator. 
    With high probability, the spectrum of $G^\star - \E G^\star$ is contained in the interval $[-2R(1+o_d(1)),2R(1+o_d(1))]$ where
    \begin{align}
        R^2 := \frac{\E[\mathcal{T}(Y)^2]}{\k(\k-1)!!} \cdot \frac{d^{\k/2}}{n}.
    \end{align}
    In addition, the spectrum of $G$ exhibits a BBP transition \citep{baik2005phase}, i.e. there is an outlier eigenvalue at $(\beta_\k + R^2/\beta_\k)(1+o_d(1))$. 
    It is also known that for $\beta_\k \ge R$, $(v_1(G^\star) \cdot w^\star)^2 \to 1 - (R/\beta_\k)^2$ where $v_1(G^\star)$ is the eigenvector of $G^\star$ corresponding to the largest eigenvalue (in absolute value). However, as we only prove universality of the spectrum of $M_n$ (and not of its eigenvectors), this does not directly imply that $v_1(M_n) \cdot w^\star$ has the same behavior. Nevertheless, we empirically verify this property for $M_n$ in \Cref{fig:universality}. The combination of Gaussian universality with the BBP transition implies there is a phase transition for weak recovery of $w^\star$ at the cutoff:
    \begin{align}
        \delta := \frac{n}{d^{\k/2}} = \frac{\E[Y^2]}{\beta_\k^2 \k(\k-1)!!} =: \delta^\star,
    \end{align}
    However, unlike for phase retrieval, in which the spectral estimator gives a tight threshold for estimation \cite{mondelli2018fundamental}, this constant can be improved by increasing the order of the tensor before applying partial trace (see \Cref{sec:discussion}). This is consistent with the known statistical-computational tradeoffs for tensor PCA (\cite{bhattiprolu2017sumofsquares,wein2019kikuchi}). We note also that the Gaussian universality does \emph{not} hold for $\k = 2$ in which $R$ and $\sigma^2$ are both order $d/n$, corresponding to the setting studied in \cite{mondelli2018fundamental}. 

\begin{remark}[Memory and Runtime of \Cref{alg:partial_trace}]
It is possible to implement \Cref{alg:partial_trace} with $O(d)$ memory (not counting the memory to store the dataset), and runtime $\tilde O(d^{\frac{k}{2}+1})$. This requires using \Cref{lem:efficient_tensor_power_iter} and \Cref{lem:efficient_partial_trace} to avoid explicitly computing the Hermite tensor $\bs{h}_k(x)$.
\end{remark}

\paragraph{Extension to unknown $\P$}

The only instance where we used the knowledge of $\P$ in the previous algorithm is in Lemma \ref{lem:bounded_T_sq_alg}, where a suitable thresholding $\mathcal{T}$ is applied to the labels. This guarantees that 
$\eta = \E [ \mathcal{T}(Y) h_k(Z)] \neq 0$, leading to the recovery of $w^\star$ in the high-dimensional regime. However, it is sufficient to consider a label transformation $\tilde{\mathcal{T}}$ that has \emph{non-negligible correlation} with $\zeta_k$. For instance, this can be guaranteed in a setting where $\P$ is only known to belong to a certain non-parametric class of distributions, as we now illustrate. 
Given $\P \in \mathcal{G}$, let $\{ \phi_k \}_k$ denote the orthogonal polynomial basis of $L^2(\R, \P_y)$. 
We now decompose $\zeta_{k^\star}$ in this basis:
$\zeta_{k^\star} = \sum_{l} \upsilon_l \phi_l$. 
\begin{assumption}[Source Condition]
\label{ass:sourcecond}
    For $m \geq 1$, define $\varepsilon_m := \lambda_{\k}^{-2}{\sum_{l \geq m} \upsilon_l^2} \leq 1$, then there exists $N$ such that $\varepsilon_m < 1$ for $m \geq N$. 
\end{assumption}
This assumption is mild as it only prevents extreme cases for which the first harmonics $(\upsilon_l)_{l \leq N}$ have infinitesimally small mass. 
Note that we choose a polynomial basis only for convenience\footnote{We note that from observed data $\{ (x_i, y_i) \}_{i\leq n}$, one could in particular estimate the marginal $\P_y$ using a (scalar) non-parametric kernel density estimator, and therefore estimate the first terms of the orthogonal polynomial basis. For simplicity, and w.l.o.g., we will assume that such basis elements are available.}.
We can now build $\tilde{\mathcal{T}}$ as a truncated random polynomial spanned by the first $M$ terms; \Cref{lem:bounded_T_agnostic} shows that such a random choice already provides a label transformation $\tilde{\mathcal{T}}$ with non-negligible correlation $\tilde{\eta} = \E_{\P} [\tilde{\mathcal{T}}(Y) h_{\k}(Z)]$, with high probability over the draw of the random polynomial.

Plugging this lemma directly in the proof of Theorem \ref{thm:optimal_sq_alg} and via a union-bound, we obtain an analogous guarantee for the recovery of $w^\star$, where the price of not knowing $\P$ only manifests itself in the worsening of the constant, from $\eta$ to $\tilde{\eta}$. 
We can now go one step further, and ignore $\k$ altogether. This may be achieved using the technique from \citep[Algorithm 2]{dudeja2018learning}, that tries all possible $k \in \{1, K \}$ and outputs the direction yielding the best goodness-of-fit on a held-out dataset of size $L$. 
The overall procedure, described in Algorithm \ref{alg:partial_trace_bis} in \Cref{app:unknownP}, enjoys the following guarantees:
\begin{restatable}[Partial Trace for unknown $\P$ and $\k$]{corollary}{unknownP}
\label{coro:partialtrace}
    Let $\{(x_i, y_i)\}_{i = 1}^{n}$ be i.i.d. samples from $ \PP_{w^\star, \P}$ with $\k(\P) = \k$. Then, under Assumption \ref{ass:sourcecond}, if $n \geq \Omega( \lambda_{\k}^2 d^{\k/2} + d/\epsilon^2)$, $L \gtrsim \delta^{-4} \mathrm{log} ( 1/\tilde{\delta})$, the procedure described in Algorithm \ref{alg:partial_trace_bis} with $K=\k$ returns $\hat{w}$ satisfying 
    $(\hat w \cdot w^\star)^2 \geq 1-\epsilon^2$ with probability greater than $1 - e^{-d^\kappa} - \delta - 2\tilde{\delta}\k$.
\end{restatable}

\section{Existence of Smooth Distributions for any \sqexp~\texorpdfstring{$\k$}{l}}
\label{sec:smoothsec}

Now that we have identified the precise sample and query complexity of the single index model  
in terms of the exponent $\k(\P)$, we turn to the question of characterizing the class $\{ \P; \k(\P) = k\}$ for any $k$. 
We focus our attention on the additive gaussian noisy setting, namely $(Z, Y) \sim \P$ satisfy $Y = \sigma(Z) + \tau \xi$, where $\sigma:\R \to \R$ is a link function, and $\xi\sim \gamma_1$ is independent of $Z$. When $\tau=0$ we recover the deterministic case. 

In this context, the \infexp provides a transparent description in terms of the Hermite decomposition of $\sigma$; in particular for each $k$ one can easily construct analytic functions such that $\kk(\P) = \kk(\sigma) = k$ (e.g. the degree-$k$ Hermite polynomial). Such structure is absent in the \sqexp setting: observe that in virtue of Lemma \ref{lem:composition_lemma}, the exponent only depends on the set of level sets $\{\{u \in \R; \sigma(u)=t\}\}_t$, which does not easily lend itself to harmonic analysis. 

A simple example of link function with $\k > 2$ is $\sigma(x) = x^2 e^{-x^2}$. For this $\sigma$, $\k = 4$ which implies the single index model determined by $\sigma$ is unlearnable in polynomial time without $n \gtrsim d^2$ samples.\footnote{Examples of $\sigma$ with $\k > 2$ were discovered in prior works as well, e.g. \cite[Remark 3]{mondelli2018fundamental}. See \Cref{fig:examples_zeta} for a figure of this construction along with the corresponding $\zeta_4$.} However, this construction does not easily generalize to higher $\k$. Nonetheless, we are able to establish the existence of smooth link functions with prescribed \sqexp:
\begin{restatable}[Smooth Single-Index models with prescribed $\k$]{theorem}{smoothlink}
\label{thm:smooth_link}
    For each $k$, there exists $\sigma \in C^\infty_b(\R)$ such that the deterministic single index model $\P = (\mathrm{Id} \otimes \sigma)_\# \gamma_1 $ satisfies $\k(\P)= k$.    
\end{restatable}
The main idea of the proof, presented in \Cref{app:smoothlink_main}, is to build the link function via the coarea formula, by evolving a one-parameter family of level sets $\{ S_t:= \sigma^{-1}(t) \}_t$. Each level set $S_t = \{ z_{1,t}, \ldots, z_{n_t, t} \}$ needs to be such that $\E[h_k(z) | z \in \sigma^{-1}(t) ] =0$, which becomes a polynomial equation in the points $z_{j,t}$. We first determine a suitable form for $S_0$, based on Hermite-Gauss quadrature, and then obtain $S_t$, $t\in [0,T)$ as the solution of an ODE that enforces that this condition is preserved through `time' using a Vandermonde-type kernel. See \Cref{fig:explicit_constructions} for examples of link functions $\sigma$ with generative exponents $\k = 3,...,8$.

\begin{figure}[H]
    \subfigure[$\k=3$]{\includegraphics[width=0.32\textwidth]{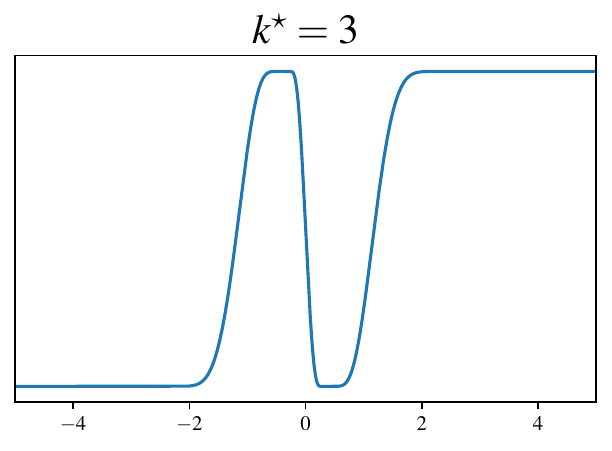}}
    \subfigure[$\k=4$]{\includegraphics[width=0.32\textwidth]{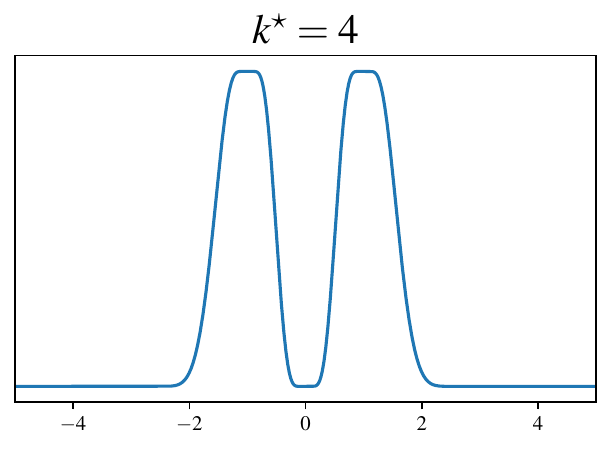}}
    \subfigure[$\k=5$]{\includegraphics[width=0.32\textwidth]{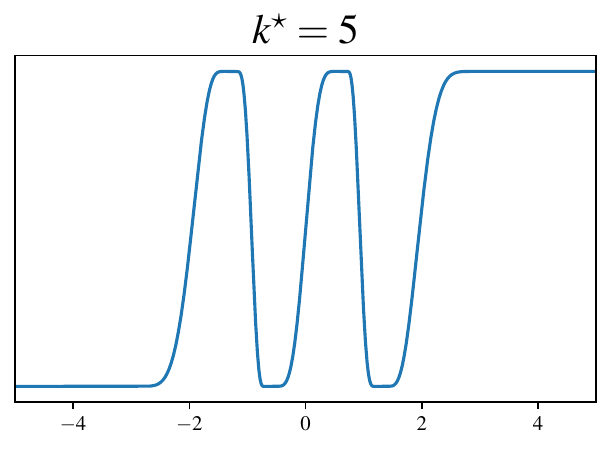}} \\
    \subfigure[$\k=6$]{\includegraphics[width=0.32\textwidth]{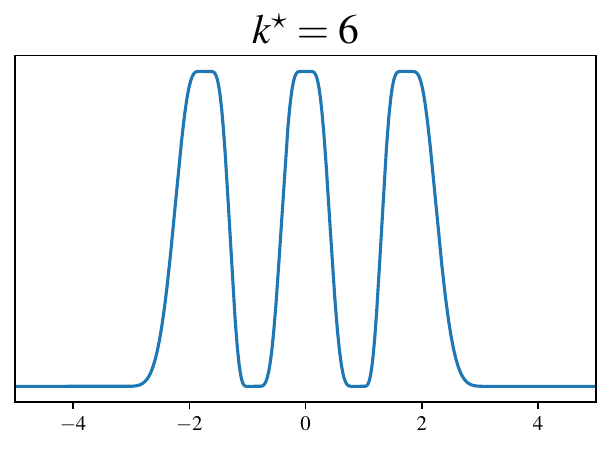}}
    \subfigure[$\k=7$]{\includegraphics[width=0.32\textwidth]{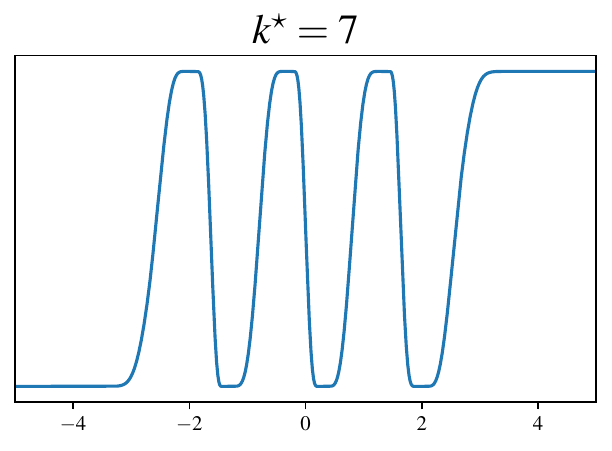}}
    \subfigure[$\k=8$]{\includegraphics[width=0.32\textwidth]{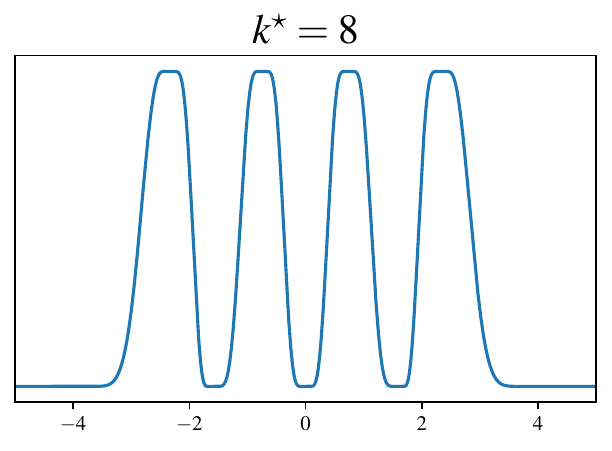}}
    \caption{Explicit constructions of $\sigma$ with different prescribed generative exponents. These were generated by numerically integrating the ODE in \cref{eq:keep_k_ode}.}
    \label{fig:explicit_constructions}
\end{figure}

Finally, we establish that adding Gaussian noise to the labels preserves the \sqexp:
\begin{restatable}[Additive Gaussian noise preserves \sqexp]{theorem}{additivelink}
\label{thm:additive_noise_link}
    For $\tau \geq 0$ and $\sigma: \R \to \R$, we denote $\Phi_{\tau, \sigma}(u, v) = ( u, \sigma(u) + \tau v) \in \R^2$. Then the additive noisy model $\tilde{\P} = (\Phi_{\tau, \sigma})_\# \gamma_2$ satisfies $\k(\tilde{\P}) = \k(\P)$, where $\P = (\mathrm{Id} \otimes \sigma)_\# \gamma_1 $. %
\end{restatable}
The proof of \Cref{thm:additive_noise_link} reveals that non-Gaussian noise distributions $\mu$ also yield the same conclusion, provided the characteristic function $\varphi(\xi)=\E_{Z\sim \mu}[e^{-i\xi Z}]$ satisfies $\varphi(\xi) \neq 0$ for all $\xi$. Under these conditions, like the information exponent, the \sqexp is `oblivious' to additive noise (albeit with potentially smaller signal strength $\tilde{\lambda}_\k$). %

\section{Information-Theoretic Sample-Complexity}
\label{sec:it}

We conclude the analysis of the single-index model by obtaining 
an upper bound for the sample complexity, irrespective of the estimation procedure. This question has been addressed in several planted problems of similar structure \citep{mondelli2018fundamental, dudeja2021statistical, montanari2014statistical}, providing the groundwork needed to establish the following result:
\begin{restatable}[Information-Theoretic Upper Bound]{theorem}{ITupperbound}
\label{thm:ITupper}
       For all $k \ge 1$, there exists a procedure which returns $\hat w$ with $(w \cdot w^\star)^2 \ge 1-\epsilon^2$ with probability at least $1-2e^{-d}$ if $n= \tilde \Theta\left(\frac{d}{\lambda_k^2 \epsilon^2}\right)$.
\end{restatable}
In essence, the (inefficient) estimator optimizes a certain correlation over the sphere, and uses a `corrector' step from \citep[Theorem 5]{dudeja2021statistical} to yield a tight dependence on $\epsilon$.
Combined with the SQ and LDP-lower bounds, these results thus establish a sharp computation-to-statistical gap (under both the SQ and LDP frameworks) for the single-index problem as soon as $\k(\P) > 2$. 

%% file: conclusions.tex
\section{Conclusions}

In this work, we have explored the question of efficiently estimating 
the planted structure from data drawn from a single-index model. 
We identified the \sqexp $\k(\P)$ as the fundamental quantity driving the sample complexity required for recovering the hidden direction, and proved tight matching upper and lower bounds, using the partial trace algorithm and the SQ and LDP frameworks respectively. Taken together, our results give a unified perspective on a variety of planted high-dimensional problems, and provide evidence of a tight computational-to-statistical gap as soon as $\k(\P) > 2$. 

That said, there are nuances that future work should aim to address: On one hand, our partial trace algorithm includes power iterations, which are not \emph{bona-fide} SQ, making the fine-grained comparison between our upper and SQ-lower bounds somewhat murky \citep{dudeja2021statistical}. 
On the other hand, our Low-Degree lower bound concerns the detection problem (that is, testing $\mathbb{P}_w$ vs $\mathbb{P}_0$ for some $w$), while the natural setting for us is the recovery problem. Although this mismatch does not impact the tightness of the rate $d^{\k/2}$, in other high-dimensional inference problems it reveals more fine-grained information, e.g. \cite{perry2018optimality}, such as sharp recovery thresholds. 

One natural extension of our work is to the \emph{multi-index} setting, in which the labels depend on a projection of the input onto a subspace of dimension $r \ge 1$. A multi-index model can be described by an orthogonal matrix $W^\star \in \R^{r \times d}$ and a joint distribution $\P \in \mathcal{P}(\R^r \times \R)$ over $(Z,Y)$ where $z = W^\star x$. The goal in the multi-index setting is to recover the subspace defined by $W^\star$. This already gives rise to rich structure not present in the single-index setting. For example, the natural generalization of the information exponent \citep{arous2021online} is the \emph{leap complexity} \citep{abbe2023sgd,Dandi2024TheBO,bietti2023learning}. However, it remains unclear what the natural generalization of the \sqexp is, even for $\k = 2$.

%% file: notation.tex
\section{Additional Notation}

\subsection{Tensor Notation}

Throughout this section let $T \in (\R^d)^{\otimes k}$ be a $k$-tensor.

\begin{definition}[Tensor Action]
	For a $j$ tensor $A \in (\R^d)^{\otimes j}$ with $j \le k$, we define the action $T[A]$ of $T$ on $A$ by
	\begin{align}
	(T[A])_{i_1,\ldots,i_{k-j}} := \sum_{i_{k-j+1},\ldots,i_k=1}^d T_{i_1,\ldots,i_k} A^{i_{k-j+1},\ldots,i_k} \in (\R^d)^{\otimes (k-j)}.
	\end{align}
\end{definition}

We will also use $\langle T,A \rangle$ to denote $T[A] = A[T]$ when $A,T$ are both $k$ tensors. Note that this corresponds to the standard dot product after flattening $A,T$.

\begin{definition}[Permutation/Transposition]
	Given a $k$-tensor $T$ and a permutation $\pi \in S_k$, we use $\pi(T)$ to denote the result of permuting the axes of $T$ by the permutation $\pi$, i.e.
	\begin{align}
		\pi(T)_{i_1,\ldots,i_k} := T_{i_{\pi(1)},\ldots,i_{\pi(k)}}.
	\end{align}
\end{definition}

\begin{definition}[Symmetrization]
	We define $\sym_k \in (\R^d)^{\otimes 2k}$ by
	\begin{align}
		(\sym_k)_{i_1,\ldots,i_k,j_1,\ldots,j_k} = \frac{1}{k!} \sum_{\pi \in S_k} \delta_{i_{\pi(1)},j_1} \cdots \delta_{i_{\pi(k)},j_k}
	\end{align}
	where $S_k$ is the symmetric group on $1,\ldots,k$. Note that $\sym_k$ acts on $k$ tensors $T$ by
	\begin{align}
		(\sym_k[T])_{i_1,\ldots,i_k} = \frac{1}{k!}\sum_{\pi \in S_k} \pi(T).
	\end{align}
	i.e. $\sym_k[T]$ is the symmetrized version of $T$.
\end{definition}
We will also overload notation and use $\sym$ to denote the symmetrization operator, i.e. if $T$ is a $k$-tensor, $\sym(T) := \sym_k[T]$.

\begin{lemma}\label{lem:symmetrized_frobenius_norm}
    For any tensor $T$,
    \begin{align}
        \norm{\sym(T)}_F \le \norm{T}_F.
    \end{align}
\end{lemma}
\begin{proof}
    \begin{align}
        \norm{\sym(T)}_F = \norm{\frac{1}{k!} \sum_{\pi \in S_k} \pi(T)}_F \le \frac{1}{k!} \sum_{\pi \in S_k} \norm{\pi(T)}_F = \norm{T}_F
    \end{align}
    because permuting the indices of $T$ does not change the Frobenius norm.
\end{proof}

%% file: sq_app.tex
\section{SQ Framework for Search Problems}
\label{app:sq_app}

The relevant framework is developed in \cite{feldman2017planted-clique}. We provide a quick recap of the relevant definitions.
\subsection{Main Ingredients}
\begin{definition}[Search Problem over Distributions]~
    Let $X$ be a domain, $\mathcal{D}$ be a set of distributions over $X$, $\mathcal{F}$ be a set of solutions, and $\mathcal{Z}: \mathcal{D} \to 2^\mathcal{F}$ be a map to the set of valid solutions. The distributional search problem is to find a valid solution $f \in \mathcal{Z}(D)$ given oracle access to samples from an unknown $D \in \mathcal{D}$. We will also use $\mathcal{Z}_f$ to denote the set of distributions $\mathcal{D}$ for which $f$ is a valid solution.
\end{definition}

\begin{definition}[STAT oracle]
    Let $D \in \mathcal{D}$ be the unknown distribution. Given a tolerance $\tau$ and a query $h: X \to [-1,1]$, the $\mathrm{STAT}(\tau)$ oracle returns a value within $\tau$ of $\E_{x \sim D}[h(x)]$.
\end{definition}

\begin{definition}[Relative Pairwise Correlation]
    Given two distributions $D_1,D_2$ and a reference distribution $D$,
    \begin{align}
        \chi_D(D_1,D_2) := \int \frac{D_1(x)D_2(x)}{D(x)} dx - 1.
    \end{align}
\end{definition}

\begin{definition}[($\gamma,\beta$)-correlation]
    We say that a set of $m$ distributions $\mathcal{D} = \{D_1,\ldots,D_m\}$ is $(\gamma,\beta)$ correlated relative to a distribution $D_0$ over $X$ if $\abs{\chi_D(D_i,D_j)} \le \gamma$ for $i \ne j$ and $\abs{\chi_{D_0}(D_i,D_i)} \le \beta$ for all $i \in [m]$.
\end{definition}

\begin{definition}[SQ Dimension]\label{def:sq_dim}
    Given a search problem $\mathcal{Z}$ and parameters $\gamma,\beta$, we define the statistical query dimension $\mathcal{SD}(\mathcal{Z},\gamma,\beta)$ to be the largest integer $m$ such that there exists a distribution $D_0$ over $X$ and a finite set of distributions $\mathcal{D}_D \subset \mathcal{D}$ with $\abs{\mathcal{D}_D} \ge m$ such that for any $f \in \mathcal{F}$, $\mathcal{D}_f := \mathcal{D}_D \setminus \mathcal{Z}_f$ is $(\gamma,\beta)$-correlated relative to $D_0$.
\end{definition}

The following lemma is from \cite[Corollary 3.12]{feldman2017planted-clique}:
\begin{lemma}[General SQ Lower Bound]\label{lem:general_sq}
    For any $\gamma' > 0$, any SQ algorithm requires at least $\mathcal{SD}(\mathcal{Z},\gamma,\beta) \cdot \frac{\gamma'}{\beta-\gamma}$ queries to $\mathrm{STAT}(\sqrt{\gamma + \gamma'})$ or $\mathrm{VSTAT}\qty(\frac{1}{3(\gamma + \gamma')})$ to solve $\mathcal{Z}$.
\end{lemma}
We will use the following corollary which is equivalent to \Cref{lem:general_sq}: 
\begin{corollary}\label{lem:easy_sq}
    For any $\gamma,\beta,\tau \ge 0$, any algorithm requires at least $\mathcal{SD}(\mathcal{Z},\gamma,\beta) \cdot \frac{\frac{3}{n} - \gamma}{\beta-\gamma}$ queries to $\mathrm{VSTAT}(n)$ to solve $\mathcal{Z}$.
\end{corollary}

\subsection{Instantation for the Single-Index Problem}
\label{sec:SQinstance_singleindex}
We now instantiate this framework for the single-index model from Definition~\ref{def:single-index_model}: 
\begin{itemize}
    \item 
    \textbf{Domain:} $X = \R^d \times \R$ (represents the $(X,Y)$ pair).
    \item 
    \textbf{Distributions:} $\mathcal{D} = \{\PP_w : w \in S^{d-1}\}$. %
    \item 
    \textbf{Solution Set:} $\mathcal{F} = S^{d-1}$.
    \item 
    \textbf{Valid Solutions:} $\mathcal{Z}(\PP_{w^\star}) = \{w \in \mathcal{F} ~:~ \abs{w \cdot w^\star} \ge \Tilde \Theta(d^{-1/2})\}$.
    \item 
    \textbf{Inverse:} $\mathcal{Z}_w = \{\PP_{w^\star} ~:~ w^\star \in S^{d-1} \text{ and } \abs{w \cdot w^\star} \ge \Tilde \Theta(d^{-1/2})\}$,
    \item 
    \textbf{Reference Distribution:} $D = \gamma_d \otimes \P_y$. %
\end{itemize}

%% file: hermite_app.tex
\section{Hermite Polynomials and Hermite Tensors}\label{sec:hermite}
We provide a brief review of the properties of Hermite polynomials and Hermite tensors.

\begin{definition}
    Let $u \in \R^d$. We define the $k$th normalized Hermite tensor $\bs{h}_k \in (\R^d)^{\otimes k}$ by
    \begin{align}
        \bs{h}_k(u) := \frac{(-1)^k}{\sqrt{k!}} \frac{\nabla^k \gamma_d(u)}{\gamma_d(u)}
    \end{align}
    where $\gamma_d(u) := \frac{e^{-\norm{u}^2/2}}{(2\pi)^{d/2}}$ is the PDF of a standard Gaussian in $d$ dimensions.
\end{definition}
Note that when $d = 1$, this definition reduces to the standard univariate Hermite polynomials $\{h_k\}$, which are orthonormal with respect to $\gamma_1$:
\begin{align}
    \E_{u \sim \gamma_1}[h_j(u)h_k(u)] = \delta_{jk}~.
\end{align}
Furthermore, if $u,v$ are correlated Gaussians with correlation $\alpha$, this inner product scales with $\alpha^k$. Explicitly,
\begin{align}
    \E_{u,v \sim \gamma_2^{(\alpha)}}[h_j(u)h_k(v)] = \delta_{jk} \alpha^k \qq{where} \gamma_2^{(\alpha)} = N\left(0,\begin{bmatrix} 1 & \alpha \\ \alpha & 1\end{bmatrix}\right).
\end{align}
The orthogonality property also has a tensor analogue:
\begin{align}
    \E_{u \sim \gamma_d} [\bs{h}_j(u) \otimes \bs{h}_k(u)] = \delta_{jk} \mathrm{Sym}_k.
\end{align}
Equivalently, for any $j$ tensor $A$ and $k$ tensor $B$:
    \begin{align}
        \E_{u \sim \gamma_d}\left[\langle \bs{h}_j(x),A\rangle\langle \bs{h}_k(x),B \rangle\right] = \delta_{jk} \langle \mathrm{Sym}(A),\mathrm{Sym}(B) \rangle.
\end{align}
The Hermite tensors in $\R^d$ are related to the univariate Hermite polynomials by the identity:
\begin{align}
    h_k(u \cdot v) = \langle \bs{h}_k(u),v^{\otimes k} \rangle \quad\text{for all}~u \in \R^d, v \in S^{d-1}.
\end{align}

%% file: proofs_sqexp.tex
\section{Proofs of Section \ref{sec:sqexp}}
\label{app:proof_of_section_sqexp}

\variancedecomp*
\begin{proof}
    We have $\mathrm{Var}_\P[\sigma(Z)] = \E[\sigma(Z)^2] - \E[\sigma(Z)]^2 $ and by decomposition of $\sigma$ into $\{h_k\}_{k\geq0}$, the orthogonal basis of $L^2(\R, \gamma_1)$, we have
    \begin{align}
        \E[\sigma(Z)^2] = \sum_{l \geq 0} \E[ \sigma(Z) h_l(Z) ]^2 = \sum_{l \geq 0} \E\left[ \E[Y | Z] h_l(Z) \right]^2 = \sum_{l \geq 0} \E[ Y  h_l(Z) ]^2~,
    \end{align}
    where the last equality stems from the property of conditional expectation. Finally, substracting $\E[\sigma(Z)]^2 = \beta_0^2$ ends the proof of the equality.
\end{proof}

We begin by computing the Hermite expansion of $\frac{d\P}{d\P_0}$ where $\P_0 = \P_z \otimes \P_y$ is the null distribution:
\begin{lemma}\label{lem:P_hermite} We have the following expansion in $L^2(\P_0)$:
    \begin{align}
        \frac{d\P}{d\P_0}(z,y) = \sum_{k \ge 0} \zeta_k(y) h_k(z).
    \end{align}
\end{lemma}
\begin{proof}
    We can directly compute the $k$th Hermite coefficient of the likelihood ratio as a function of $Y$:
    \begin{align}
        \E_{\P_0}\qty[\frac{d\P}{d\P_0}(Z,Y) \bs{h}_k(Z)|Y] = \E_{\P}[\bs{h}_k(Z)|Y] = \zeta_k.
    \end{align}
    Therefore the Hermite expansion of $\frac{d\P}{d\P_0}$ is,
    \begin{align}
        \frac{d\P}{d\P_0}(z,y) \stackrel{L^2(\P_0)}{=} \sum_{k \ge 0} \zeta_k(y) h_k(z).
    \end{align}
\end{proof}

We can now restate and prove Lemma~\ref{lem:chiinfo}. 
\chiinfo*
\begin{proof}[Proof of \Cref{lem:chiinfo}]
Recall that the mutual information is given by
\begin{align}
    I_{\chi^2}[\P] = \E_{\P_0}\qty[\qty(\frac{d\P}{d\P_0})^2] - 1.
\end{align}
Therefore by \Cref{lem:P_hermite}, this is equal to
\begin{align}
    I_{\chi^2}[\P] = \sum_{k \ge 0} \E_{\P_y}[\zeta_k(Y)^2] - 1 = \sum_{k \ge 0} \lambda_k^2 - 1 = \sum_{k \ge 1} \lambda_k^2.
\end{align}
\end{proof}

\sqexpvariational*
\begin{proof}
    Let $k^\star = \k(\P)$. For any $k < k^\star$ and $\mathcal{T} \in L^2(\R, \P_y)$, By properties of the conditional expectation:
\begin{align}
    \E\left[ \mathcal{T}(Y)h_k(Z) \right] =  \E\left[  \E\left[ \mathcal{T}(Y)h_k(Z) \vert Y \right] \right] = \E\left[  \mathcal{T}(Y) \E\left[ h_k(Z) \vert Y \right] \right] = \E\left[  \mathcal{T}(Y)\zeta_k(Y) \right] = 0.
\end{align}

    Therefore for all $\mathcal{T} \in L^2(\R, \P_y)$, we have $\k(\P) \leq \kk((\mathrm{Id} \otimes \mathcal{T})_\#  \P)~$ and hence taking the infimum over such $\mathcal{T}$, it yields that
    \begin{align}
         k^\star \le \inf_{\mathcal{T} \in L^2(\P_y)} \kk((\mathrm{Id} \otimes \mathcal{T})_\#  \P)~.
    \end{align}
    Next, let define for all $y \in \R$, $\mathcal{T^*}(y) := \zeta_{\k}(y) = \mathbb{E}[h_{\k}(Z)|Y = y]$. Note that $\mathcal{T} \in L^2(\P_y)$ and by the same calculation as previously, we have
    \begin{align}
        \E\left[ \mathcal{T}^*(Y)h_\k(Z) \right] = \E \left[ \mathcal{T}^*(Y)\zeta_\k(Y)  \right] = \E \left[\zeta^2_\k(Y)  \right] = \| \zeta_\k \|^2_{\P_y} > 0~,
    \end{align}
    which concludes the proof of the theorem.

\end{proof}

\examplegen*
\begin{proof}
When $\sigma$ is a polynomial, we need to show 
that there exists $g \in \P_y$ such that 
\begin{align}
\label{eq:simplepoly}
    \E[ g( \sigma(z)) h_{l}(z) ] &\neq 0~
\end{align}
for $l=1$ or $l=2$.
Since both $\sigma$ and $h_{l}$ are monotonic after their largest root, picking $g(t) = \mathbf{1}_{t \in [R-\delta, R+\delta]}$ for sufficiently large $R$ and $\delta>0$ yields 
\begin{align}
    \E[ g( \sigma(z)) h_{l}(z) ] &= \E[ h_l(z) \mathbf{1}_{z \in A}]~,
\end{align}
where either $A=I$ is a single interval (if $\sigma$ has odd degree) or $A=I_{-} \cup I_{+}$ two intervals if $\sigma$ has even degree. We have $\E[ h_1(z) \mathbf{1}_{z \in A}] \neq 0$ whenever $A=I$ or $I_{-} \neq -I_{+}$, and $\E[ h_2(z) \mathbf{1}_{z \in A}] \neq 0$ otherwise. To conclude, observe that we can find $R$ and $\delta$ such that $I_{-} \neq -I_{+}$ iff $\sigma$ is not even.

\begin{lemma}\label{lem:z2expz2_k4}
    Let $\sigma(Z) := Z^2 \exp(-Z^2)$. Then the single index model defined by $\P := (\mathrm{Id} \otimes \sigma)_\# \gamma_1$ satisfies $\k(\P) = 4$.
\end{lemma}
\begin{proof}[Proof of Lemma \ref{lem:z2expz2_k4}]
    As $\sigma$ is even it suffices to prove that $\lambda_2 = 0$ and $\lambda_4 > 0$. By \Cref{lem:integral_vs_differential}, to prove that $\lambda_4 = 0$ it suffices to check that
    \begin{align}
        \sum_{z \in \sigma^{-1}(y)} \sign(\sigma'(z)) z \gamma(z)
    \end{align}
    is constant in $y$ $\P_y$-almost everywhere. Therefore it suffices to check this for $0<y<1$, as $\max_z \sigma(z) = 1$. On this interval, $\sigma^{-1}(y) = \{-z_2(y),-z_1(y),z_1(y),z_2(y)\}$ with $\sigma'(z_1(y)) > 0$ and $\sigma'(z_2(y)) < 0$. Therefore for $y \in (0,1)$,
    \begin{align}
        &\sum_{z \in \sigma^{-1}(y)} \sign(\sigma'(z)) z \gamma(z) \nonumber\\
        &= 2[z_1(y) \gamma(z_1(y)) - z_2(y) \gamma(z_2(y))] \nonumber\\
        &= 2[\sqrt{y} - \sqrt{y}] \nonumber\\
        &= 0.
    \end{align}
    Next, we need to verify that $\zeta_4 \ne 0$. In fact, we claim that $\beta_4 \ne 0$, i.e. $\kk(\P) = 4$. We have that:
    \begin{align}
        \E[\sigma(Z) h_4(Z)]
        &= \int_{-\infty}^\infty Z^2 e^{-Z^2} \cdot He_4(Z) \cdot \frac{e^{-Z^2/2}}{\sqrt{2\pi}} dZ \nonumber\\
        &= \int_{-\infty}^\infty (Z^6-6Z^4+3Z^2) \cdot \frac{e^{-3Z^2/2}}{\sqrt{2\pi}} dZ \nonumber\\
        &= -\frac{4 \sqrt{3}}{27} \nonumber \\
        &\ne 0
    \end{align}
    by routine Gaussian integration.
\end{proof}

\end{proof}

%% file: proofs_lowerbound.tex
\section{Proofs of Section \ref{sec:sqlower}}
\label{app:sec_proof_section_sqlower}

The main goal of this section is to prove Theorem~\ref{thm:sq_lower_bound}. We begin this section by proving some necessary intermediate results, and then conclude the section proving the aforementioned theorem. Finally we conclude providing results on reduction from the NGCA to the single index model and back. 

\subsection{Proof of the preliminary results}

We begin by generalizing \Cref{lem:P_hermite} to the full distribution $\PP_w$ over $(X,Y)$ where the null distribution is now $\PP_0 := \gamma_d \otimes \P_y$.

\begin{lemma}\label{lem:likelihood_hermite} For any $w \in S^{d-1}$ we have the following expansion in $L^2(\PP_0)$:
    \begin{align}
        \frac{d\PP_w}{d\PP_0}(X,Y) \stackrel{L^2(\PP_0)}{=} \sum_{k \ge 0} \zeta_k(Y) h_k(X \cdot w)
    \end{align}
\end{lemma}
\begin{proof}
    Note that by \Cref{def:single-index_model}, $d\PP_w(X,Y) = \gamma_{d-1}(X^\perp) \P(X \cdot w, Y)$ and $d\PP_0(X,Y) = \gamma_{d-1}(X^\perp) \P_0(X \cdot w, Y)$. Therefore by \Cref{lem:P_hermite},
    \begin{align}
        \frac{d\PP_w}{d\PP_0}(X,Y) = \frac{d\P}{d\P_0}(X \cdot w,Y) = \sum_{k \ge 0} \zeta_k(Y) h_k(X \cdot w).
    \end{align}
\end{proof}

We can now prove the expansion of $\chi^2_{0}(\PP_{w},\PP_{w'})$ presented in Lemma~\ref{lem:chi2key}.
\chikey*
\begin{proof}[Proof of Lemma \ref{lem:chi2key}]
    We have:
    \begin{align}
        \chi^2_{0}(\PP_{w},\PP_{w'}) := \E_{\PP_0}\qty[\frac{d\PP_w}{d\PP_0} \cdot \frac{d\PP_{w'}}{d\PP_0}] - 1.
    \end{align}
    Therefore by \Cref{lem:likelihood_hermite} and the orthogonality property of Hermite polynomials, this is equal to
    \begin{align*}
        \chi^2_{\PP_0}(\PP_{w},\PP_{w'}) = \sum_{k \ge 0} \E[\zeta_k(Y)^2] (w \cdot w')^k - 1 = \sum_{k \ge 1} \lambda_k^2 m^k.
    \end{align*}

\end{proof}

\subsection{Proof of Theorem~\ref{thm:sq_lower_bound}}
\label{app:proof_sq_lower}

We now move towards proving Theorem~\ref{thm:sq_lower_bound}. For this we introduce two intermediate results in Lemma~\ref{lem:chi_easyupper} and~\ref{lem:many_orthogonal_vecs}.

\begin{lemma}
\label{lem:chi_easyupper}
   Let $m = w \cdot w'$. We have
    \begin{align}
        \chi^2_0(\PP_{w},\PP_{w'}) \le \lambda^2_\k m^{\k} + \frac{m^{k^\star+1}}{1-m}.
    \end{align} 
\end{lemma}
\begin{proof} From the previous lemma, we have that
    \begin{align}
        \chi^2_0(\PP_{w},\PP_{w'})
         &= \sum_{k \ge k^\star} \lambda^2_k m^k \nonumber\\
         &=  \lambda^2_\k m_\k  + \sum_{k > k^\star} \lambda_k m^k \nonumber\\
         &\le \lambda^2_\k m_\k + \sum_{k > k^\star} m^k \nonumber\\
         &= \lambda^2_\k m_\k + \frac{m^{\k+1}}{1-m}.
    \end{align}
\end{proof}

The following lemma shows that there are a large number of nearly orthogonal vectors:
\begin{lemma}\label{lem:many_orthogonal_vecs}
    There exists an absolute constant $C$ such that for any $m \le d^{d^{1/4}}$, there exist $m$ vectors $w_1,\ldots,w_m$ with $\max_{i \ne j} |w_i \cdot w_j| \le \epsilon$ for $\epsilon = \sqrt{\frac{C\log_d(m)^2}{d}}.$
\end{lemma}
\begin{proof}
    From \cite{nelson2012deterministic}, we have that there exists a constant $c$ such that such points exist whenever
    \begin{align}
        d \le c \epsilon^{-2} \qty(\frac{\log m}{\log \log m + \log(1/\epsilon)})^2.
    \end{align}
    We will take
    \begin{align}
        \epsilon = \sqrt{\frac{C\log_d(m)^2}{d}}.
    \end{align}
    for a sufficiently large constant $C$. Then we have that
    \begin{align}
        \log m \le 1/\epsilon \iff \log^2 m \le \frac{d \log^2 d}{C^2 \log^2 m} \iff \log m \le \frac{d^{1/4} \log^{1/2} d}{C^{1/2}}
    \end{align}
    which is true by the assumption on $m$. In addition, note that $\log(1/\epsilon) \lesssim \log(d)$. Therefore,
    \begin{align}
        \epsilon^{-2} \qty(\frac{\log m}{\log \log m + \log(1/\epsilon)})^2 \gtrsim \frac{d}{C \log_d(m)^2} \cdot \qty(\frac{\log m}{\log d})^2 = \frac{d}{C}
    \end{align}
    and taking $C$ sufficiently large completes the proof.
\end{proof}

Now we are in place to prove the main result of the section, Theorem~\ref{thm:sq_lower_bound}.

\begin{proof}[Proof of Theorem \ref{thm:sq_lower_bound}]
    Let $m < d^{d^{1/4}}$ be a positive integer to be chosen later. From \Cref{lem:many_orthogonal_vecs}, there exist $m$ vectors $w_1,\ldots,w_m$ such that $\abs{w_i \cdot w_j} \le \epsilon$ for all $i \ne j$ where $\epsilon = \sqrt{\frac{C\log_d(m)^2}{d}}$ for an absolute constant $C$. Let $\mathcal{D} = \{\PP_{w_i} ~:~ i \in [m]\}$. For any $\PP_{w_i},\PP_{w_j}$ with $i \ne j$,
    \begin{align}
        \chi^2_0(D_{w_1},D_{w_2}) \le \lambda^2_\k (w_1 \cdot w_2)^{\k} + \frac{(w_1 \cdot w_2)^{\k+1}}{1-(w_2 \cdot w_2)} \le \lambda^2_\k \epsilon^{\k} + 2\epsilon^{\k+1}.
    \end{align}
    Therefore for $\lambda^2_\k > 2\epsilon$, we can bound this by $2 \lambda^2_\k \epsilon^{\k}$. Therefore $\mathcal{SD}(\mathcal{Z},2\lambda^2_\k \epsilon^\k,1) \ge m$. Then by \Cref{lem:easy_sq},
    \begin{align}
        q \ge m \cdot \frac{\frac{3}{n} - 2 \lambda^2_\k \epsilon^{k^\star}}{1-2 \lambda^2_\k\epsilon^{k^\star}} \ge \frac{1}{2} m \cdot \qty(\frac{3}{n} - 2 \lambda^2_\k \epsilon^{k^\star})
    \end{align}
    which implies
    \begin{align}
        \frac{3}{n} \le 2 \lambda^2_\k \epsilon^{k^\star} + \frac{2q}{m}.
    \end{align}
    Now setting $m = 2qn$ gives that
    \begin{align}
        n \ge \frac{1}{\lambda^2_\k \epsilon^{k^\star}} \gtrsim \frac{1}{\lambda^2_\k} \cdot \qty(\frac{d}{\log^2_d(2qn)})^{\k/2}.
    \end{align}
    Now let $c_\k$ be a sufficiently small constant which depends only on $\k$ and assume for the sake of contradiction that
    \begin{align}
        n \le \frac{c_\k}{\lambda^2_\k} \cdot \qty(\frac{d}{\log^2_d(q)})^{\k/2}.
    \end{align}
    Then we must have that
    \begin{align}
        \log_d(2n) \le \frac{\k+1}{2} + \log_d\qty(2c_\k) \le \k.
    \end{align}
    Therefore for the algorithm to succeed we must have
    \begin{align}
        n \gtrsim \frac{1}{\lambda^2_\k} \qty(\frac{d}{\log^2_d(2qn)})^{\k/2} \ge \frac{1}{\lambda^2_\k} \qty(\frac{d}{(\log_d(q) + \k)^2})^{\k/2} \gtrsim_\k \frac{1}{\lambda^2_\k} \qty(\frac{d}{\log_d(q)})^{\k/2} = \frac{n}{c_\k}.
    \end{align}
    Therefore for sufficiently small $c_\k$ we have derived a contradiction so we must have that
    \begin{align}
        n \ge \frac{c_\k}{\lambda^2_\k} \cdot \qty(\frac{d}{\log^2_d(q)})^{\k/2}
    \end{align}
    which completes the proof.
\end{proof}

\begin{lemma}[Lower Bound for Highly Periodic Neuron]\label{lem:clwe_lower_bound}
    Let $\sigma(Z) = \cos(2\pi \gamma Z)$. Then for $w,w'$ with $m := \abs{w \cdot w'} < \frac{1}{8 \pi \gamma}$,
    \begin{align}
        \chi^2_0(\PP_w,\PP_{w'}) \lesssim \exp(-2\pi^2 \gamma^2) + \frac{m^\gamma}{1-m}.
    \end{align}
    In addition, when $\gamma = c d^{1/4}$ and $m = d^{-1/4}$ for a sufficiently small constant $c$, any algorithm requires $m \gtrsim d^{d^{1/4}}$ queries to $\mathrm{VSTAT}(n)$ to recover $w^\star$ when $\P = \gamma(z) \delta_{\sigma(z)}(y)$ unless $n \gtrsim d^{d^{1/4}}$.
\end{lemma}
\begin{proof}
    We will begin by computing $\P_{z|y}$. Note that the level sets are given by:
    \begin{align}
        \{z ~:~ \sigma(z) = y\} = \bigcup_{j \in \mathbb{Z}} \{Z_j^+(y),Z_j^-(y)\} \qq{where} Z_j^{\pm}(y) := \frac{\pm \arccos(y) + 2\pi j}{2 \pi \gamma}.
    \end{align}
    Therefore by the coarea formula,
    \begin{align}
        \P_{z | y} = \frac{\sum_{j \in \mathbb{Z}} \gamma(Z_j^+(y)) \delta_{Z_j^+(y)}(z) + \gamma(Z_j^-(y)) \delta_{Z_j^+(y)}(z)}{\sum_{j \in \mathbb{Z}} \gamma(Z_j^+(y)) + \gamma(Z_j^-(y))}.
    \end{align}
    We will now use the Poisson summation formula. For any $f$,
    \begin{align}
        \sum_{j \in \mathbb{Z}} f(Z_j^+(y)) = \gamma \sum_{j \in \mathbb{Z}} \cos(j \arccos(y)) \hat f( \gamma j).
    \end{align}
    Plugging in $f = \gamma$ gives
    \begin{align}
        \sum_{j \in \mathbb{Z}} \gamma(Z_j^+(y)) &= \gamma \sum_{j \in \mathbb{Z}} \cos(j \arccos(y)) \exp(-2\pi^2 j^2 \gamma^2) \\
        &= \gamma + \gamma\sum_{j \ne 0} \cos(j \arccos(y)) \exp(-2\pi^2 j^2 \gamma^2).
    \end{align}
    This second term can be bounded by
    \begin{align}
        2\sum_{j > 1} \exp(-2\pi^2 j^2 \gamma^2) \le 2\sum_{j > 1} \exp(-2\pi^2 j \gamma^2) = \frac{2}{\exp(2\pi^2 \gamma^2) - 1} \le 4\exp(-2\pi^2 \gamma^2).
    \end{align}
    We can conduct the same analysis with $Z_j^-$ so the normalization factor is $2\gamma(1 + O(e^{-c \gamma^2}))$. We now apply the same technique to compute $\{\zeta_k\}$. We have that for $k \le \gamma$,
    \begin{align}
        &\abs{\sum_{j \in \mathbb{Z}} \gamma(Z_j^+(y)) h_k(Z_j^+(y))} \nonumber\\
        &= \gamma \abs{\sum_{j \in \mathbb{Z}} \cos(j \arccos(y)) (2\pi j \gamma)^k \exp(-2 \pi^2 j^2 \gamma^2) \exp(-j i \pi/2)} \nonumber\\
        &= \abs{2\gamma \sum_{j > 1} \cos(j \arccos(y)) (2\pi j \gamma)^k \exp(-2 \pi^2 j^2 \gamma^2) \exp(-j i \pi/2)} \nonumber\\
        &\le 2\gamma \sum_{j > 1} (2\pi j \gamma)^k \exp(-2 \pi^2 j^2 \gamma^2) \nonumber\\
        &\le 2\gamma (2\pi \gamma)^k \sum_{j > 1} j^k \exp(-2 \pi^2 j \gamma^2) \nonumber\\
        &\le 2\gamma (2\pi \gamma)^k \frac{\exp(2k\pi^2 \gamma^2) + k!\exp(2(k-1)\pi^2 \gamma^2)}{(\exp(2\pi^2 \gamma^2)-1)^{k+1}} \nonumber\\
        &\lesssim \gamma (4\pi \gamma)^k \exp(-2\pi^2 \gamma^2).
    \end{align}
    Therefore for $m < \frac{1}{8 \pi \gamma}$,
    \begin{align}
        \chi^2_0(\PP_w,\PP_{w'})
        &\lesssim \exp(-2\pi^2 \gamma^2) \sum_{k \le \gamma} (4 \pi \gamma m_j)^k + \frac{m^{\gamma}}{1-m} \\
        &\le 2\exp(-2\pi^2 \gamma^2) + \frac{m^{\gamma}}{1-m}.
    \end{align}
    The SQ lower bound now directly follows from \Cref{lem:easy_sq}.
\end{proof}

\subsection{Low Degree Polynomial Lower Bound}
\label{app:low_degree_lower}

Using \Cref{lem:likelihood_hermite}, we can directly prove \Cref{thm:low_degree}:

\lowdegthm*
\begin{proof}[Proof of \Cref{thm:low_degree}]
    Let $\L_w := \frac{d\PP_w}{d\PP_0}$ denote the likelihood ratio conditioned on $w$. We begin by computing the full likelihood ratio:
	\begin{align*}
		\L\qty((x_1,y_1),\ldots,(x_n,y_n))
		 & = \frac{\E_w\qty[\prod_{i=1}^n \PP_w[x_i,y_i]]}{\prod_{i=1}^n \PP[x_i]\PP[y_i]} = \E_w\qty[\prod_{i=1}^n \L_w(x_i,y_i)].
	\end{align*}
	Then by \Cref{lem:likelihood_hermite}, we can expand this as
	\begin{align*}
		\L
		 & = \E_w\qty[\prod_{i=1}^n \qty(\sum_{k \ge 0} \zeta_k(y_i) h_k(x_i \cdot w))].
	\end{align*}
	We will isolate the low degree part with respect to $\{x_1,\ldots,x_n\}$, which we denote by $\L_{\le D}$. To compute this, we need to switch the product and the summation:
	\begin{align*}
		\L
		 & = \E_w\qty[\sum_{p=0}^\infty \sum_{k_1 + \ldots + k_n = p} \qty(\prod_{i=1}^n \zeta_{k_i}(y_i) h_{k_i}(x_i \cdot w))].
	\end{align*}
	We note that each term on the right hand side is a polynomial in $x_1,\ldots,x_n$ of degree $p$ which is orthogonal to all polynomials of degree less than $p$. Therefore $\L_{\le D}$ is given by:
	\begin{align*}
		\L_{\le D}
		 & = \E_w\qty[\sum_{p=0}^D \sum_{k_1 + \ldots + k_n = p} \qty(\prod_{i=1}^n \zeta_{k_i}(y_i) h_{k_i}(x_i \cdot w))].
	\end{align*}
	We can now use the orthogonality property of Hermite polynomials to compute the norms with respect to the null distribution $\PP_0$. If if $w,w'$ are independent draws from the prior on $w$ then:
	\begin{align*}
		\norm{\L_{\le D}}_{L^2(\PP_0)}^2
		 & = \E_{w,w'}\qty[\sum_{p=0}^D \sum_{k_1 + \ldots + k_n = p} \qty(\prod_{i=1}^n \lambda_{k_i}^2 (w \cdot w')^{k_i})] \\
		 & = \sum_{p=0}^D \E[(w \cdot w')^p] \sum_{k_1 + \ldots + k_n = p} \qty(\prod_{i=1}^n \lambda_{k_i}^2).
	\end{align*}
    Let $z$ be a random variable with distribution $w \cdot w'$ where $w,w'$ are drawn independently from the prior on $w$, and let $\mathcal{P}_{\le D}$ be the projection operator onto polynomials of degree at most $D$ in $z$. Then we can rewrite the above expression as:
    \begin{align*}
		\norm{\L_{\le D}}_{L^2(\PP_0)}^2
		 & = \E_z \qty[\mathcal{P}_{\le D} \qty[\qty(\sum_{k \ge 0} \lambda_k^2 z^k)^n]].
	\end{align*}
    By linearity of expectation and of the projection operator $\mathcal{P}_{\le D}$, we can expand this using the binomial theorem:
    \begin{align*}
		\norm{\L_{\le D}}_{L^2(\PP_0)}^2
		 & = \sum_{j \ge 0} \binom{n}{j} \E\qty[\mathcal{P}_{\le D}\qty[\qty(\sum_{k \ge \k} \lambda_k^2 z^k)^j]].
	\end{align*}
    We can now lower bound $\norm{\L_{\le D}}_{L^2(\PP_0)}^2$ by isolating the terms where $k = \k$:
    \begin{align*}
		\norm{\L_{\le D}}_{L^2(\PP_0)}^2
		 &\ge \sum_{j = 0}^{\lfloor D/\k \rfloor} \binom{n}{j} \E\qty[\mathcal{P}_{\le D}\qty[\lambda_\k^{2j} z^{\k j}]] \\
         &= \sum_{j = 0}^{\lfloor D/\k \rfloor} \binom{n}{j} \lambda_k^{2j} \1_{2 \mid \k j} \frac{(\k j -1)!!}{\prod_{i=0}^{\k j/2 - 1} (d+2i)} =: \mathcal{R}^\star.
	\end{align*}
    We can similarly upper bound this expression by using that $\lambda_k^2 \le 1$. Plugging this in for $k > \k$ gives:
    \begin{align*}
		\norm{\L_{\le D}}_{L^2(\PP_0)}^2
		 &\le \sum_{j=0}^{\lfloor D/\k \rfloor} \binom{n}{j} \E\qty[\mathcal{P}_{\le D}\qty[\qty(\lambda_\k^2 z^\k + \frac{z^{\k+1}}{1-z})^j]] \\
         &= \L^\star + \sum_{j=0}^{\lfloor D/\k \rfloor} \binom{n}{j} \E\qty[\mathcal{P}_{\le D}\qty[ \sum_{i=1}^j \binom{j}{i} (\lambda_\k^2 z^\k)^{j-i} \qty(\frac{z^{\k+1}}{1-z})^{i}]] \\
         &= \L^\star + \sum_{j=0}^{\lfloor D/\k \rfloor} \binom{n}{j} \E\qty[\sum_{i=1}^j \binom{j}{i} (\lambda_\k^2 z^\k)^{j-i} \mathcal{P}_{\le D-\k (j-i)}\qty[\qty(\frac{z^{\k+1}}{1-z})^{i}]] \\
         &\le \L^\star + \sum_{j=0}^{\lfloor D/\k \rfloor} \binom{n}{j} \E\qty[\sum_{i=1}^j \binom{j}{i} (\lambda_\k^2 z^\k)^{j-i} \qty[\qty(\frac{z^{\k+1}}{1-z})^{i}]] \\
         &= \L^\star + \sum_{j=0}^{\lfloor D/\k \rfloor} \binom{n}{j} \sum_{i=1}^j \binom{j}{i} \lambda_\k^{2(j-i)} \E\qty[\frac{z^{\k j + i}}{(1-z)^i}] \\
         &\le \L^\star + C \sum_{j=0}^{\lfloor D/\k \rfloor} \binom{n}{j} \sum_{i=1}^j \binom{j}{i} \lambda_\k^{2(j-i)} \E\qty[z^{\k j + i}]
	\end{align*}
    where the last line follows from \citep[Lemma 26]{damian2023smoothing}. We can now relate the $\k j + i$ and the $\k j$-th moments:
    \begin{align*}
		\norm{\L_{\le D}}_{L^2(\PP_0)}^2
         &\le \L^\star + C \sum_{j=0}^{\lfloor D/\k \rfloor} \binom{n}{j} \sum_{i=1}^j \binom{j}{i} \lambda_\k^{2(j-i)} \norm{z}_{j\k}^{j\k} \qty(\frac{j (\k+1)}{d})^{i/2} \\
         &= \L^\star + C \sum_{j=0}^{\lfloor D/\k \rfloor} \binom{n}{j} \lambda_\k^{2j} \norm{z}_{j\k}^{j\k} \qty(\qty(1 + \sqrt{\frac{j (\k+1)}{\lambda_\k^4 d}})^{j} - 1) \\
         &\le \L^\star + \sqrt{\frac{C D}{\lambda_\k^4 d}} \sum_{j=0}^{\lfloor D/\k \rfloor} \binom{n}{j} \lambda_\k^{2j} \norm{z}_{j\k}^{j\k} \\
         &\le \L^\star(1 + o_d(1)),
	\end{align*}
    which completes the proof.
\end{proof}

This implies the following corollary:
\lowdegcorollary*
\begin{proof}
    The weak recovery threshold follows directly from \Cref{thm:low_degree} by setting $\delta = d^{\gamma - \frac{\k}{2}}$. For the strong recovery threshold, we have by \Cref{thm:low_degree},
    \begin{align*}
        \norm{\L_{\le D}}_{\PP_0}^2
        &\lesssim \sum_{j=0}^{\lfloor \frac{D}{k^\star} \rfloor} \1_{2 \mid k^\star j} (k^\star j-1)!! \frac{(\lambda_{k^\star}^2 \delta)^{j}}{j!} \\
        &\lesssim \sum_{j=0}^{\lfloor \frac{D}{k^\star} \rfloor}\qty(\frac{e}{j})^j \qty(\frac{\k j}{e})^{\k j/2} (\lambda_{k^\star}^2 \delta)^{j} \\
        &= \sum_{j=0}^{\lfloor \frac{D}{\k} \rfloor}\qty(\frac{\delta \k^{\k/2} j^{\k/2-1}}{e^{\k/2-1}})^j \\
        &\le \sum_{j=0}^{\infty}\qty(\frac{\delta \k D^{\k/2-1}}{e^{\k/2-1}})^j \\
        &\le \sum_{j=0}^{\infty}\qty(2/e)^j \\
        &= \frac{1}{1-2/e}
    \end{align*}
    which completes the proof.
\end{proof}

\subsection{Reduction from Single-Index Model to NGCA}

We restate and prove Proposition~\ref{prop:singleNGCA}.

\singleNGCA*
\begin{proof}
    Let us fix a single index problem, $\PP_{w^*, \P}$ with \sqexp $ k = \k(\P)$, and direction $w^* \in S^{d-1}$. By definition of the \sqexp, we have that $\|\zeta_k\|_{\P_y} > 0$. Define $A_+ = \{ y; \zeta_k(y) > 0\}$, $A_{-} = \{ y; \zeta_k(y) < 0\}$ and $\alpha:= \max\left\{ \P_y(A_+) , \P_y(A_-) \right\}$. Then $\alpha > 0$ as otherwise $\zeta_k$ would be $0$ $\P_y$-a.e. In the following, we consider without loss of generality that $\alpha = \P_y(A_+)$. 
    
   Let us define $\mu(dz) = \alpha^{-1} \int_{A_+} \P_{z|y}(dz,y) \, \P_y(dy) \in \mathcal{P}(\R)$,  where we recall that $\P_{z|y}$ is the conditional distribution of $Z$ given $Y$. 
   Let us assume that we have access to $\{(x_i, y_i)\}_{i\leq n}$, i.i.d. samples distributed according to $\PP_{w^\star, \P}$.  Now, let us perform \textit{rejection sampling}, that is, for all $i \leq n$, we keep $x_i$ if and only if $y_i \in A_{+}$. This builds a data set of $\tilde{n}(n)$ samples $\{\tilde{x}_i\}_{i\leq \tilde{n}(n)}$ that are i.i.d. from a NGCA-model with distribution $\mu$ and direction $w^\star$. From $\alpha>0$, we have that $\E_{\mu}[h_k(Z)] = \alpha^{-1} \int \zeta_k(y) \P_y(dy) > 0$. Moreover, the number of samples obtained after rejection sampling will concentrate (with binomial tails) at $\tilde{n} = \alpha n$, with $\alpha$ independent of dimension. Hence any SQ-algorithm that solves the built NGCA model, i.e. that is able to find the direction $w^*$, will solve the later single-index one efficiently. 
\end{proof}

\subsection{Reduction from NGCA to Single-Index Model Variant}
\label{sec:pancake2}

Given $X' \sim \mathbb{Q}_{\mu, w^\star}$, we consider $X \sim \gamma_d$ drawn independently from $X'$, and $Y= h_k( X \cdot X')$. Assume wlog that $w^\star$ is the first canonical vector, and let us write $x' = (x'_1, \bar{x}')$ and $x=(x_1, \bar{x})$.
The conditional distribution of $Y$ given $X$ satisfies 
\begin{align}
  p(y|x) &= \int p(y|x, x') \mathbb{Q}(dx') = \int_{\R} \left(\int_{\R^{d-1}} \delta_{y - h_k(x_1 x'_1  +  \bar{x} \cdot \bar{x}') } \gamma_{d-1}(d\bar{x}') \right) \mu(dx_1') \nonumber\\
  &= \int_{\R} \left(\int_{\R} \delta_{y - h_k(x_1 x'_1  +  \|\bar{x} \| z) } \gamma_{1}(dz) \right) \mu(dx_1') \nonumber\\
  &= \E_{x_1' \sim \mu} \left[ \int_{\{ z; h_k( x_1 x_1' + \|\bar{x} \| z) = y\}} \frac{\gamma_1(z)}{\|\bar{x}\| |h_k'(x_1 x_1' + \|\bar{x} \| z)|} d\mathcal{H}_0(z) \right] \nonumber \\
  &= \tilde{\psi}_k\left( y , x_1, \|x\| \right) ~,
\end{align}
where we %
used the coarea formula and the rotational symmetry of the Gaussian measure. In other words, the conditional distribution $p(y|x)$ now depends on two scalar summary statistics: the projection along one hidden direction \emph{and} the norm of $x$. The limitation of this construction, however, is the fact that the signal strength $x_1$ is of order $O(1)$, while the fluctuations of the uninformative norm $\|x\|$ are also of order $\Theta(1)$, leading to a presumably harder estimation task than that of Definition \ref{def:single-index_model}. 

%% file: proofs_upperbound.tex
\section{Proofs of Section \ref{sec:partialtrace}}
\label{app:proofs_upper}

We begin by defining the un-normalized Hermite polynomials and Hermite tensors:
\begin{definition}\label{def:unnormalized_hermite}
    \begin{align}
        \He_k(x) := \sqrt{k!} h_k(x) \qand \bs{\He}_k(x) := \sqrt{k!} \bs{h}_k(x).
    \end{align}
\end{definition}
Unlike $h_k$, $\He_k$ naturally tensorizes, i.e.
\begin{align}
    \bs{\He}_k(x)_{i_1,\ldots,i_k} = \prod_{i=1}^d \He_{|\{j ~:~ i_j = i\}|}(x_i).
\end{align}
For example,
\begin{align}
    \bs{\He}_3(x)_{1,1,2} = \He_2(x_1) \He_1(x_2).
\end{align}

\subsection{Truncating \texorpdfstring{$\zeta_k(Y)$}{zeta}}

Throughout this section, let
\begin{align}
    \tilde \lambda_k^2 := \frac{\lambda_k^2}{\log(3/\lambda_k)^{k/2}}.
\end{align}

\begin{lemma}\label{lem:bounded_T_sq_alg}
    Let $\P \in \mathcal{G}$. Then for any $k$, there exists a bounded function $\mathcal{T}: \R \to [-1,1]$ such that $\E_{\P}[\mathcal{T}(Y) h_k(Z)] \gtrsim \tilde \lambda_k^2$ and $\E_{\P}[\mathcal{T}(Y)^2] \lesssim \tilde \lambda_k^2$.
\end{lemma}

\begin{proof}
    Let $\k = \k(\P)$ and recall $\zeta_{\k}(y) := \E_{\P}[h_\k(Z)|Y=y]$ where $(Z,Y) \sim \P$. We will fix a truncation radius $R$ and define $\mathcal{T}: \R \to \R$ by $\mathcal{T}(y) := \frac{1}{R} \zeta_{\k}(y) \1_{\abs{\zeta_\k(y)} \le R}.$ Note that $\abs{\mathcal{T}(y)} \le 1$ by definition. Then,
    \begin{align}
        R \cdot \E_\P[\mathcal{T}(Y)h_\k(Z)] &= \E_\P[\zeta_\k(Y)^2] - \E_\P[\zeta_\k(Y)^2 \1_{\abs{\zeta_\k(Y)} \ge R}] \nonumber \\
        &= \lambda_\k^2 - \E_\P[\zeta_\k(Y)^2 \1_{\abs{\zeta_\k(Y)} \ge R}].
    \end{align}
    The first term is nonzero because by the definition of $\k$. Therefore it suffices to prove that the second term vanishes as $R \to \infty$. Using \Cref{lem:poly_tail_holder} and Markov's inequality, we can bound this second term by:
    \begin{align}
        \abs{\E_\P[\zeta_\k(Y)^2 \1_{\abs{\zeta_\k(Y)} \ge R}]}
        &\le \sqrt{\E_\P[\zeta_\k(Y)^2h_\k(Z)^2] \PP[\abs{\zeta_\k(Y)} \ge R]} \nonumber\\
        &\lesssim \sqrt{\E[\zeta_k(Y)^2] \log(3/\E[\zeta_k(Y)^2])^{\k}} \cdot \frac{\E[\zeta_k(Y)^2]}{R} \nonumber\\
        &= \frac{\log(3/\lambda_\k)^{\k/2} \lambda_{\k}^2}{R}.
    \end{align}
    Therefore taking $R = C \log(3/\lambda_\k)^{\k/2}$ for a sufficiently large constant $C$ gives:
    \begin{align}
        \E_\P[\mathcal{T}(Y)h_\k(Z)] \gtrsim \frac{\lambda_\k^2}{\log(3/\lambda_\k)^{\k/2}}.
    \end{align}
    In addition, $\mathcal{T}$ also satisfies:
    \begin{align}
        \E_\P[\mathcal{T}(Y)^2] \le \frac{1}{R} \E_\P[\zeta_\k(Y)^2] \lesssim \frac{\lambda_\k^2}{\log(3/\lambda_\k)^{\k/2}}.
    \end{align}
\end{proof}

\subsection{Tensor Power Iteration}

The following lemma shows that Tensor Power Iteration can be computed with $O(d)$ memory as it does not require storing the tensor $\bs{h}_k(x)$.
\begin{lemma}[Efficient Tensor Power Iteration]\label{lem:efficient_tensor_power_iter}
    For any $v \in S^{d-1}$,
    \begin{align}
        \bs{\He}_k(x)[v^{\otimes (k-1)}] &= x \He_{k-1}(x \cdot v) - (k-1) v \He_{k-2}(x \cdot v).
    \end{align}
\end{lemma}
\begin{proof}
    Because $\bs{\He}_k(x)$ tensorizes, we have that for $w \perp v$,
    \begin{align}
        \bs{\He}_k(x)[v^{\otimes (k-1)}] \cdot v = \He_k(x \cdot v) \qand \bs{\He}_k(x)[v^{\otimes (k-1)}] \cdot w = (x \cdot w) \He_{k-1}(x \cdot v).
    \end{align}
    Therefore,
    \begin{align}
        \bs{\He}_k(x)[v^{\otimes (k-1)}]
        &= x\He_{k-1}(x \cdot v) + v [\He_k(x \cdot v) - (x \cdot v) \He_{k-1}(x \cdot v)] \nonumber \\
        &= x\He_{k-1}(x \cdot v) - (k-1) v \He_{k-2}(x \cdot v)
    \end{align}
    by the two term recurrence for $\He_k$.
\end{proof}

\begin{lemma}[Tensor Power Iteration Convergence]\label{lem:power_iter_convergence}
    Let $c,C = c(k),C(k)$ be constants depending only on $k$. Let $v \in S^{d-1}$ with $\alpha := v \cdot w^\star$. Let $\mathcal{D} = \{(x_i,y_i)\}_{i \in [n]}$ be a dataset of size $n$ with $\abs{y_i} \le 1$. Let 
    \begin{align}
        \hat v = \frac{1}{n} \sum_{(x,y) \in \mathcal{D}} y \bs{h}_k(x)[v^{\otimes (k-1)}] ,\qand v' = \frac{\hat v}{\|\hat v\|}~.
    \end{align}
    We will denote the effective signal strength by $\tilde \beta_k$:
    \begin{align}
        \tilde \beta_k := \frac{\E[Y h_k(Z)]}{\sqrt{\E[Y^2]\log^k(3/\E[Y^2])}}
    \end{align}
    and we will assume without loss of generality that $\tilde \beta_k > 0$. Then for $\alpha \in [d^{-1/4},1/2]$ we have that if
    \begin{align}
    n \ge \frac{C d}{\tilde \beta^2_k \alpha^{2k-4}},
    \end{align}
    then with probability at least $1-2e^{-d^c}$, $v' \cdot w \ge 2 (v \cdot w)$. In addition, for any $\epsilon > 0$ if
    \begin{align}
    n \ge \frac{Cd}{\epsilon \tilde \beta^2_k \alpha^{2k-2}},
    \end{align}
    then with probability at least $1-2d^{-c}$, $v' \cdot w^\star \ge 1-\epsilon$.
\end{lemma}
\begin{proof}
    Let $\tilde \lambda^2 := \E[Y^2] \log^k(3/\E[Y^2])$. Note that $\E[\hat v] = w^\star \beta_k \alpha^{k-1}$. In addition for any $w$, we have by \Cref{lem:poly_tail_holder} that
    \begin{align}
        \E[(\hat v \cdot w)^2] - (\E[\hat v] \cdot w)^2
        \le \frac{1}{n} \E[Y^2 \bs{h}_k(x)[v^{\otimes (k-1)},w]^2] \lesssim \frac{\tilde \lambda^2}{n}.
    \end{align}
    Therefore by \Cref{lem:poly_tail_bound}, for any $w \in S^{d-1}$, with probability at least $1-\delta$, 
    \begin{align}
        \abs{(\hat v-\E[\hat v]) \cdot w}
        \lesssim \sqrt{\frac{\tilde \lambda^2 \log(1/\delta)}{n}} + \frac{\log(n/\delta)^{k/2}}{n}
        \lesssim \sqrt{\frac{\tilde \lambda^2 \log(1/\delta)}{n}},
    \end{align}
    because $n \ge d$. Next, for any $x$ let $x^\perp := P_{w,v}^\perp x$ and decompose $\hat v^\perp$ as:
    \begin{align}
        \hat v^\perp = \frac{1}{\sqrt{k}} \cdot \frac{1}{n} \sum_{(x,y) \in \mathcal{D}} y x^\perp h_{k-1}(v \cdot x).
    \end{align}
    Then $x^\perp$ is independent of $y,h_{k-1}(v \cdot x)$ so this is equal in distribution to
    \begin{align}
        N\left(0,\frac{1}{k} \cdot \frac{1}{n^2} \sum_{(x,y) \in \mathcal{D}} y^2 h_{k-1}(v \cdot x)^2 I_d\right).
    \end{align}
    Therefore by the standard $\chi^2$ tail bound, with probability at least $1-2e^{-d}$,
    \begin{align}
        \|v^\perp\|^2 \lesssim \frac{d}{n} \cdot \frac{1}{n} \sum_{(x,y) \in \mathcal{D}} y^2 h_{k-1}(v \cdot x)^2.
    \end{align}
    Again by \Cref{lem:poly_tail_bound} we have that with probability at least $1-\delta$,
    \begin{align}
        \|v^\perp\|^2 \lesssim \frac{d}{n} \cdot \qty[\tilde \lambda^2 + \sqrt{\frac{\tilde \lambda \log(1/\delta)}{n}} + \frac{\log(n/\delta)^k}{n}].
    \end{align}
    Therefore for $n \ge d$, we have with probability $1-2e^{-d^c}$ that
    $
        \|v^\perp\|^2 \lesssim \frac{\tilde \lambda^2 d}{n}.
    $
    In combination with the above bound on $\abs{(\hat v-\E[\hat v]) \cdot w}$, this gives that with probability at least $1-2e^{-d^c}$,
    $$
        \|\hat v - \E[\hat v]\|^2 \lesssim \frac{\tilde \lambda^2 d}{n}.
    $$
    In the first case when $\alpha \in [d^{-1/4},1/4]$ and $n \gtrsim \frac{d}{\tilde \beta^2_k \alpha^{2k-4}}$, this gives that with probability at least $1-2e^{-d^c}$,
    \begin{align}
        v' \cdot w^\star
        &= \frac{\hat v \cdot w^\star}{\|\hat v\|} \nonumber\\
        &\ge \frac{\beta_k \alpha^{k-1} + (\hat v - \E[\hat v]) \cdot w^\star}{\beta_k \alpha^{k-1} + \|\hat v - \E[\hat v]\|} \nonumber\\
        &\ge \frac{\frac{3}{4} \beta_k \alpha^{k-1}}{\beta_k \alpha^{k-1} + \|\hat v - \E[\hat v]\|} \nonumber\\
        &\ge \frac{\frac{3}{4} \beta_k \alpha^{k-1}}{\beta_k \alpha^{k-1} + \beta_k \alpha^{k-2}/8} \nonumber\\
        &= \frac{6 \alpha}{8\alpha + 1} \nonumber\\
        &\ge 2\alpha
    \end{align}
    because $\alpha \le 1/4$. In the second case when $n \gtrsim \frac{d}{\epsilon \tilde \beta^2_k \alpha^{2k-2}}$, we have with probability at least $1-2e^{-d^c}$,
    \begin{align}
        v' \cdot w^\star = \frac{\hat v \cdot w^\star}{\|\hat v\|} \ge \frac{\beta_k \alpha^{k-1} + (\hat v - \E[\hat v]) \cdot w^\star}{\beta_k \alpha^{k-1} + \|\hat v - \E[\hat v]\|} \ge \frac{\beta_k \alpha^{k-1} (1-\epsilon/2)}{\beta_k \alpha^{k-1}(1+\epsilon/2)} \ge 1-\epsilon
    \end{align}
    which completes the proof.
\end{proof}

\subsection{Partial Trace}
We begin by defining the un-normalized orthogonal polynomials for $\chi^2$ random variables:
\begin{definition}
    Let $p_k^{(d)}(r)$ be defined by
    \begin{align}
        p_k^{(d)}(r) = \sum_{j=0}^k \binom{k}{j} r^j (-1)^{k-j} \prod_{i=j}^{k-1} (d+2i).
    \end{align}
\end{definition}
These satisfy the following orthogonality relation:
\begin{lemma}
    Let $r \sim \chi^2(d)$. Then,
    \begin{align}
        \E[p_j^{(d)}(r)p_k^{(d)}(r)] = \delta_{jk} k! 2^k \prod_{i=0}^{k-1} (d+2i).
    \end{align}
\end{lemma}
\begin{proof}
    For any $j \le k$,
    \begin{align}
        \E[r^j p_k^{(d)}(r)]
        &= \sum_{l=0}^k \binom{k}{l} \E[r^{l+j}] (-1)^{k-l} \prod_{i=l}^{k-1} (d+2i) \nonumber\\
        &= \sum_{l=0}^k \binom{k}{l} (-1)^{k-l} \prod_{i=0}^{l+j-1} (d+2i) \prod_{i=l}^{k-1} (d+2i) \nonumber\\
        &= \prod_{i=0}^{k-1} (d+2i) \sum_{l=0}^k \binom{k}{l} (-1)^{k-l} \prod_{i=l}^{l+j-1} (d+2i) \nonumber\\
        &= \prod_{i=0}^{k-1} (d+2i) \sum_{l=0}^k \binom{k}{l} (-1)^{k-l} \binom{\frac{d}{2}+l+j-1}{j}2^{j} j! \nonumber\\
        &= \1_{j = k} 2^k k! \prod_{i=0}^{k-1} (d+2i) 
    \end{align}
    where the last line followed from the fact that the $k$th finite difference of the polynomial $g(l) = \binom{\frac{d}{2}+l+j-1}{j}$ is $0$ unless $j=k$, in which case it is 1.
\end{proof}
These are related to the Hermite polynomials by the following lemma:
\begin{lemma} \label{lem:hermite_to_chisquare_poly} For any $x \in \R^d$,
    \begin{align}
        p_k^{(d)}(\|x\|^2) = \bs{\He}_{2k}(x)[I^{\otimes k}]
    \end{align}
\end{lemma}
\begin{proof}
    Note that both sides are monic polynomials in $\norm{x}^2$. In addition, the right hand side is an orthogonal family of polynomials in $\norm{x}^2$ because for $j \ne k$,
    \begin{align}
        \E\qty[\bs{h}_{2j}(x)[I^{\otimes j}]\bs{h}_{2k}(x)[I^{\otimes k}]] = 0
    \end{align}
    because $\E[\bs{h}_{2j} \otimes \bs{h}_{2k}] = 0$. Because the set of monic orthogonal polynomials is unique, we must have that $p_k^{(d)}(\|x\|^2) = (2k!)^{1/2} \bs{h}_{2k}(x)[I^{\otimes k}]$ for all $x$.
\end{proof}

\begin{lemma}[Efficient Computation of Partial Trace]\label{lem:efficient_partial_trace} If $k = 2j+1$,
    \begin{align}
        \bs{\He}_k(x)[I^{\otimes j}] = p_{j}^{(d+2)}(\norm{x}^2) x
    \end{align}
    and if $k = 2j+2$. Then,
    \begin{align}
        \bs{\He}_k(x)[I^{\otimes j}] = p_{j}^{(d+4)}(\norm{x}^2) xx^T - p_{j}^{(d+2)}(\norm{x}^2) I_d.
    \end{align}
\end{lemma}
\begin{proof}
    We start with \Cref{lem:hermite_to_chisquare_poly} for $j+1$:
    \begin{align}
        \bs{\He}_{2j+2}(x)[I^{\otimes (j+1)}] = p_{j+1}^{(d)}(\|x\|^2).
    \end{align}
    Differentiating both sides with respect to $x$ in the $v$ direction gives:
    \begin{align}
        (2j+2) \sym(v \otimes \bs{\He}_{2j+1}(x))[I^{\otimes j+1}] = (2j+2) (x \cdot v) p_{j}^{(d+2)}(\|x\|^2).
    \end{align}
    Note that the left hand side is simply equal to $\bs{\He}_{2j+1}(x)[I^{\otimes j}] \cdot v$, so
    and rearranging gives the result for $k$ odd. Next, we will differentiate again in the $v$ direction. Then we have that:
    \begin{align}
        &(2j+1) \sym(v \otimes v \otimes \bs{\He}_{2j}(x))[I^{\otimes j+1}] \\
        &= (2j) (x \cdot v)^2 p_{j-1}^{(d+4)}(\|x\|^2) + p_{j}^{(d+2)}(\|x\|^2).
    \end{align}
    Of the $(2j+1)!!$ pairings on the left hand side, $(2j-1)!!$ pair $v$ with itself and then pair all indices of $\bs{\He}_{2j}(x)$ while $2j(2j-1)!!$ pair $2j-2$ indices of $\bs{\He}_{2j}(x)$ and the remaining two with $v$. Therefore the left hand side is equal to:
    \begin{align}
        p_j^{(d)}(\|x\|^2) + (2j)v^T \bs{\He}_{2j}(x)[I^{\otimes (j-1)}] v.
    \end{align}
    Therefore we must have
    \begin{align}
        v^T \bs{\He}_{2j}(x)[I^{\otimes (j-1)}] v = (x \cdot v)^2 p_{j-1}^{(d+4)}(\|x\|^2) + \frac{p_{j}^{(d+2)}(\|x\|^2)-p_{j}^{d}(\|x\|^2)}{2j}.
    \end{align}
    Because $p_j^{(d+2)} - p_j^{(d)} = - (2j) p_{j-1}^{(d+2)}$, this reduces to 
    \begin{align}
        v^T \bs{\He}_{2j}(x)[I^{\otimes (j-1)}] v = (x \cdot v)^2 p_{j-1}^{(d+4)}(\|x\|^2) - p_{j-1}^{(d+2)}.
    \end{align}
    As this is true for any $v \in S^{d-1}$, this completes the proof.
\end{proof}

\subsection{The Even Case}
\label{sec:evencasepartial}
First we will start with the even case. We will show that $v := v_1(M_n)$ has good alignment with $w^\star$.

\begin{lemma}\label{lem:partial_trace_even}
    Let $k \ge 4$ and let $M_n$ be the partial trace matrix in \Cref{alg:partial_trace} and assume that $\abs{y_i} \le 1$ for all $i$. Define the effective signal strength
    \begin{align}
        \tilde \beta_k := \frac{\E[Y h_k(Z)]}{\sqrt{\E[Y^2]\log^k(3/\E[Y^2])}}.
    \end{align}
    Then for $k \ge 4$, with probability at least $1-2e^{-d^c}$,
    \begin{align}
        \norm{M_n - \E[M_n]}_2 \lesssim \sqrt{\frac{\beta_k^2 d^{k/2}}{\tilde \beta_k^2 n}}.
    \end{align}
    In addition, for 
    \begin{align}
        n \gtrsim \frac{d^{k/2}}{\epsilon \tilde \beta_k^2},
    \end{align}
    we have that $v = v_1(M_n)$ satisfies $(v \cdot w^\star)^2 \ge 1-\epsilon$.
\end{lemma}
\begin{proof}
    Let $\eta = \frac{\beta_k^2}{\tilde \beta_k^2} = \E[Y^2] \log^k(3/\E[Y^2]])$. We will use \Cref{lem:matrix_concentration_universality}. First we compute $\sigma^2 := \|\E[(M_n - M)^2]\|_2$.
    \begin{align}
        \sigma^2
        &:= \|\E[(M_n - M)^2]\|_2 \nonumber\\
        &= \frac{1}{n} \|\E[(M_1 - M)^2]\|_2 \nonumber\\
        &\lesssim \frac{1}{n} \cdot \|\E[M_1^2]\|_2 \nonumber\\
        &= \frac{1}{n} \cdot \left\|\E_{X,Y}\left[Y^2 \bs{h}_\k(X)\left[I^{\otimes \frac{\k-2}{2}}\right]^2\right]\right\|_2.
    \end{align}
    Now by \Cref{lem:poly_tail_holder}, for any $v \in S^{d-1}$,
    \begin{align}
        &\E_{X,Y}\left[Y^2 \norm{\bs{h}_\k(X)\left[I^{\otimes \frac{\k-2}{2}},v\right]}^2\right] \nonumber\\
        &\lesssim \eta \cdot \E\qty[\norm{\bs{h}_\k(X)\left[I^{\otimes \frac{\k-2}{2}},v\right]}^2] \nonumber\\
        &\le \eta d \norm{\sym\qty(I^{\otimes \frac{\k-2}{2}} \otimes v)}^2 \nonumber\\
        &\le \eta d^{k/2}.
    \end{align}
    Therefore,
    \begin{align}
        \sigma^2 \lesssim \frac{\eta d^{k/2}}{n}.
    \end{align}
    We bound $\sigma_\ast$ similarly:
    \begin{align}
        \sigma_\ast^2
        &:= \sup_{v,w \in S^{d-1}} \E[(v^T(M_n - M)w)^2] \nonumber\\
        &= \frac{1}{n} \E[(u^T(M_1 - M)v)^2] \nonumber\\
        &\lesssim \frac{1}{n} \E[(u^T M_1 v)^2] \nonumber\\
        &= \frac{1}{n} \E_{X,Y}[Y \bs{h}_\k(X)[I^{\otimes \frac{\k-2}{2}},u,v]^2] \nonumber\\
        &\le \frac{1}{n} \E_{X}[\bs{h}_\k(X)[I^{\otimes \frac{\k-2}{2}},u,v]^2] \nonumber\\
        &= \frac{1}{n} \cdot \left\|\sym\left(I^{\otimes \frac{\k-2}{2}} \otimes u \otimes v\right)\right\|_F^2 \nonumber\\
        &\le \frac{1}{n} \cdot \|I\|_F^{\k-2}\nonumber \\
        &= \frac{d^{\frac{\k}{2}-1}}{n}.
    \end{align}
    Next, by \Cref{lem:partial_trace_even_pointwise} and \Cref{lem:max_poly_tail} we have that
    \begin{align}
        \overline{R} \lesssim \frac{d^{\frac{\k+2}{4}} + \log(n)^{\frac{\k}{2}} d^{\frac{\k}{4}}}{n}
    \end{align} and for any $\delta' \ge 0$, with probability at least $1-\delta'$ we have that
    \begin{align}
        \max_{i \in [n]} \frac{\|M_i\|_2}{n} \lesssim \frac{d^{\frac{\k+2}{4}} + \log(n/\delta)^{\frac{\k}{2}} d^{\frac{\k}{4}}}{n}.
    \end{align}
    Therefore with probability at least $1-n \exp(-d^{1/\k})$, we have 
    \begin{align}
        \max_{i \in [n]} \frac{\|M_i\|_2}{n} \lesssim \frac{d^\frac{\k+2}{4}}{n}.
    \end{align}
    Now we can set
    \begin{align}
        R = C\sqrt{\frac{d^\frac{\k+2}{4}}{n} \cdot \frac{d^{\frac{\k}{4}}}{n^{1/2}}} = C \cdot \frac{d^{\frac{\k+1}{2}}}{n^{3/4}}
    \end{align}
    for a sufficiently large constant $C$ and apply \Cref{lem:matrix_concentration_universality} to get that with probability at least $1-de^{-t} - n e^{-d^{1/k}}$,
    \begin{align}
        \|M_n - \E M_n\|_2 \le 2 \sigma + O\left(\frac{d^{\frac{\k-2}{4}}}{n^{1/2}}t^{1/2} + \frac{d^{\frac{3\k+1}{12}}}{n^{7/12}}t^{2/3} + \frac{d^{\frac{\k+2}{4}}}{n}t\right).
    \end{align}
    Now for $\k \ge 4$ we can set $t = d^c$ for $c < 1/8$ to get that with probability at least $1-de^{-d^c} - n e^{-d^{1/k}}$, for any $n \ge d^{\frac{\k}{2}}$ we have
    \begin{align}
        \|M_n - \E M_n\|_2 \lesssim \sigma = \sqrt{\frac{\eta d^{k/2}}{n}}.
    \end{align}
    The conclusion for $v = v_1(M_n)$ now directly follows from the Davis-Kahan inequality.
\end{proof}

\begin{lemma}\label{lem:compute_cov_M1} Assume that $(Z,Y) \sim \P$ and $\k(\P) > 2$. Then for any $k \ge 4$, if we define $\Sigma$ by
    \begin{align}
        \Sigma := \frac{d^{(k-2)/2}}{k(k-1)!!} \cdot \qty[\E[Y^2 \He_4(Z)] {w^\star}^{\otimes 4} + 4 \E[Y^2 \He_2(Z)] T + 2 \E[Y^2] \sym_2]
    \end{align}
    where $T = 6\sym({w^\star}^{\otimes 2} \otimes I) - {w^\star}^{\otimes 2} \otimes I - I \otimes {w^\star}^{\otimes 2}$, then
    \begin{align}
        \|\E[M_1 \otimes M_1] - \Sigma\|_{op} \lesssim d^{\frac{k-4}{2}}.
    \end{align}
\end{lemma}
\begin{proof}
    We will temporarily switch to the un-normalized Hermite polynomials $\He_k$. If $k = 2j+2$, this is equal to
    \begin{align}
        &k!\E[Y^2 \bs{h}_k(X)[I^{\otimes \frac{k-2}{2}}] \otimes \bs{h}_k(X)[I^{\otimes \frac{k-2}{2}}]] \nonumber\\
        &\E[Y^2 \bs{\He}_k(X)[I^{\otimes \frac{k-2}{2}}] \otimes \bs{\He}_k(X)[I^{\otimes \frac{k-2}{2}}]] \nonumber\\
        &= \E\qty[Y^2 \qty[p^{(d+4)}_j(\|x\|) xx^T - p^{(d+2)}_j(\|x\|^2) I] \otimes \qty[p^{(d+4)}_j(\|x\|) xx^T - p^{(d+2)}_j(\|x\|^2) I]] \nonumber\\
        &= \E\qty[Y^2 p^{(d+4)}_j(\|x\|)^2 x^{\otimes 4}] \nonumber\\
        &\quad - \E\qty[Y^2 p^{(d+4)}_j(\|x\|) p^{(d+2)}_j(\|x\|) x^{\otimes 2} \otimes I] \nonumber\\
        &\quad - \E\qty[Y^2 p^{(d+4)}_j(\|x\|) p^{(d+2)}_j(\|x\|) I \otimes x^{\otimes 2}] \nonumber\\
        &\quad + \E\qty[Y^2 p^{(d+2)}_j(\|x\|)^2 I \otimes I].
    \end{align}
    We will now compute this term by term. We will use $O(d^{j-1})$ to refer to error terms whose tensor operator norms are bounded by $O(d^{j-1})$. For the first term, we have that
    \begin{align}
        &\E\qty[Y^2 p^{(d+4)}_j(\|x\|)^2 x^{\otimes 4}] \nonumber\\
        &= 2^j j! d^j \E[Y^2 x^{\otimes 4}] + O(d^{j-1}) \nonumber\\
        &= 2^j j! d^j [\E[Y^2 \He_4(Z)] {w^\star}^{\otimes 4} \nonumber\\
        &\qquad + 6 \E[Y^2 \He_2(Z)] \sym({w^\star}^{\otimes 2} \otimes I) \nonumber\\
        &\qquad + 3\E[Y^2] \sym(I \otimes I)] + O(d^{j-1}).
    \end{align}
    For the second and third terms,
    \begin{align}
        &\E\qty[Y^2 p^{(d+4)}_j(\|x\|) p^{(d+2)}_j(\|x\|) x^{\otimes 2} \otimes I] \nonumber\\
        &= 2^j j! d^j \E[Y^2 x^{\otimes 2} \otimes I] + O(d^{j-1}) \nonumber\\
        &= 2^j j! d^j [\E[Y^2 \He_2(Z)] {w^\star}^{\otimes 2} \otimes I + \E[Y^2] I \otimes I] + O(d^{j-1}).
    \end{align}
    Finally for the last term we have:
    \begin{align}
        &\E\qty[Y^2 p^{(d+2)}_j(\|x\|)^2 I \otimes I] \nonumber\\
        &= 2^j j! d^j \E[Y^2] I \otimes I + O(d^{j-1}).
    \end{align}
    Renormalizing by $k!$ to reduce back to the normalized Hermite polynomials $h_k$ gives:
    \begin{align}
        &\E[Y^2 \bs{h}_k(X)[I^{\otimes \frac{k-2}{2}}] \otimes \bs{h}_k(X)[I^{\otimes \frac{k-2}{2}}]] \nonumber\\
        &= \frac{d^j}{k(k-1)!!} \cdot \Bigl[\E[Y^2 \He_4(Z)] {w^\star}^{\otimes 4} \nonumber\\
        &\qquad + \E[Y^2 \He_2(Z)] [6\sym({w^\star}^{\otimes 2} \otimes I) - {w^\star}^{\otimes 2} \otimes I - I \otimes {w^\star}^{\otimes 2}]\nonumber\\
        &\qquad + \E[Y^2] [3\mathrm{Sym}(I \otimes I) - I \otimes I]\Bigr] + O(d^{j-1}) \nonumber\\
        &= \frac{d^j}{k(k-1)!!} \cdot \qty[\E[Y^2 \He_4(Z)] {w^\star}^{\otimes 4} + 4 \E[Y^2 \He_2(Z)] T + 2 \E[Y^2] \sym_2] + O(d^{j-1})
    \end{align}
    where $T = 6\sym({w^\star}^{\otimes 2} \otimes I) - {w^\star}^{\otimes 2} \otimes I - I \otimes {w^\star}^{\otimes 2}$.
\end{proof}

\subsection{The Odd Case}

Next, we will study the odd case. This is much simpler than the even case as it doesn't require matrix concentration. However, it is not possible to directly reach $\epsilon$ error with this step. We therefore analyze the sample complexity for reaching $v \cdot w^\star \ge d^{-1/4}$:

\begin{lemma}\label{lem:partial_trace_odd}
    Let $k$ be odd and let $v_n$ be the vector from stage $1$ of \Cref{alg:partial_trace}. Assume that $\abs{y_i} \le 1$. Denote the effective signal strength $\tilde \beta_k$ by
    \begin{align}
        \tilde \beta_k := \frac{\E[Y h_k(Z)]}{\sqrt{\E[Y^2]\log^k(3/\E[Y^2])}}.
    \end{align}
    Then for a sufficiently large constant $C = C(k)$, if
    \begin{align}
        n = \frac{Cd^{k/2}}{\tilde \beta_k^2}
    \end{align}
    with probability at least $1-2e^{-d^c}$, $\frac{v_n}{\|v_n\|} \cdot w^\star \ge d^{-1/4}$. Furthermore, if $n \ge C \frac{d^{\frac{k+1}{2}}}{\epsilon \tilde \beta_k^2}$, we have that with probability at least $1-2e^{-d^c}$, $v \cdot w^\star \ge 1-\epsilon$.
\end{lemma}
\begin{proof}
    As in \Cref{lem:partial_trace_even}, let $\eta := \lambda_k^2 \log^k(3/\lambda_k)$.
    We will begin by computing the variance. Note that for any $v \in S^{d-1}$, by \Cref{lem:poly_tail_holder} we have:
    \begin{align}
        &\E[Y^2 \bs{h}_k(X)[I^{\otimes \frac{k-1}{2}},v]^2] \nonumber\\
        &\lesssim \eta \E[\bs{h}_k(X)[I^{\otimes \frac{k-1}{2}},v]^2] \nonumber\\
        &\le \eta d^{\frac{k-1}{2}}.
    \end{align}
    Therefore by \Cref{lem:poly_tail_bound}, with probability at least $1-\delta$ we have
    \begin{align}
        (v_n - \E[v_n]) \cdot w \lesssim \sqrt{\frac{\eta \log(1/\delta) d^{\frac{k-1}{2}}}{n}} + \frac{\log(n/\delta)^{k/2+1} d^{\frac{k-1}{2}}}{n}.
    \end{align}
    For the norm, recall that if $k = 2j+1$,
    \begin{align}
        \sqrt{k!} \bs{h}_k(x)[I^{\otimes \frac{k-1}{2}}] = x p^{(d)}_j(\|x\|^2).
    \end{align}
    Let $\bar x_i := \frac{x_i}{\|x_i\|}$ and let $\bar x_i^\perp = \bar x_i - w^\star (\bar x_i \cdot w^\star)$. Then $\bar x_i$ is independent of $\|x_i\|$ and $x_i \cdot w^\star$. Therefore viewed as a function of $\{\bar x_i^\perp\}$, $v_n^\perp$ is a sub-Gaussian vector with constant $\sigma^2 = \frac{1}{n} \sum_{i=1}^n y_i^2 p_j^{(d)}(\|x\|)^2.$ Therefore with probability at least $1-2e^{-d}$, $\|v_n^\perp\| \lesssim \sigma \sqrt{d}$ so it suffices to bound $\sigma$. We have by \Cref{lem:poly_tail_bound},
    \begin{align}
        \sigma^2 \lesssim \eta d^{j} + \tilde O(d^{j}/\sqrt{n}).
    \end{align}
    Therefore, with probability at least $1-2e^{-d^c}$, $\sigma^2 \lesssim \lambda_k^2 \log^k(3/\lambda_k) d^{j}$ and
    \begin{align}
        \norm{v_n^\perp} \lesssim \sqrt{\frac{\eta d^{\frac{k+1}{2}}}{n}}.
    \end{align}
    Combining this with the bound of $(v_n - \E[v_n]) \cdot w^\star$ this gives with probability at least $1-e^{-d^c}$,
    \begin{align}
        \norm{v_n - \E[v_n]} \lesssim \sqrt{\frac{\eta d^{\frac{k+1}{2}}}{n}}.
    \end{align}
    Combining everything gives that when $n = C d^{k/2}/\tilde \beta_k^2$,
    \begin{align}
        \frac{v_n}{\|v_n\|} \cdot w^\star \ge \frac{\beta_k + (v_n - \E[v_n]) \cdot w^\star}{\beta_k + \|v_n - \E[v_n]\|} \gtrsim \sqrt{\frac{\beta_k^2 n}{\eta d^{\frac{k+1}{2}}}} = \sqrt{\frac{\tilde \beta_k^2 n}{d^{\frac{k+1}{2}}}} = C d^{-1/4}.
    \end{align}
    In addition, when $n = C d^{\frac{k+1}{2}}/(\epsilon \tilde \beta_k^2)$,
    \begin{align}
        \frac{v_n}{\|v_n\|} \cdot w^\star \ge \frac{\beta_k + (v_n - \E[v_n]) \cdot w^\star}{\beta_k + \|v_n - \E[v_n]\|} \ge \frac{\beta_k(1-\epsilon/2)}{\beta_k(1+\epsilon/2)} \ge 1-\epsilon.
    \end{align}
\end{proof}

\subsection{Proof of Theorem \ref{thm:optimal_sq_alg}}

We are now ready to prove \Cref{thm:optimal_sq_alg}:
\begin{proof}[Proof of \Cref{thm:optimal_sq_alg}]
    By Lemma \ref{lem:bounded_T_sq_alg} there exists a truncation of $\zeta_k$, $\mathcal{T}: \R \to [-1,1]$ such that the effective signal strength $\tilde \lambda_k$ satisfies:
    \begin{align}
        \tilde \lambda_k^2 = \frac{\E_\P[\mathcal{T}(Y)h_\k(Z)]^2}{\E_\P[\mathcal{T}(Y)^2]\log^k(3/\E_\P[\mathcal{T}(Y)^2])} \gtrsim \frac{\lambda_\k^2}{\log(3/\lambda_k)^{2\k}}.
    \end{align}

    We will first show that the output of the first stage satisfies $v \cdot w^\star = \Theta(1)$. For $k = 1$, this follows directly from \Cref{lem:partial_trace_odd}. For $k=2$, the result follows from \cite[Theorem 2]{mondelli2018fundamental}. For $k \ge 2$ with $k$ even, the result follows from \Cref{lem:partial_trace_even}. Finally, when $k \ge 3$ with $k$ odd we have that after the partial trace warm start, by \Cref{lem:partial_trace_odd} $v \cdot w^\star \ge d^{-1/4}$. Then until $v \cdot w^\star = 1/4$, each step of tensor power iteration will double $v \cdot w^\star$ with $n \gtrsim \frac{d}{(v \cdot w^\star)^{2k-4} \tilde \beta_k^2}$ by \Cref{lem:power_iter_convergence}. This will converge to $v \cdot w^\star = 1/4$ in $\log(d^{1/4}) \le \log(d)$ steps. Finally, by \Cref{lem:power_iter_convergence} one more step of tensor power iteration with
    $
        n \ge \frac{d}{\epsilon \tilde \beta_k^2}
    $
    gives $(v \cdot w^\star)^2 \ge 1-\epsilon$. As every step happens with probability at least $1-2e^{-d^c}$, a union bound gives that the final success probability is also $1-2e^{-d^c}$ for constant $c$ depending only on $k$.
\end{proof}

\subsection{Concentration}

\begin{lemma}\label{lem:partial_trace_even_pointwise}
    Let $X \sim \gamma_d$, let $k$ be an even number, and define $M := \bs{h}_k(x)\left[I^{\otimes \frac{k-2}{2}}\right]$. Then,
    \begin{align}
        \E[\|M_2\|_2^2]^{1/2} \le d^{\frac{k+2}{4}} \qq{and}
        \E\left[|\|M\|_2 - \E\|M\|_2|^p\right]^{1/p} \lesssim p^{k/2} d^{\frac{k}{4}}.
    \end{align}
\end{lemma}
\begin{proof}
    For the first inequality we have:
    \begin{align}
        \E[\|M_2\|_2^2]
        &\le \E[\|M_2\|_F^2] \nonumber\\
        &= \sum_{i,j} \E[\bs{h}_k(X)[I^{\otimes \frac{k-2}{2}},e_i,e_j]^2] \nonumber\\
        &= \sum_{i,j} \left\|\sym\left(I^{\otimes \frac{k-2}{2}} \otimes e_i \otimes e_j\right)\right\|_F^2 \nonumber\\
        &\le d^2 \cdot d^{\frac{k-2}{2}} = d^{\frac{k+2}{2}}.
    \end{align}
    For the moment bound, first note that
    \begin{align}
        \E[\tr(M)^2] = \E[\bs{h}_k(X)[I^{\otimes k/2}]^2] \le \|I\|_F^k = d^{k/2}.
    \end{align}
    Next, note that by symmetry, there exist polynomials $p,q$ of degree at most $\frac{\k}{2}-1$ such that $M = p(\|x\|^2) xx^T + q(\|x\|)^2 I$. We can expand $\tr(M)$ as:
    \begin{align}
        \tr(M) = p(\|x\|^2) \|x\|^2 + q(\|x\|^2) d.
    \end{align}
    Therefore both $p(\|x\|^2) \|x\|^2$ and $q(\|x\|^2) d$ must have variance bounded by $d^{k/2}$. By Gaussian hypercontractivity, they also have $p$ norms bounded by $(p-1)^{k/2} d^{k/4}$. Then,
    \begin{align}
        \|M\|_{2} &= \max\left(|q(\|x\|^2)|,|p(\|x\|^2)\|x\|^2 + q(\|x\|^2)|\right) \nonumber\\
        &\le |p(\|x\|^2)|\|x\|^2 + |q(\|x\|^2)|
    \end{align}
    so letting $C$ denote the mean of the right hand side, if we subtract $C$ from both sides we get that
    \begin{align}
        \E[|\|M\|_2-C|^p]^{1/p} \lesssim p^{k/2} d^{k/4}.
    \end{align}
    Finally,
    \begin{align}
        \E[|\|M\|_2-\E\|M\|_2|^p]^{1/p}
        &\le \E[|\|M\|_2-C|^p]^{1/p} + |\E \|M\|_2-C| \nonumber\\
        &\le 2\E[|\|M\|_2-C|^p]^{1/p} \nonumber\\
        &\lesssim p^{k/2} d^{k/4}
    \end{align}
    where the second to last line follows from Jensen's inequality.
\end{proof}

\begin{lemma}\label{lem:partial_trace_odd_pointwise}
    Let $X \sim \gamma_d$, let $k$ be an even number, let $v \in S^{d-1}$ and define $M := \bs{h}_k(x)\left[I^{\otimes \frac{k-3}{2}},v\right]$. Then,
    \begin{align}
        \E[\|M_2\|_2^2]^{1/2} \le d^{\frac{k+1}{4}} \qq{and}
        \E\left[|\|M\|_2 - \E\|M\|_2|^p\right]^{1/p} \lesssim p^{k/2} d^{\frac{k-1}{4}}.
    \end{align}
\end{lemma}
\begin{proof}
    As above we have
    \begin{align}
        \E[\|M_2\|_2^2]
        &\le \E[\|M_2\|_F^2] \nonumber\\
        &= \sum_{i,j} \E[\bs{h}_k(X)[I^{\otimes \frac{k-3}{2}},v,e_i,e_j]^2] \nonumber\\
        &= \sum_{i,j} \left\|\sym\left(I^{\otimes \frac{k-3}{2}} \otimes v \otimes e_i \otimes e_j\right)\right\|_F^2 \nonumber\\
        &\le d^2 \cdot d^{\frac{k-3}{2}} = d^{\frac{k+1}{2}}.
    \end{align}
    In addition,
    As above we have
    \begin{align}
        \E[\tr(M)^2] = \E[\bs{h}_k[I^{\otimes \frac{k-1}{2}},v]^2] \le \|I\|_F^{k-1} = d^{\frac{k-1}{2}}.
    \end{align}
    The remainder of the proof is identical to the proof of \Cref{lem:partial_trace_even_pointwise}.
\end{proof}

\subsection{Proofs for Unknown \texorpdfstring{$\P$}{P} Learning}
\label{app:unknownP}

\begin{restatable}[Unknown $\P$ label transformation]{lemma}{lemagnosticP}
\label{lem:bounded_T_agnostic}
    Assume $M \geq N$ and Assumption \ref{ass:sourcecond}. Let $\theta \sim \mathrm{Unif}(\mathcal{S}^{M-1})$ and consider $\Psi = \sum_{l=1}^M \theta_l \phi_l$. Let $R>0$ and define $\tilde{\mathcal{T}}(y):= \frac1R \Psi(y) \mathbf{1}_{|\Psi(y)|\leq R}$. Then with probability greater than $1- \delta$ over the draw of $\theta$, for $R = \frac{3^\k 4 \sqrt{M}}{\lambda_\k \delta \sqrt{1-\varepsilon_M}}$ we have 
    \begin{equation}
        \tilde{\eta} := \left|\E_{\P} [\tilde{\mathcal{T}}(Y) h_{\k}(Z)] \right| \gtrsim \delta^2 \lambda_\k^2 > 0~,
    \end{equation} 
    where $\gtrsim$ hides constants in $\k$ and appearing in Assumption \ref{ass:sourcecond}. 
\end{restatable}
\begin{proof}
        Let $A_M \zeta_{\k} = \sum_{l \leq M} \upsilon_l \phi_l$ the $L^2(\R, \P_y)$-orthogonal projection of $\zeta_{\k}$ onto the space spanned by degree-$M$ polynomials. We have
    \begin{align}
        R \, \E_{\P} [\tilde{\mathcal{T}}(Y) h_\k(Z)] &= \langle \Psi, \zeta_\k \rangle_{\P_y} - \E_{\P} [ \zeta_\k(Y) \Psi(Y) \mathbf{1}_{|\Psi(Y)| \geq R }]  \nonumber\\
        &= \langle \Psi, A_M \zeta_\k \rangle_{\P_y} - \E_{\P} [ \zeta_\k(Y) \Psi(Y) \mathbf{1}_{|\Psi(Y)| \geq R }] \nonumber\\
        &= \| A_M \zeta_k \| \left\langle \Psi, \frac{A_M \zeta_\k}{\| A_M \zeta_k \|} \right\rangle_{\P_y} - \E_{\P} [ \zeta_\k(Y) \Psi(Y) \mathbf{1}_{|\Psi(Y)| \geq R }]~.
    \end{align}
    Now, following the proof of Lemma \ref{lem:bounded_T_sq_alg} we bound the second term in the RHS:
\begin{align}
        \abs{\E_\P[\zeta_\k(Y) \Psi(Y) \1_{\abs{\Psi(Y)} \ge R}]}
        &\le \sqrt{\|\Psi(Y) \|_{\P_y}^2 \E_{\P} [ \zeta_\k(Y)^2 \1_{\abs{\Psi(Y)} \ge R} ]  }  \\  
        &= \sqrt{\E_{\P} [ \zeta_\k(Y)^2 \1_{\abs{\Psi(Y)} \ge R} ]} \nonumber\\
        &\le \sqrt{\E_\P[h_\k(Z)^4] \PP[\abs{\Psi(Y)} \ge R]} \nonumber \\
        &\le \frac{3^\k \sqrt{\E[\Psi(Y)^2]}}{R} = \frac{3^\k }{R}~, \nonumber
    \end{align}
where we have used the fact that $\| \Psi \| =1$. 

Now, thanks to Assumption \ref{ass:sourcecond}, and $M \geq N$, there exists $\kappa > 0$ such that 
$\| A_M \zeta_\k \|^2 = \lambda_\k^2( 1- \varepsilon_M) \geq \kappa \lambda_\k^2$. We thus obtain 
\begin{align}
    \left| \E_{\P} [\tilde{\mathcal{T}}(Y) h_\k(Z)] \right| & \geq \frac1R \lambda_\k \sqrt{\kappa} \left| \left\langle \Psi, \frac{A_M \zeta_\k}{\| A_M \zeta_k \|} \right\rangle_{\P_y}\right| - \frac{3^\k }{R^2} \nonumber \\
    &= \frac{\lambda_\k \sqrt{\kappa}}{R} | \theta \cdot \upsilon | -  \frac{3^\k }{R^2}~.
\end{align}

Finally, we conclude with a basic anti-concentration property of the uniform measure on $\mathcal{S}_{M-1}$: 
    \begin{lemma}[Anti-Concentration of the Uniform measure, {\cite[Lemma A.7]{bietti2022learning}}]
    \label{lem:sphere_anti}
    Let $\theta \sim \mathrm{Unif}(\mathcal{S}_{M-1})$ and $\theta_0 \in \mathcal{S}_{M-1}$ arbitrary.
    For any $\epsilon>0$, we have $\mathbb{P}[ | \theta \cdot \theta_0| \leq \epsilon ] \leq 4 \epsilon \sqrt{M}$. 
    \end{lemma}
Putting everything together, and picking $\epsilon = \delta /(4\sqrt{M})$ in Lemma \ref{lem:sphere_anti}, we obtain that with probability greater than $1-\delta$, for $R \geq \frac{3^\k 4 \sqrt{M}}{\lambda_\k \delta \sqrt{\kappa}}$, 
\begin{align}
    \left| \E_{\P} [\tilde{\mathcal{T}}(Y) h_\k(Z)] \right| & \geq \frac{\lambda_\k \delta \sqrt{\kappa}}{4 \sqrt{M} R} - \frac{3^\k }{R^2} \nonumber\\
    & \geq \frac{\lambda_\k \delta \sqrt{\kappa}}{8 \sqrt{M} R}~.
\end{align}
\end{proof}

\begin{algorithm}[h]
\SetAlgoLined
\KwIn{Dataset $\mathcal{D} = \{(x_i,y_i)\}_{i=1}^n$, largest degree $M$, largest exponent $\k$, signal strength $\lambda_{\k}$, basis $(\phi_l)_{l \leq M}$}
    Set $R = C / \lambda_\k^2$ and $\tilde{R} = \tilde{C}$.  \\
    Split $\mathcal{D}$ into train $\mathcal{D}'$ and validation $\tilde{D}$ such that $|\tilde{D}|=L$.\\
    Draw random $\theta \in \textrm{Unif}(\mathcal{S}_{M-1})$ and form $\Psi = \sum_{l\leq M} \theta_l \phi_l$.\\
    \For{$k \leq  \k$}{        
        Run Algorithm \ref{alg:partial_trace} on $\mathcal{D}'$ with $\mathcal{T} = R^{-1}\Psi \mathbf{1}_{|\Psi|\leq R}$ to obtain $\hat{w}_k$.\\
        Compute $F_k = \frac1L \sum_{l=1}^L \Psi(y_l) h_k( x_l \cdot \hat{w}_k) \mathbf{1}_{|\Psi(y_l) | \leq \tilde{R}} $.
    }
    Define $\hat{k} = \arg\max_{k} |F_k|$. \\
\KwOut{$\hat{w}_{\hat{k}}$}
\caption{Partial Trace Algorithm, unknown $\P$ and $\k$}
\label{alg:partial_trace_bis}
\end{algorithm}

\unknownP*
\begin{proof}
The proof is adapted from \citep[Theorem 14]{dudeja2018learning}. 

\Cref{thm:optimal_sq_alg} together with \Cref{lem:bounded_T_agnostic} ensures that, provided $n \gg \frac{d^{\k/2}}{\delta^2 \lambda_\k^2}$, with probability greater than $1- \delta - e^{-d^\kappa}$, 
$\| \hat{w}_\k - w^\star\| \lesssim \sqrt{d^{\k/2}/n}$, as well as $|\langle \Psi, \zeta_\k \rangle | \geq C \lambda_\k \delta$.

We now study the accuracy of our goodness-of-fit statistic:
\begin{lemma}[Concentration of Goodness-of-Fit Statistic]
\label{lem:goodness_conc}
For any $k\in \{1,\k\}$, $\tilde{R}>0$ and $\delta>0$, we have 
\begin{equation}
   \mathbb{P}_{\mathcal{D}'} \left\{ \left| F_k - (\hat{w}_k \cdot w^\star)^k \langle \Psi, \zeta_k \rangle_{\P_y} \right| \leq  \frac{3^k }{\tilde{R}} + C_K \frac{\tilde{R}\sqrt{\log(1/\delta)}}{\sqrt{L}} + \frac{\tilde{R}}{L} \log(1/\delta) \log(L/\delta)^{k/2} \right\} \geq 1 - 2 \delta~.
\end{equation}
\end{lemma}
For $L \gg \log \delta^{-1} \log( L \delta^{-1})^k$, a union bound over $\{1, \k\}$ thus yields, with probability greater than $1 - 2\tilde{\delta} \k$, 
\begin{align}
    \left| F_k - (\hat{w}_k \cdot w^\star)^k \langle \Psi, \zeta_k \rangle_{\P_y} \right| &\leq \inf_{\tilde{R}} \frac{3^{\k} }{\tilde{R}} + \tilde{C}_K \frac{\tilde{R}\sqrt{\log(1/\tilde{\delta})}}{\sqrt{L}}  \\
    & = O\left( \frac{\log \tilde{\delta}^{-1}}{L}\right)^{1/4} \nonumber \\
    &:= \Delta(\tilde{\delta}, L) ~~,~\forall ~k \in \{1, \k\}~.\nonumber
\end{align}

Let us now relate the performance of our estimator $\hat{w}_{\hat{k}}$ in terms of the `good' estimator $\hat{w}_{\k}$. 
Following \cite[Theorem 14]{dudeja2018learning}, we have 
\begin{align}
    |(\hat{w}_{\hat{k}} \cdot w^\star)^{\hat{k}} |\cdot | \langle \Psi, \zeta_{\hat{k}} \rangle| & \geq F_{\hat{k}} - \Delta  \\ 
    &\geq F_{\k} - \Delta \nonumber\\
    &\geq |(\hat{w}_{\k} \cdot w^\star)^{\k} |\cdot | \langle \Psi, \zeta_{\k} \rangle| - 2 \Delta \nonumber\\
    & > 0 \nonumber
\end{align}
whenever $\Delta(\tilde{\delta}, L) < C\lambda_{\k} \delta \leq  \frac12 | \langle \Psi, \zeta_{\k} \rangle|$. 
But this implies that $| \langle \Psi, \zeta_{\hat{k}} \rangle| $, which means that $\hat{k} = \k$. 

\end{proof}
\begin{proof}[Proof of Lemma \ref{lem:goodness_conc}]
    We have $\E[ F_k^2] \leq \tilde{R}^2 \E[h_k( X \cdot \hat{w}_k)^2] \leq \tilde{R}^2$, 
    and by Gaussian hypercontractivity, 
    \begin{align}
        \E[ F_k^l] & \leq \tilde{R}^l \E[h_k(Z)^l] \leq \tilde{R}^l (l-1)^{kl/2}~.
    \end{align}  
    We can then apply Lemma \cref{lem:poly_tail_bound}, to $F_k - \E[F_k]$ to deduce that for any $\delta >0 $, %
\begin{equation}
    \mathbb{P}\left[ | F_k - \E[F_k]| \gtrsim_k \frac{\tilde{R}\sqrt{\log(1/\delta)}}{\sqrt{L}} + \frac{\tilde{R}}{L} \log(1/\delta) \log(L/\delta)^{k/2}  \right] \leq 2\delta~. \frac{\tilde{R}^2}{L t^2}~.
\end{equation}
Next we bound the effect of the truncation:
    \begin{align}
       \left| \E[ F_k] - \E[\Psi(Y) h_k(X \cdot \hat{w}_k)] \right|&= 
        \left|\E_{\PP} \left[\Psi(Y) h_\k(X \cdot \hat{w}_k) \cdot \mathbf{1}_{|\Psi(Y) | > \tilde{R}}\right] \right| \\
        &\leq \sqrt{\|\Psi(Y) \|_{\P_y}^2 \E [ h_k(X \cdot \tilde{w}_k)^2 \1_{\abs{\Psi(Y)} \ge \tilde{R}} ]  } \nonumber \\  
        &= \sqrt{\E [ h_k(X \cdot \tilde{w}_k)^2 \1_{\abs{\Psi(Y)} \ge R} ]} \nonumber\\
        &\le \sqrt{\E[h_k(Z)^4] \PP[\abs{\Psi(Y)} \ge \tilde{R}]} \nonumber \\
        &\le \frac{3^k \sqrt{\E[\Psi(Y)^2]}}{\tilde{R}} = \frac{3^k }{\tilde{R}}~, \nonumber
    \end{align}
and finally let us compute
\begin{align}
    \E[\Psi(Y) h_k(X \cdot \hat{w}_k)] &= \E_Y [ \Psi(Y) \E_{Z|Y} \E_{W\sim \gamma} h_k( (\hat{w}_k \cdot w^\star) Z + \sqrt{1-(\hat{w}_k \cdot w^\star)^2} W) ] \nonumber \\
    &= (\hat{w}_k \cdot w^\star)^k \E_Y [ \Psi(Y) \E_{Z|Y} h_k(Z)] \nonumber\\ 
    &= (\hat{w}_k \cdot w^\star)^k \langle \Psi, \zeta_k \rangle_{\P_y} ~,
\end{align}
where we used the fact that Hermite polynomials are eigenfunctions of the Ornstein-Ulhenbeck semigroup.
\end{proof}

%% file: proofs_smoothness.tex
\section{Proofs of Section \ref{sec:smoothsec}}

\subsection{Proof of Theorem \ref{thm:smooth_link}}
\label{app:smoothlink_main}
Throughout this section, we will occasionally use the un-normalized Hermite polynomials (\Cref{def:unnormalized_hermite}) which will simplify the notation.

\begin{lemma}\label{lem:mollifier}
    There exists a function $f: [0,1] \to [0,1]$ such that $f(0) = 0$, $f(1) = 1$, $f$ is strictly monotonic and for all $k \in \N$, $f^{(k)}(0) = f^{(k)}(1) = 0$.
\end{lemma}
\begin{proof}
    Let
    \begin{align}
        f(x) = \frac{1}{Z} \int_0^x \exp\qty{-\frac{1}{s(1-s)}} ds \qq{where} Z = \int_0^1 \exp\qty{-\frac{1}{s(1-s)}} ds
    \end{align}
    Then it is clear that $f(0) = 0$, $f(1) = 1$, and if $a < b$,
    \begin{align}
        f(b) - f(a) = \frac{1}{Z}\int_a^b \exp\qty{-\frac{1}{s(1-s)}} ds > 0
    \end{align}
    so $f$ is monotonic. Finally, we have that
    \begin{align}
        f^{(k)}(0)
         & = \frac{1}{Z} \qty{\frac{d^{k-1}}{d x^{k-1}} \exp\qty{-\frac{1}{x(1-x)}}} \\
         & = \lim_{x \to 0}\qty{\frac{p_k(x)}{q_k(x)} \exp\qty{-\frac{1}{x(1-x)}}}   \nonumber\\
         & = 0. \nonumber
    \end{align}
    where $p_k(x),q_k(x)$ are polynomials in $x$. The proof for $f^{(k)}(1) = 0$ is identical.
\end{proof}
In the deterministic setting $\P = ( \mathrm{Id} \otimes \sigma )_\# \gamma$ where $\sigma \in C^1(\R)$, the condition that $\zeta_k = 0$ in $L^2(\R, \P_y)$ reduces to $\E[ g(Y) h_k(Z)]=0$ for any $g \in L^2_{\P_y}$. 
From Sard's theorem, the set of critical values of $\sigma$, ie, the set $S_\sigma:=\{ y \in \R; \, \exists z \in \sigma^{-1}(y) \, s.t.\, \sigma'(z)=0\}$ has Lebesgue measure zero. 
In particular, we have that a.e. $\sigma^{-1}(y)$ is a discrete set  with $\sigma'(z) \neq 0$ for any $z \in \sigma^{-1}(y)$. Therefore, applying the coarea formula leads to  
\begin{align}
\label{eq:vanilla}
    0 &= \mathbb{E}[ g(Y) h_k(Z)] \nonumber \\
    &= \int g( \sigma(z)) h_k(z) \gamma(z) dz \nonumber\\
    &= \int g( \sigma(z)) h_k(z) \gamma(z) \frac{|\sigma'(z)|}{|\sigma'(z)|} dz \nonumber\\
    &= \int g(y) \left(\int_{\sigma^{-1}(y)} \frac{h_k(z) \gamma(z)}{|\sigma'(z)|} d\mathcal{H}_0(z) \right) dy ~,  
\end{align}
and therefore 
\begin{align}
\label{eq:coarea}
    0 &= \int_{\sigma^{-1}(y)} \frac{h_k(z) \gamma(z)}{|\sigma'(z)|} d\mathcal{H}_0(z) \nonumber\\
    &= \sum_{i} \frac{h_k(\sigma^{-1}(y)_i)\gamma(\sigma^{-1}(y)_i)}{|\sigma'(\sigma^{-1}(y)_i)|}~~ \,,~\text{for } y \notin S_\sigma~.
\end{align}

We first verify an equivalent integral condition:
\begin{lemma}[Integral Form]
\label{lem:integral_vs_differential}
    The condition $\int_{\sigma^{-1}(y)} \frac{h_k(z)\gamma(z)}{|\sigma'(z)|} d z = 0$ for $\P_y$-a.e. $y$ is equivalent to the following quantity being constant in $y$:
    \begin{align}
        \int_{\sigma^{-1}(y)} \gamma(z) h_{k-1}(z) \mathrm{sign}(\sigma'(z)) d\mathcal{H}_0(z) = C~.
    \end{align}
\end{lemma}
\begin{proof}
    Follows directly from integrating with respect to $y$, 
    and using that $[ h_k \gamma]' = - h_{k-1} \gamma$. 
\end{proof}

\begin{definition}
    Given $u=(u_1,\ldots,u_n) \in \R^n$, we define $Q(u)$ by:
    \begin{align}
        Q(u) := \mathbb{E}_{z \sim \gamma}{\prod_{i=1}^n (u_i + z)} = \sum_{i = 0}^{\lfloor n/2 \rfloor} (2i-1)!! e_{n-2i}(u)~,
    \end{align}
    where $e_k(u)$ is the $k$th elementary symmetric polynomial on $u_1,\ldots,u_n$.
\end{definition}
\begin{definition}
    Given $n$ distinct points $u_1,\ldots,u_n$ we define $v(u) \in \R^n$, $u=(u_1, \ldots, u_n)$ by
    \begin{align}
        v(u)_i = \frac{Q(u_1,\ldots,\hat u_i,\ldots,u_n)}{\prod_{i \ne j}(u_j - u_i)},
    \end{align}
    where $(u_1,\ldots,\hat u_i,\ldots,u_n)$ is the $(n-1)$-dimensional vector in which the $i$-th coordinate has been removed. 
\end{definition}

\begin{lemma}
    For any distinct points $u_1,\ldots,u_n$ and for $0 \le k \le n$ we have
    \begin{align}
        \sum_{i=1}^n \He_k(u_i) v(u)_i =
        \begin{cases}
            1               & k=0       \\
            0               & 0 < k < n \\
            (-1)^{n+1} Q(u) & k=n
        \end{cases}.
    \end{align}
\end{lemma}
\begin{proof}
    We will first prove the result for $k < n$. Let $A \in \R^{n \times n}$ be defined by $A_{ij} := \He_{i-1}(x_j)$.  Then we can decompose $A = C V$ where $C \in R^{n \times n}$ contains the coefficients of the Hermite polynomials, i.e. $C_{ij}$ is the coefficient of $x^j$ in $\He_i(x)$, and $V \in R^{n \times n}$ is a Vandermonde matrix defined by $V_{ij} = x_j^{i-1}$. Note that $C$ is invertible (since it is triangular) and as the $x_i$ are distinct, $V$ is invertible as well so $A$ is invertible. Then the unique solution to $A v^\star = e_1$ is given by $v^\star = V^{-1} C^{-1} e_1$. By the formula for converting Hermite polynomials to monomials, we have that $C^{-1} = \abs{C}$ so $(C^{-1} e_1)_j = (j-2)!! \mathbf{1}_{2 \mid j-1}$. Therefore from the formula for the inverse of a Vandermonde matrix,
    \begin{align}
        v^\star_i
         & = \sum_{2 \mid j-1} (V^{-1})_{ij} (j-2)!!                                                               \nonumber\\
         & = \sum_{2 \mid j-1} \frac{e_{n-j}(x_1,\ldots,\hat x_i,\ldots,x_n)}{\prod_{k \ne i} (x_k - x_i)} (j-2)!! \nonumber\\
         & = \frac{\sum_{2 \mid j} e_{n-j-1}(x_1,\ldots,\hat x_i,\ldots,x_n) (j-1)!!}{\prod_{j \ne i} (x_j - x_i)} \nonumber\\
         & = \frac{Q(x_1,\ldots,\hat x_i,\ldots,x_n)}{\prod_{j \ne i}(x_j - x_i)}                                  \nonumber\\
         & = v(x)_i.
    \end{align}
    Therefore $v^\star = v(x)$ so $A v(x) = e_1$. Next, pick some $x_{n+1}$ which is distinct from $x_1,\ldots,x_n$. By the above computation we know that
    \begin{align}
        \sum_{i=1}^{n+1} \He_n(x_i) v(x_1,\ldots,x_{n+1})_i = 0
    \end{align}
    which implies that
    \begin{align}
        \sum_{i=1}^{n} \He_n(x_i) v(x_1,\ldots,x_{n+1})_i & = -\He_n(x_{n+1}) v(x_1,\ldots,x_{n+1})_i                                 \nonumber \\
                                                         & = -\He_n(x_{n+1}) \frac{Q(x_1,\ldots,x_n)}{\prod_{i=1}^n (x_i - x_{n+1})}.
    \end{align}
    We can now take $x_{n+1} \to \infty$ on both sides. For the left hand side we have for $i \le n$:
    \begin{align}
        \lim_{x_{n+1} \to \infty} v(x_1,\ldots,x_{n+1})_i
         & = \lim_{x_{n+1} \to \infty} \frac{Q(x_1,\ldots,\hat x_i,\ldots,x_n,x_{n+1})}{(x_{n+1} - x_i)\prod_{j \ne i, j \le n} (x_j - x_i)} \nonumber\\
         & = \lim_{x_{n+1} \to \infty} \frac{\mathbb{E}_{z \sim N(0,1)[\prod_{j \ne i} (x_j + z)]}}{\prod_{j \ne i} (x_j - x_i)}            \nonumber \\
         & = v(x_1,\ldots,x_n)_i.
    \end{align}
    On the right hand side we have:
    \begin{align}
        \lim_{x_{n+1} \to \infty} \frac{\He_n(x_{n+1})}{\prod_{i=1}^n (x_i - x_{n+1})} = (-1)^n.
    \end{align}
    Putting it together gives that
    \begin{align}
        \sum_{i=1}^n \He_n(x_i) v(x_1,\ldots,x_n)_i = (-1)^{n+1} Q(x_1,\ldots,x_n)
    \end{align}
    which completes the proof.
\end{proof}

We will use the following well-know result for the Gauss-Hermite quadrature (\cite{davis2007methods}): 
\begin{lemma}[Gauss-Hermite Quadrature]\label{lem:gauss_hermite}
    Let $r_1,\ldots,r_n$ be the roots of $\He_n$. Then
    \begin{align}
        v(r_1,\ldots,r_n)_i = \frac{1}{n \He_{n-1}(r_i)^2} > 0.
    \end{align}
\end{lemma}

\begin{lemma}\label{lem:perturb_x_for_Q}
    For any $k^\star \ge 0$ and any $\epsilon > 0$ there exist points $x_1,\ldots,x_{k^\star}$ such that if $r_1,\ldots,r_{k^\star}$ are the roots of $\He_{k^\star}(x)$, we have
    \begin{align}
        \abs{x_i - r_i} \le \epsilon ~\forall i \qq{and}Q(x_1,\ldots,x_n) \ne 0.
    \end{align}
\end{lemma}
\begin{proof}
    Note that $Q$ is a polynomial in $n$ variables of degree $n$ which can have only finitely many roots. Therefore it is not possible for all points $x \in \bigtimes_{i=1}^n [r_i - \epsilon,r_i + \epsilon]$ to be roots of $Q$.
\end{proof}

We are now ready to prove Theorem \ref{thm:smooth_link}:
\smoothlink*
\begin{proof}
    Let $r_1,\ldots,r_{k^\star}$ be the roots of $\He_{k^\star}$. From Lemma \ref{lem:gauss_hermite}, we know that $v(r_1,\ldots,r_{k^\star}) > 0$. Therefore by continuity there exists $\delta$ such that
    \begin{align}
        \abs{x_i - r_i} \le \delta ~\forall i \implies v(x_1,\ldots,x_n) > 0.
    \end{align}
    Then by Lemma \ref{lem:perturb_x_for_Q} there exist $x_1(0),\ldots,x_n(0)$ with $\abs{x_i(0) - r_i} \le \delta/2$ such that $Q(x_1(0),\ldots,x_n(0)) \ne 0$. Again by continuity there exists $\epsilon$ such that for all $x_1,\ldots,x_n$ with $\abs{x_i - x_i(0)} \le \epsilon$ for all $i$, we have both
    \begin{align}
        v(x_1,\ldots,x_n) > 0 \qq{and} \mathrm{sign}(Q(x_1,\ldots,x_n)) = \mathrm{sign}(Q(x(0))).
    \end{align}
    Now let $\gamma(x) := \frac{e^{-x^2/2}}{\sqrt{2\pi}}$ be the PDF of a standard Gaussian and let $x$ evolve according to the ODE:
    \begin{align}\label{eq:keep_k_ode}
        x_i'(t) = \gamma(x_i)^{-1} v(x_1,\ldots,x_n)_i \qq{for} i=1,\ldots,n.
    \end{align}
    We will run this ODE for $t \in [-\tau,\tau]$ for $\tau$ sufficiently small so that $\|x(t)-x(0)\|_1 < \epsilon$ for all $t \in [-\tau,\tau]$. We will also define the quantity:
    \begin{align}
        Z_k(t) := \sum_{i=1}^n \He_{k-1}(x_i(t)) \gamma(x_i(t)).
    \end{align}
    Then using that $\frac{d}{dx} \He_{k-1}(x) \gamma(x) = -\He_{k}(x) \gamma(x)$, we have that for any $1 \le k \le k^\star$,
    \begin{align}
        Z_k'(t)
         & = -\sum_{i=1}^n \He_{k}(x_i)\gamma(x_i) x_i'(t)\nonumber \\
         & = -\sum_{i=1}^n \He_{k}(x_i) v(x(t))_i         \nonumber \\
         & = \begin{cases}
            0                        & k < k^\star \\
            (-1)^{k^\star+1} Q(x(t)) & k = k^\star
        \end{cases}.
    \end{align}
    Therefore we have that $Z_k(t) = Z_k(0)$ for all $k < k^\star$ and $\abs{Z_{k^\star}(t) - Z_{k^\star}(-t)} > 0.$

    Note that by construction, $x_i'(t) > 0$ for all $t \in [-\tau,\tau]$. Therefore $x_i: [-\tau,\tau]: \R$ is injective and for $\epsilon$ sufficiently small, the images $\{x_i([-\tau,\tau])\}_i$ don't intersect. We can now define $\hat \sigma$ by
    \begin{align}
        \hat\sigma(x) :=
        \begin{cases}
            \abs{x_i^{-1}(x)} & x \in x_i([-\tau,\tau]) \\
            \tau              & \text{otherwise}.
        \end{cases}
    \end{align}
    Note that $\hat\sigma$ is smooth except at $\hat\sigma^{-1}(0) \cup \hat\sigma^{-1}(\tau)$. Now let $f: [0,1] \to [0,1]$ be the mollifier constructed in Lemma \ref{lem:mollifier}. Then if $\hat\sigma$ has \sqexp $k^\star$, $\sigma(x) := f(\hat\sigma(x)/\tau)$ also has \sqexp $k^\star$ and is smooth. Therefore it suffices to show that $\hat \sigma$ has \sqexp $k^\star$.

    We will compute $\mathbb{E}[\He_k(x)|y]$ for $y \in (0,\tau]$. First, let us consider $y = \tau$. Using $\mathbb{E}[\He_k(x)] = 0$ we can write:
    \begin{align}
        \mathbb{E}[\He_k(x)|y = \tau]
         & = -\mathbb{E}[\He_k(x)|y < \tau]                                                                \nonumber\\
         & = - \sum_{i=1}^n \int_{x_i(-\tau)}^{x_i(\tau)} \He_k(x) \gamma(x) dx                            \nonumber\\
         & = \sum_{i=1}^n \He_{k-1}(x_i(\tau)) \gamma(x_i(\tau)) - \He_{k-1}(x_i(-\tau)) \gamma(x_i(-\tau)) \nonumber\\
         & = Z_k(\tau) - Z_k(-\tau).
    \end{align}
    By the above calculation this is $0$ for $k < k^\star$ because $Z_k(\tau) = Z_k(-\tau) = Z(0)$ and it is nonzero for $k = k^\star$. Next, let us consider $y \in (0,\tau)$. Then ${\hat\sigma}^{-1}(y)$ is a set of discrete points at which $\sigma$ is smooth so $y$ has a continuous density and we can apply the co-area formula (\ref{eq:coarea}):
    \begin{align}
        \mathbb{E}[\He_k(x)|y]
         & \propto \sum_{x \in \hat \sigma^{-1}(y)} \frac{\He_k(x) \gamma(x)}{\abs{\hat \sigma'(x)}}                                                           \nonumber\\
         & = \sum_{i=1}^n \frac{\He_k(x_i(y)) \gamma(x_i(y))}{\abs{\hat \sigma'(x_i(y))}} + \frac{\He_k(x_i(-y)) \gamma(x_i(-y))}{\abs{\hat \sigma'(x_i(-y))}}.
    \end{align}
    From the definition of $\hat \sigma$ we have that if $x \in x_i([-\tau,\tau])$:
    \begin{align}
        {\hat \sigma}'(x) = \frac{\mathrm{sign}(x_i^{-1}(x))}{x_i'(x_i^{-1}(x))}.
    \end{align}
    Therefore we can simplify the above expression as:
    \begin{align}
        \mathbb{E}[\He_k(x)|y]
         & \propto \sum_{i=1}^n \He_k(x_i(y)) \gamma(x_i(y))\abs{x_i'(y)} + \He_k(x_i(-y)) \gamma(x_i(-y))\abs{x_i'(-y)}.
    \end{align}
    Now because $x_i'(t) > 0$ for all $t \in [-\tau,\tau]$, we can remove the absolute values to get:
    \begin{align}
        \mathbb{E}[\He_k(x)|y]
         & \propto \sum_{i=1}^n \He_k(x_i(y)) \gamma(x_i(y))x_i'(y) + \He_k(x_i(-y)) \gamma(x_i(-y))x_i'(-y) \nonumber\\
         & = -[Z_k'(y) - Z_k'(-y)]
    \end{align}
    which we computed above. In particular, this is $0$ for $k < k^\star$ and nonzero for $k = k^\star$ for any $y > 0$. Therefore, $\mathbb{E}_y\bqty{\mathbb{E}_x\bqty{\He_k(x)|y}^2} = 0$ for $k < k^\star$ and is nonzero for $k = k^\star$ which completes the proof.
\end{proof}

\subsection{Proof of Theorem \ref{thm:additive_noise_link}}

\additivelink*
\begin{proof}
Since $\k(\P) = \k$, by \Cref{lem:bounded_T_sq_alg} there exists $g: \R \to [-1,1]$ such that $\E[g(Y) h_\k(Z)] \neq 0$. We consider ${g}_R(y) := g(y) \mathbf{1}_{|y| \leq R}$. For $R$ sufficiently large, we claim that $\E[g_R(Y) h_\k(Z)] \neq 0$. We have that
\begin{align}
    \abs{\E[g_R(Y) h_\k(Z)] - \E[g(Y) h_\k(Z)]}
    &= \abs{\E[g(Y) h_\k(Z) \1_{|y| \ge R}]} \nonumber\\
    &\le \sqrt{\E[g(Y)^2 h_\k(Z)^2] \PP[|Y| \ge R]} \nonumber\\
    &\le \sqrt{\E[Y^2]/R^2}
\end{align}
which vanishes as $R \to \infty$. Therefore for sufficiently large $R$ we have $\E[g_R(Y) h_\k(Z)] \neq 0$. Now $g_R \in L^1(\R) \cap L^2(\R)$. Let us consider its Fourier representation 
$\mathcal{T}(y) = \int \hat{g_R}(\xi) e^{i \xi y} d\xi$. 
Then 
\begin{align}
    \E[g_R(Y) h_\k(Z)] &= \int \hat{g}_R(\xi) \E[ e^{i\xi Y} h_\k(Z)] d\xi~,
\end{align}
which shows that there must exist $\xi$ such that $\phi_\xi(y, z) = e^{i\xi y} h_\k(z)$ satisfies $ \E_{\P}[ \phi_\xi(Y,Z) ] \neq 0$. 
Now, let $G_\tau(y,z) = \tau \gamma( \tau y) \delta_z$, where $\gamma$ is the standard Gaussian density. By definition we have $\tilde{\P} = \P \ast G_\tau := \mathcal{G}_\tau \P$. Recall that $\mathcal{G}_\tau$ is self-adjoint in $L^2(\R)$.  
Thus
\begin{align}
    \E_{\tilde{P}}[\phi_\xi(Y,Z) ] &= \int \phi_\xi(y,z) d\tilde{P}(y,z) \nonumber\\
    &= \int \phi_\zeta(y,z) d(\mathcal{G}_\tau{\P})(y,z) \nonumber\\
    &= \int \mathcal{G}_\tau^*( \phi)(y,z) d\P(y,z) \nonumber\\
    &= \exp(-\xi^2 \tau^2) \int \phi(y,z) d\P(y,z) \nonumber\\
    &= \exp(-\xi^2 \tau^2) \E_{\P}[ \phi_\xi(Y,Z) ] \neq 0~.
\end{align}
This shows that $\k(\tilde{\P}) \leq \k(\P)$. But note that by \Cref{lem:composition_lemma}, we also have $\k(\tilde \P) \ge \k(\P)$. Therefore $\k(\tilde \P) = \k(\P)$.

\end{proof}

%% file: proofs_IT.tex
\section{Proofs of Section \ref{sec:it}}

\ITupperbound*
\begin{proof}

    We will show that 
    \begin{align}
        n \ge C_k \frac{d\log(\frac{3}{\lambda_k})^{k/2}\log(\frac{3}{\lambda_k \epsilon})}{\lambda_k^2 \epsilon^2} = \tilde \Theta\left(\frac{d}{\lambda_k^2 \epsilon^2}\right)~,
    \end{align}
    where $C_k$ is an constant depending only on $k$, is sufficient for recovery whp.

    Throughout the proof we will use $\lesssim$ to hide constants that depend only on $k$. Let $\mathcal{N}_\delta$ be an $\delta$ net of $S^{d-1}$ with $\abs{\mathcal{N}_\delta} \le (\frac{3}{\delta})^d$. We define $g(z,y) := \zeta_k(y) h_k(z)$ and $\lambda_k = \norm{\zeta_k}_{L^2(\P_y)}$. Fix a truncation radius $R$ and define $L_n(w)$ by
    \begin{align}
        L_n(w) := \frac{1}{n}\sum_{i=1}^n g(w \cdot x_i, y_i) \1_{\abs{g(w \cdot x_i, y_i)} \le R}
    \end{align}
    We will consider the estimator:
    \begin{align}
        \hat w \in \argmin_{\hat w \in S^{d-1}} \max_{w \in \mathcal{N}_\delta} \abs{L_n(w) - \lambda_k^2 (w \cdot \hat w)^k}.
    \end{align}
    We will begin by concentrating $L_n$. First note that by \Cref{lem:poly_tail_holder},
    \begin{align}
        \E[g(w \cdot X,Y)^2] = \E[\xi_k(Y)^2 h_k(w \cdot X)^2] \lesssim \lambda_k^2 \log(3/\lambda_k^2)^{k/2}.
    \end{align}
    Therefore by Bernstein's inequality we have that for any $w \in S^{d-1}$, with probability at least $1-2e^{-\iota}$,
    \begin{align}
        \abs{L_n(w) - \E[L_n(w)]} \lesssim \sqrt{\frac{\lambda_k^2 \log(\frac{3}{\lambda_k^2})^{\frac{k}{2}}\iota}{n}} + \frac{R \iota}{n}.
    \end{align}
    Therefore by a union bound we have that with probability at least $1-2e^{-d}$,
    \begin{align}
        \sup_{w \in \mathcal{N}_\delta} \abs{L_n(\hat w) - \E[L_n(\hat w)]} \lesssim \sqrt{\frac{\lambda_k^2 \log(\frac{3}{\lambda_k^2})^{\frac{k}{2}} d \log(3/\delta)}{n}} + R\frac{d\log(3/\delta)}{n}.
    \end{align}
    Next we bound the effect of truncation on $\E[L_n(w)]$:
    \begin{align}
        &\abs{\E[L_n(w)] - \E[g(w \cdot X,Y)]} \nonumber\\
        &= \abs{\E[g(w \cdot X,Y) \1_{\abs{g(w \cdot X,Y)} > R}]} \nonumber\\
        &\le \sqrt{\E[g(w \cdot X,Y)^2]\PP[\abs{g(w \cdot X,Y)} > R]} \nonumber\\
        &\le 3^k \sqrt{\PP[\abs{g(w \cdot X,Y)} > R]}.
    \end{align}
    To control this probability, we use the moment method. By Jensen's inequality and Gaussian hypercontractivity we have for all $p \ge 1$,
    \begin{align}
        \E[\abs{g(w \cdot X,Y)}^p] \le \sqrt{\E[\xi_k(Y)^{2p}]\E[h_k(Z)^{2p}]} \le \E[h_k(Z)^{2p}] \le (2p)^{k p}.
    \end{align}
    Therefore for $R \ge (2e)^k$, we can take $p = R^{1/k}/(2e)$ to get
    \begin{align}
        \PP[\abs{g(w \cdot X,Y)} > R] \le \frac{(2p)^{k p}}{R^p} = \exp\left(-\frac{k}{2e} R^{1/k}\right). 
    \end{align}
    Plugging this in gives
    \begin{align}
        \abs{\E[L_n(w)] - \E[g(w \cdot X,Y)]} &\le 3^k \exp\left(-\frac{k}{4e} R^{1/k}\right).
    \end{align}
    Finally, because the $k$-th Hermite coefficient of $\nu$ is $\lambda_k^2$, we have that
    \begin{align}
        \E[g(w \cdot X,Y)] = \lambda_k^2 (w \cdot w^\star)^k.
    \end{align}
    Combining everything gives that with probability at least $1-2e^{-d}$:
    \begin{align}
        &\max_{w \in \mathcal{N}_\delta} \abs{\lambda_k^2 (w \cdot \hat w)^k - \lambda_k^2 (w \cdot w^\star)^k} \nonumber\\
        &\le \max_{w \in \mathcal{N}_\delta} \abs{L_n(w) - \lambda_k^2 (w \cdot \hat w)^k} + \max_{w \in \mathcal{N}_\delta} \abs{L_n(w) - \lambda_k^2 (w \cdot w^\star)^k} \nonumber \\
        &\le 2\max_{w \in \mathcal{N}_\delta} \abs{L_n(w) - \lambda_k^2 (w \cdot w^\star)^k} \nonumber \\
        &= 2\max_{w \in \mathcal{N}_\delta} \abs{L_n(w) - \E[g(w \cdot X,Y)]} \nonumber\\
        &\lesssim \sup_{w \in \mathcal{N}_\delta} \abs{L_n(\hat w) - \E[L_n(\hat w)]} + 3^k \exp\left(-\frac{k}{4e} R^{1/k}\right) \nonumber\\
        &\lesssim \sqrt{\frac{\lambda_k^2 \log(\frac{3}{\lambda_k^2})^{\frac{k}{2}} d \log(3/\delta)}{n}} + R\frac{d\log(3/\delta)}{n} + 3^k \exp\left(-\frac{k}{4e} R^{1/k}\right).
    \end{align}
    Therefore by \cite[Lemma 25]{dudeja2021statistical}, we have that
    \begin{align}
        &\lambda_k^2 \min(\norm{\hat w - w^\star},\norm{\hat w + w^\star}) \nonumber\\
        &\lesssim \delta + \sqrt{\frac{\lambda_k^2 \log(\frac{3}{\lambda_k^2})^{\frac{k}{2}} d \log(3/\delta)}{n}} + R\frac{d\log(3/\delta)}{n} + \exp\left(-\frac{k}{4e} R^{1/k}\right).
    \end{align}
    Now take $R = (4e\log(3/\delta))^k$. Then,
    \begin{align}
        \lambda_k^2 \min(\norm{\hat w - w^\star},\norm{\hat w + w^\star}) \lesssim \delta + \sqrt{\frac{\lambda_k^2 \log(\frac{3}{\lambda_k^2})^{\frac{k}{2}} d \log(3/\delta)}{n}} + \frac{d\log(3/\delta)^{k+1}}{n}
    \end{align}
    and taking $\delta = O(\epsilon \lambda_k^2)$ and using that $3/\delta \ge \log(3/\delta)^k$ completes the proof.
\end{proof}

%% file: concentration_app.tex
\section{Concentration Lemmas}

\begin{lemma}[Gaussian Hypercontractivity]\label{lem:hypercontractivity}
    Let $f$ be a polynomial of degree $k$. Then for $p \ge 2$,
    \begin{align}
        \E_{X \sim \gamma}[|f(X)|^p]^{2/p} \le (p-1)^k ~\E_{x \sim \gamma}[f(X)^2].
    \end{align}
\end{lemma}
Such moments imply the following tail bound (e.g. see \cite[Lemma 24]{damian2023smoothing}):
\begin{lemma}\label{lem:p_norm_to_tail_bound}
	Let $\delta \ge 0$ and let $X$ be a mean zero random variable satisfying
    \begin{align}
        \E[\abs{X}^p]^{1/p} \le B p^{k/2} \qq{for} p = \frac{2\log(1/\delta)}{k}
    \end{align}
    for some $k$. Then with probability at least $1-\delta$, $\abs{X} \le B (e p)^{k/2}$.
\end{lemma}
We will combine this with the following tail bound which can be easily proved with a routine truncation argument:
\begin{lemma}\label{lem:poly_tail_bound}
    Let $X_1,\ldots,X_n$ be independent mean zero random variables such that for all $p \ge 2$, $\norm{X_i}_p \le B p^{k/2}$ for some $k$ and let $\sigma = \sum_{i=1}^n \E[\norm{X_i}^2]$. Let $Y = \sum_{i=1}^n X_i$. Then with probability at least $1-2\delta$,
    \begin{align}
        \norm{Y} \lesssim_k \sigma \sqrt{\log(1/\delta)} + B \log(1/\delta) \log(n/\delta)^{k/2}.
    \end{align}
\end{lemma}
\begin{proof}
    Let $R$ be a truncation radius to be chosen later and define $\tilde X_i = X_i \1_{\norm{X_i} \le R}$. We have that with probability at least $1-\delta$, $\norm{X_1} \le C_k B \log^{k/2}(1/\delta).$ Therefore by a union bound we have that with probability at least $1-\delta$, $\max_i \norm{X_i} \le C_k B \log^{k/2}(n/\delta) =: R$. Let $\tilde Y = \sum_i \tilde X_i$. Then,
    \begin{align}
        \norm{\E[\tilde Y] - \E[Y]}
        &\le \sum_{i=1}^n \norm{\E[X_i]-\E[\tilde X_i]} \nonumber\\
        &=\sum_{i=1}^n \norm{\E[X_i \1_{\norm{X} \ge R}]}\nonumber \\
        &\le \sum_{i=1}^n \sqrt{\E[\norm{X_i}^2] \PP[\norm{X} \ge R]}\nonumber \\
        &\le \sigma \sum_{i=1}^n \PP[\norm{X_i} \ge R] \nonumber\\
        &\le \sigma \delta.
    \end{align}
    Finally, because $\E[\tilde X_i^2] \le \E[X_i^2]$, we have by Bernstein's inequality that with probability at least $1-\delta$,
    \begin{align}
        \norm{\tilde Y - \E[\tilde Y]}
        &\lesssim \sigma \sqrt{\log(1/\delta)} + R \log(1/\delta) \nonumber\\
        &\lesssim \sigma \sqrt{\log(1/\delta)} + B\log^{k/2}(n/\delta) \log(1/\delta).
    \end{align}
    Combining everything gives with probability at least $1-2\delta$,
    \begin{align}
        \norm{Y} &\lesssim \sigma \delta + \sigma \sqrt{\log(1/\delta)} + B\log^{k/2}(n/\delta) \log(1/\delta)\nonumber \\
        &\lesssim \sigma \sqrt{\log(1/\delta)} + B\log^{k/2}(n/\delta) \log(1/\delta).
    \end{align}
\end{proof}

We will often use the following lemma from \cite[Lemma 23]{damian2023smoothing} when we have  tight bounds on $\norm{X}_1$ and all moments of $Y$ but only a very loose bound on $\norm{X}_2$:
\begin{lemma}\label{lem:poly_tail_holder}
	Let $X,Y$ be random variables with $\norm{Y}_p \le B p^{k/2}$ for $$p = \min\qty(2,\frac{1}{k} \cdot \log\qty(\frac{\norm{X}_2}{\norm{X}_1})).$$ Then,
	\begin{align}
		\E[XY] \le \norm{X}_1 \cdot B \cdot (ep)^{k/2}.
	\end{align}
\end{lemma}

For matrix concentration we will use the following corollary of \cite[Theorem 2.7]{brailovskaya2023universality}:
\begin{lemma}\label{lem:matrix_concentration_universality}
    Let $Y = \sum_{i=1}^n Z_i$ where $\{Z_i\}_{i=1}^n$ are self-adjoint, mean zero, and independent random matrices. Define:
    \begin{align}
        \sigma := \|\E[Y^2]\|_2^{1/2} ,\quad
        \sigma_\ast := \sup_{v,w \in S^{d-1}} \E[(v^T Y w)^2]^{1/2} ,\quad
        \overline{R} := \E\left[\max_{i \in [n]} \|Z_i\|_2^2\right]^{1/2}.
    \end{align}
    Then for any $R \ge \overline R^{1/2} \sigma^{1/2} + \overline R \sqrt{2}$ and any $t \ge 0$, if $\delta = \PP[\max_{i \in [n]} \|Z_i\|_2 \ge R]$, then with probability at least $1-\delta-de^{-t}$,
    \begin{align}
        \|Y- \E Y\|_2 - 2\sigma \lesssim \sigma_\ast t^{1/2} + R^{1/3} \sigma^{2/3} t^{2/3} + Rt.
    \end{align}
\end{lemma}
To compute $\overline{R}$ and the tail probability $\delta$, we will use the following lemma:
\begin{lemma}\label{lem:max_poly_tail}
    Let $\{Z_i\}_{i=1}^n$ be a sequence of independent random variables with polynomial tails, i.e. there exists $B,k$ such that $\E[|Z_i|^p]^{1/p} \le B p^{k/2}$. Define $R = \max_{i=1}^n Z_i$. Then for any $p \le \log(n)/k$, $E[|R|^p]^{1/p} \lesssim B \log^{k/2}(n)$ and for any $\delta \ge 0$, with probability at least $1-\delta$, $R \lesssim B \log^{k/2}(n/\delta).$
\end{lemma}
\begin{proof}
    \begin{align}
        \E\left[|R|^p\right]^{1/p} = \E\left[\max_{i \in [n]} |Z_i|_2^p\right]^{1/p} \le \E\left[\sum_{i=1}^n |Z_i|_2^p\right]^{1/p} \le n^{1/p} B p^{k/2}.
    \end{align}
    Now plugging in $p = \log(n)/k$ gives:
    \begin{align}
        \E[R^p]^{1/p} = \E\left[\left(\max_{i \in [n]} |Z_i|_2\right)^p\right]^{1/p} \le B\left(\frac{e^2 \log(n)}{k}\right)^{\frac{k}{2}} \lesssim B\log^{\frac{k}{2}}(n).
    \end{align}
    In addition by Markov's inequality we have that when $p = \frac{t^{2/k}}{e}$,
    \begin{align}
        \PP\left[R \ge t B\right] \le \left(\frac{n^{1/p} p^{k/2}}{t}\right)^p = n \exp\left(-\frac{k}{2e} t^{2/k}\right).
    \end{align}
\end{proof}